\documentclass{article}

\PassOptionsToPackage{numbers}{natbib}

\usepackage[preprint]{neurips_2025}

\usepackage{tcolorbox}

\newcommand{\obsbox}[1]{
    \begin{tcolorbox}[colframe=black!70, colback=lightgray!15, boxrule=1pt, arc=2mm]
        \small#1
    \end{tcolorbox}
}

\usepackage{booktabs} 
\usepackage{multirow}

\usepackage{microtype}
\usepackage{graphicx}

\usepackage{wrapfig}
\usepackage{booktabs} 
\usepackage{makecell}
\usepackage{pifont}

\usepackage[colorlinks=true, citecolor=blue]{hyperref}

\usepackage{amsthm}

\newtheorem{lemma}{Lemma}

\usepackage{amsmath}
\usepackage{amssymb}
\usepackage{mathtools}
\usepackage{amsthm}

\usepackage[capitalize,noabbrev]{cleveref}

\usepackage{graphicx}

\usepackage{subcaption}
\usepackage{array}

\usepackage{xcolor}

\title{Two Is Better Than One: Rotations Scale LoRAs}

\author{
  \textbf{Hongcan Guo}\textsuperscript{1}, 
  \textbf{Guoshun Nan}\textsuperscript{1}, 
  \textbf{Yuan Yang}\textsuperscript{1}, 
  \textbf{Diyang Zhang}\textsuperscript{1}, 
  \textbf{Haotian Li}\textsuperscript{1}, 
  \textbf{Zhican Chen}\textsuperscript{1}, \\
  \textbf{Qinchuan Zhou}\textsuperscript{1}, 
  \textbf{ Yuhan Ran}\textsuperscript{3}, 
  \textbf{ Xinye Cao}\textsuperscript{1}, 
  \textbf{Sicong Leng}\textsuperscript{2}, 
  \textbf{Xiaofeng Tao}\textsuperscript{1}, 
  \textbf{and}~\textbf{Xudong Jiang}\textsuperscript{2}
  \\[1ex]
  \textsuperscript{1}Beijing University of Posts and Telecommunications, China\\
  \textsuperscript{2}Nanyang Technological University, Singapore\\
  \textsuperscript{3} University of Bristol, UK
}

\begin{document}

\maketitle

\vspace{-10pt}
\begin{abstract}
\vspace{-5pt}

Scaling Low-Rank Adaptation (LoRA)-based Mixture-of-Experts (MoE) facilitates large language models (LLMs) to efficiently adapt to diverse tasks. However, traditional gating mechanisms that route inputs to the best experts may fundamentally hinder LLMs' scalability, leading to poor generalization and underfitting issues. We identify that the root cause lies in the restricted expressiveness of existing weighted-sum mechanisms, both within and outside the convex cone of LoRA representations. This motivates us to propose \textit{RadarGate}, a novel geometrically inspired gating method that introduces rotational operations of LoRAs representations to boost the expressiveness and facilitate richer feature interactions among multiple LoRAs for scalable LLMs. Specifically, we first fuse each LoRA representation to other LoRAs using a learnable component and then feed the output to a rotation matrix. This matrix involves learnable parameters that define the relative angular relationship between LoRA representations. Such a simple yet effective mechanism provides an extra degree of freedom, facilitating the learning of cross-LoRA synergies and properly tracking the challenging poor generalization and underfitting issues as the number of LoRA grows. Extensive experiments on 6 public benchmarks across 21 tasks show the effectiveness of our \textit{RadarGate} for scaling LoRAs. We also provide valuable insights, revealing that the rotations to each pair of representations are contrastive,  encouraging closer alignment of semantically similar representations during geometrical transformation while pushing distance ones further apart. We will release our code to the community.

\end{abstract}

\vspace{-15pt}
\section{Introduction}
\vspace{-5pt}

Scaling large-language models (LLMs) to adapt to diverse downstream tasks is non-trivial in real-world applications \cite{DBLP:conf/icml/Zhao0CWAT24}. However, the increasing size and complexity of these models pose significant challenges in terms of computational resources and training efficiency. To alleviate considerable computation costs of full fine-tuning, Low-Rank Adaptation (LoRA)  \cite{hu2022lora} has emerged as a parameter-efficient solution that freezes weights of the pre-trained model and injects trainable low-rank components. Meanwhile, the increasing demand to simultaneously handle various domain-specific tasks highlights the need for generalization and scalability of LLMs \cite{DBLP:journals/corr/abs-2406-11424}. Towards this direction, the concept of Mixture of Experts (MoE)~ \cite{zhou2022mixture, jacobs1991adaptive} was introduced
to substantially scale up LLM's capacity for various tasks. Therefore, marrying LoRA with MoE \cite{Agiza_2024_CVPR}, i.e., LoRA-MoE, offers significant potential for parameter-efficient and scalable LLMs, routing inputs to the best LoRAs for different tasks and thus facilitating various parameter-efficient adaptations.

Existing gating mechanisms of LoRA-MoE can be categorized into rule-based and learnable ones. Specifically, the rule-based gating networks, such as LoraHub \cite{huang2024lorahub}, PEMs \cite{Zhang2023ComposingPM} and Arrow \cite{DBLP:conf/icml/OstapenkoSPCRCS24}, heavily rely on pre-defined templates or mathematical formulations for the compositions of LoRA experts. In contrast, the learnable methods, such as HydraLoRA \cite{tian2024hydralora}, MoLE \cite{DBLP:conf/iclr/WuHW24}, OMoE \cite{Feng2025OMoEDM}, explore cutting edge learning techniques to adaptively active LoRA experts for downstream tasks. These methods greatly advance the state of the art of LoRA-MoE, efficiently facilitating scalable LLMs.

\begin{wrapfigure}[25]{r}{0.5\textwidth}
    \centering 
    \captionsetup[subfigure]{skip=0pt}
    
    \begin{minipage}{\linewidth}
        \centering

        \begin{subfigure}{\linewidth}
            \centering
            
            \includegraphics[height=2.275cm]{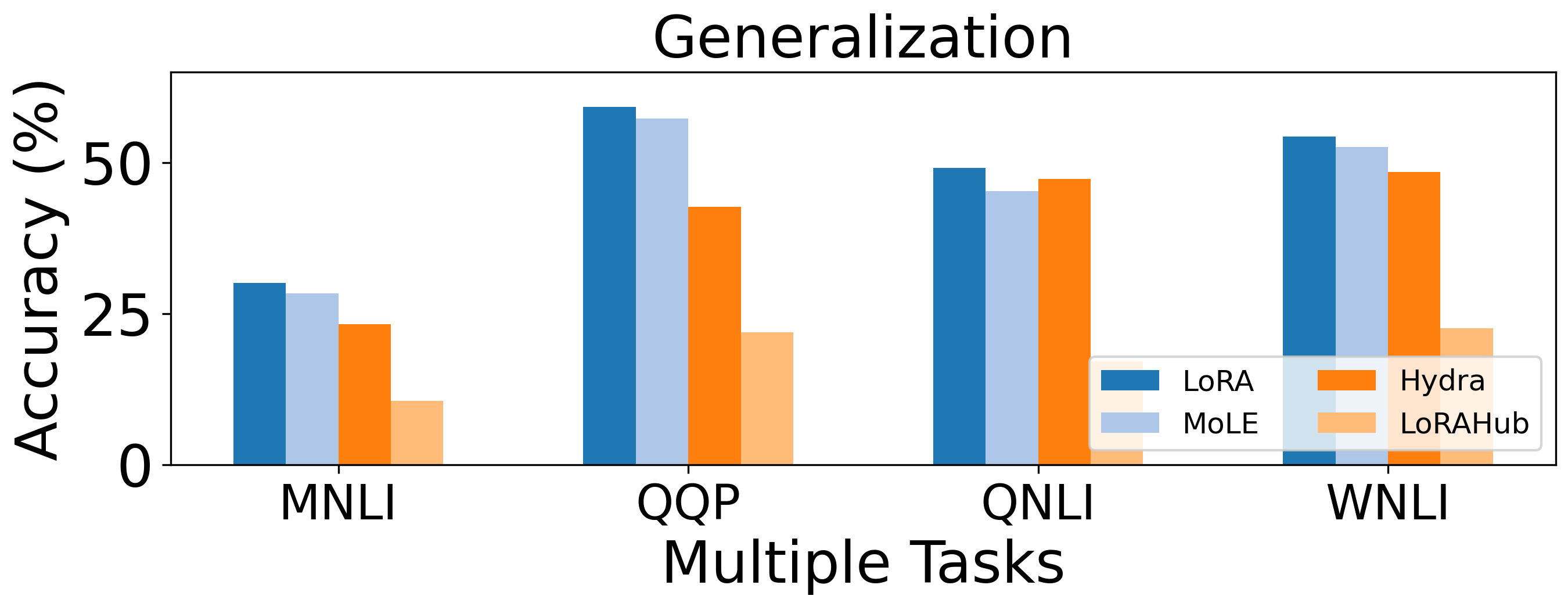}
            \caption{}
            \label{fig:subfig_a}
        \end{subfigure}

        \begin{subfigure}{\linewidth}
            \centering
            \includegraphics[height=1.95cm]{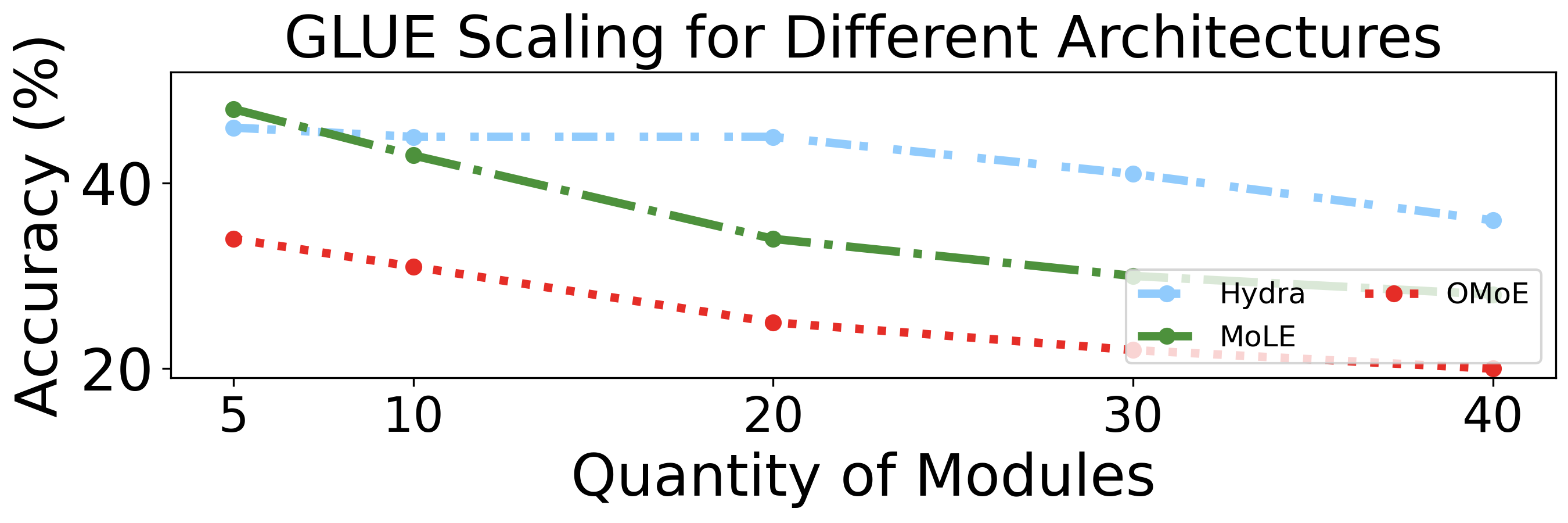}
            \caption{}
            \label{fig:subfig_b}
        \end{subfigure}

        \begin{subfigure}{\linewidth}
            \centering
            \includegraphics[height=1.95cm]{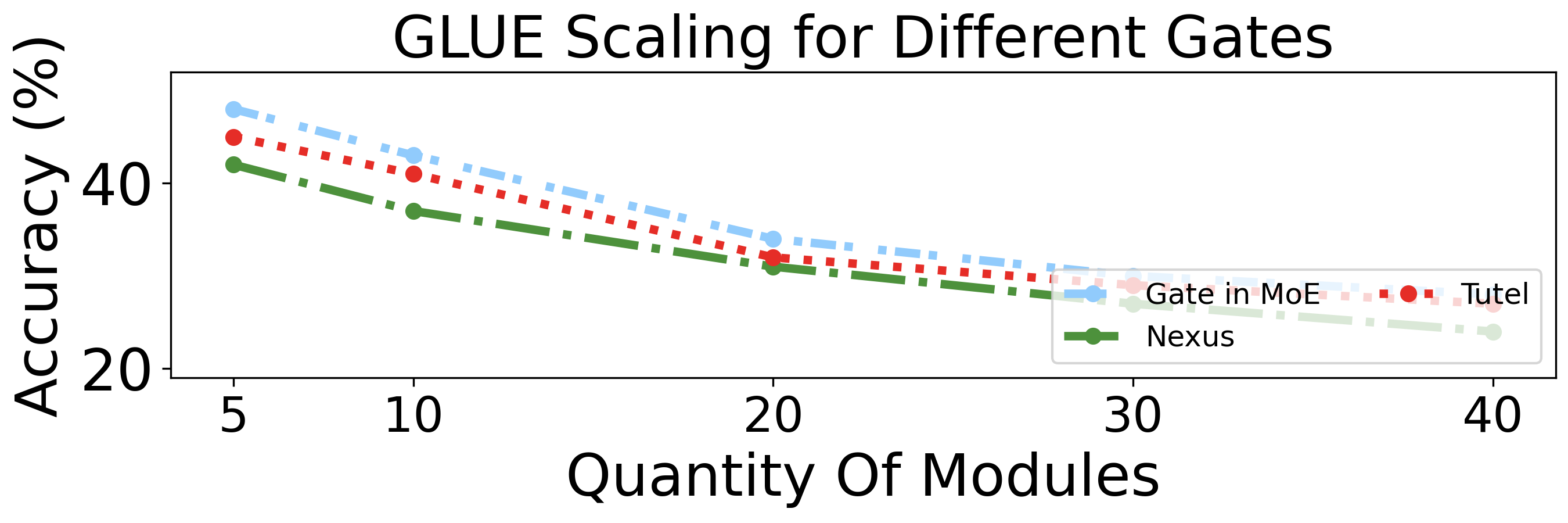}
            \caption{}
            \label{fig:subfig_c}
        \end{subfigure}
    \end{minipage}

    \vspace{-2mm}
    \caption{(a) Composable LoRA-MoE performs even worse than vanilla LoRA. (b) Poor generalization of different LoRA-MoE architectures as the number of LoRA grows. (c) Underfitting of various gating methods as the LoRA scales up.}
    \label{fig:main_stacked_figure}
\end{wrapfigure}

However, as the number of LoRA grows, existing LoRA-MoE gating methods may limit the LLMs' scalability and face two critical challenges regarding generalization and underfitting. Figure \ref{fig:main_stacked_figure} (\subref{fig:subfig_a}) demonstrates that existing gating methods in Hydra \cite{tian2024hydralora} and Lorahub \cite{huang2024lorahub} architectures, which involve more than 5 LoRAs, perform worse than the vanilla LoRA \cite{hu2022lora}, when these methods are trained on FLAN and applied to MNLI, QQP and WNLI tasks \cite{socher2013recursive}. The poor generation of existing methods causes a limited ability to adapt to diverse tasks. Figure \ref{fig:main_stacked_figure} (\subref{fig:subfig_b}) shows that the accuracy of the same gating method \cite{DBLP:conf/iclr/WuHW24} on three different MoE architectures decreases sharply by 10\%, 20\%, and 14\%, respectively, as the number of LoRA experts increases from 5 to 40. Figure \ref{fig:main_stacked_figure} (\subref{fig:subfig_b}) demonstrates that the performance of three different gating methods on the same MoE architecture (MoLE) \cite{DBLP:conf/iclr/WuHW24} also rapidly declines when increasing the LoRA number. This suggests that large-scale LoRAs may lead to suboptimal routing decisions due to poor generalization. Figure \ref{fig:main_stacked_figure} (\subref{fig:subfig_c}) shows that the training of the gating networks will be much slower and unstable when the number of LoRAs is increased from 5 to 40, also indicating underfitting issues when scaling LoRAs.

There is a line of work that attempts to achieve scalable LoRA-MoE. Early studies improved the scalability by reducing the computation and storage costs \cite{Hwang2023PregatedMA} or promoted the generations by mitigating interference among tasks \cite{DBLP:conf/nips/ZhuZWWLWD22}. A top-$k$ approach ExpertChoice \cite{DBLP:conf/nips/ZhouLLDHZDCLL22} reduces convergence time with a load balance mechanism. Recently, an LoRA library \cite{DBLP:conf/icml/OstapenkoSPCRCS24} was proposed to achieve zero-shot routing for better generalization. MoDE \cite{ning2024mode} introduced a flexible adapter to facilitate multi-task LLMs. Mostly related studies to our work are newly emerged Nexus \cite{DBLP:journals/corr/abs-2408-15901} and MoLE. Nexus is an enhanced MoE architecture that relies on adaptive routing to reduce the training cost of MoEs and enable efficient adaptation to new tasks. MoLE can dynamically compose multiple trained LoRAs for better generalization. Although effective, we empirically show that they still suffer from three challenges as the number of LoRA grows (see details in Section \ref{sect-comparison}).    

This paper proposes \textit{RadarGate}, a novel geometrical gating method that addresses the above two challenges by introducing rotational interactions of LoRA's representations for scalable LLMs. Our \textit{RadarGate} consists of two key components: a \textit{RotationGate} and a \textit{StretchGate}. The \textit{RotationGate} first learns the complex interactions between LoRAs represented by angles between LoRAs. The output of the \textit{RotationGate} will be fed into \textit{StretchGate}, which further assigns the weights for each LoRA representation. A weight indicates the importance of the LoRA for the task. Experiments show the effectiveness of our \textit{RadarGate}. The main contributions of our work are summarized as follows.

\begin{itemize}
    \item We propose \textit{RadarGate},  
    a novel geometrical gating method that introduces rotational operations of LoRAs representations for scalable LLMs, aiming to boost the expressiveness and facilitate richer feature interactions among multiple LoRAs. Such a straightforward yet effective method provides an extra degree of freedom beyond the weighted-sum mechanisms, thus facilitating the learning of cross-LoRA synergies as the number of LoRA grows.
    \item We present two key components \textit{RotationGate} and \textit{StretchGate}, where the former dynamically generates angles between LoRAs, and the latter further refines this interaction. Such a geometrical transformation properly addresses the two challenges of scalable LoRA-MoE.
    \item We conduct extensive experiments to show the effectiveness of our \textit{RadarGate}, and provide  valuable insights of scalable LLMs. For example, \textit{we observe that the rotations to each pair of representations are contrastive, encouraging closer alignment of semantically similar representations while pushing distant ones further apart,} and such an interesting finding can interpret that the rotations help to converge the representations as LoRA scales up. 
\end{itemize}

\section{Related Work}
\vspace{-5pt}
\noindent

\textbf{Rule-based Gating Methods} 
use pre-defined formulations, such as subspace functional decomposition and recomposition~\cite{DBLP:journals/corr/abs-2402-16843,ludziejewski2024scaling}, gradient-free arithmetic averaging~\cite{huang2024lorahub,DBLP:conf/nips/ZhouLLDHZDCLL22}, and specific arithmetic functions~\cite{liang2024matrix,DBLP:conf/iclr/IlharcoRWSHF23,Zhang2023ComposingPM,DBLP:conf/icml/OstapenkoSPCRCS24,DBLP:conf/nips/NieLFZMLZL024} for LoRA activation. These methods typically follow fixed logic or heuristics to guide expert selection, offering simplicity and low overhead. Although effective, pre-defined rules may face challenges for the unseen tasks due to their limited flexibility. Different from rule-based ones, our \textit{RadarGate} enhances the capability to coordinate different inputs through learnable magnitude scaling and angular rotation modules..

\noindent
\textbf{Learnable Gating Methods} explore deep learning techniques~\cite{DBLP:conf/iclr/ShazeerMMDLHD17,tian2024hydralora,DBLP:conf/iclr/WuHW24,DBLP:journals/corr/abs-2408-15901,DBLP:conf/nips/YunCPWBZXLC24,DBLP:conf/nips/HanNHHS24} or optimization strategies~\cite{he2021fastmoe,hwang2023tutel,aminabadi2022deepspeed,Hwang2023PregatedMA,DBLP:journals/corr/abs-2408-15901} for adaptive selection of LoRAs. These methods aim to dynamically assign experts based on input features or training signals, improving adaptability across domains. Different from these methods, our \textit{RadarGate} expands the degrees of freedom of gating architectures by incorporating a rotation module, achieving better fitting capability and generalization, and demonstrates superior performance in large-scale LoRA scenarios.

\vspace{-10pt}
\section{Motivation}
\vspace{-5pt}

In this section, we first detail existing gating architectures (sec~\ref{sec:composable_loras_architecture_long_eq}), then present two key observations on fitting and generalization. We provide theoretical explanations for these observations (sec~\ref{subsec:obs}) and analyze why the scaling performance of LoRA modules degrades based on our findings (sec~\ref{subsec:scaling}).

\vspace{-5pt}
\subsection{Composable LoRAs Architecture}
\vspace{-5pt}
\label{sec:composable_loras_architecture_long_eq}

This subsection outlines the forward computation and backpropagation of the composable LoRA architecture~\cite{DBLP:conf/iclr/WuHW24}. We feed the input $\mathbf{x} \in \mathbb{R}^{1 \times d_{in}}$ to a neural module $W \in \mathbb{R}^{d_{in} \times d_{out}}$ of a pretrained model. The LoRA group involves $n$ LoRA modules and a gating module with top-$k$ activation ($A_i \in \mathbb{R}^{d_{in} \times r}, B_i \in \mathbb{R}^{r \times d_{out}}$, with rank $r \ll d_{in}, d_{out}$). We denote the output of this composable LoRAs as $\mathbf{y} \in \mathbb{R}^{1 \times d_{out}}$, which can be expressed by a weighted sum of the base model's output and the ones of selected LoRA modules. We give the formulation of $\mathbf{y}$ and the gate $\mathbf{g}$ in Equation \eqref{eq:long_output}.
\begin{equation}
\mathbf{y} = \mathbf{x} W +  \sum_{i=1}^n g_i \mathbf{v}_i, \mathbf{v}_i=\mathbf{x} A_i B_i
\text{ and } \mathbf{g} = \mathrm{topk}\left(\mathrm{softmax}\left( \frac{\mathbf{x} \mathbf{\theta}}{\tau} \right)\right)=[g_1, \dots, g_n],
\label{eq:long_output}
\end{equation}
where $\mathbf{v}_i \in \mathbb{R}^ {d_{out}}$ represents the output of the $i$-th LoRA module, $\mathbf{\theta} \in \mathbb{R}^{d_{in} \times n}$ is the learnable parameter of the gating module, and $g_i \in \mathbb{R}$ is the corresponding gating weight. These weights are derived from the input $\mathbf{x}$ via a gating network. Here $\mathbf{\theta}$ is a learnable projection matrix mapping the input $\mathbf{x}$ to $n$ gating scores (logits), and $\tau$ is the softmax temperature. The $\mathrm{topk}(\cdot)$ function is used to renormalize the top-$k$ weights, while simultaneously setting the remaining gating weights to zero.

\vspace{-5pt}
\subsection{Observation}\label{subsec:obs}
\vspace{-5pt}
In this subsection, we will present two observations regarding fitting and generalization of the LoRAs' gating module, and then provide our insights of the underlying cause from a theoretical perspective.
\obsbox{
\textit{\textbf{Obs I (Underfitting)}: \label{obs: Underfitting}
Existing gating mechanisms struggle to capture complex patterns of ideal $g_i^*(\mathbf{x})$ distribution within the convex cone $\mathcal{H}$, resulting in an underfitting ensemble of multiple LoRAs. }
}

We consider the scenario when the fitting target $\mathbf{y}_{\text{target}}$ is inside the convex cone $\mathcal{H}$ of LoRA representations. We denote the input of the LoRA-based MoE as $\mathbf{x}$. Assume there exists a set of ideal, non-negative weights $g_i^*$ that sum to 1, such that the current fitting target can be expressed as.
\begin{equation}
\Delta\mathbf{y}_{\text{target}} \approx \sum_{i=1}^{n} g_i^* \mathbf{v}_i, g_i^* \ge 0, \sum_{i=1}^{n} g_i^* = 1, \mathbf{v}_i=\mathbf{x} A_i B_i
\label{eq:ideal_combination}
\end{equation}

\obsbox{
\textit{\textbf{Obs II (Poor Generalization)}: \label{obs: Undergeneralization}
Existing gating methods heavily rely on a weighted sum of LoRA representation, thus degrading generalization as the LoRA output is limited in the convex cone $\mathcal{H}$.}}

Each LoRA representation $\mathbf{v}_i$ modifies the input $\mathbf{x}$ in a specific semantic direction.
As shown in Equation~\eqref{eq:long_output}, these representations are weighted and summed to form a composite representation $\mathbf{y}$.
The coefficient $g_i$ only scales the magnitude of $\mathbf{v}_i$', leaving directions unchanged in the vector space.
Thus, $\mathbf{y}$ will be limited to the convex cone $\mathcal{H}$ spanned by non-negative linear combinations of $\{\mathbf{v}_i\}$:
\begin{equation}\label{eq:generalization_conv}
    \mathcal{H} = \text{conv}(\{\mathbf{v}_1, \dots, \mathbf{v}_n\}) = \left\{ \sum_{i=1}^{n} g_i \mathbf{v}_i \mid g_i \ge 0, \sum_{i=1}^{n} g_i = 1 \right\}
\end{equation}
When the target output $\mathbf{y}_{\textbf{target}} \notin \mathcal{H}$, it becomes infeasible to achieve an adequate fit using only magnitude scaling as the single degree of freedom, thereby leading to suboptimal generalization.

\vspace{-5pt}
\subsection{Scaling Issues}\vspace{-5pt}\label{subsec:scaling}

As shown in Fig.~\ref{fig:main_stacked_figure}(\subref{fig:subfig_b})(\subref{fig:subfig_c}), the model accuracy sharply degrades along with LoRA modules scaling up. From the aforementioned \textit{\textbf{Obs I}} and \textit{\textbf{Obs II}}, we provide our insights into scalable LoRAs as follows:

\obsbox{\textbf{Summary.} As the LoRA modules (module numbers and parameters) scale up, the pattern of the target magnitude weights $\mathbf{g}^*$ that existing gating methods need to fit becomes more complex. Moreover, the expressiveness of LoRA representations $\mathbf{v}_i$ remains confined within the convex cone $\mathcal{H}$. Consequently, underfitting and poor generalization become more pronounced as the scale of LoRA modules increases.}

\begin{figure}
\centering
\includegraphics[width=1.0\textwidth]{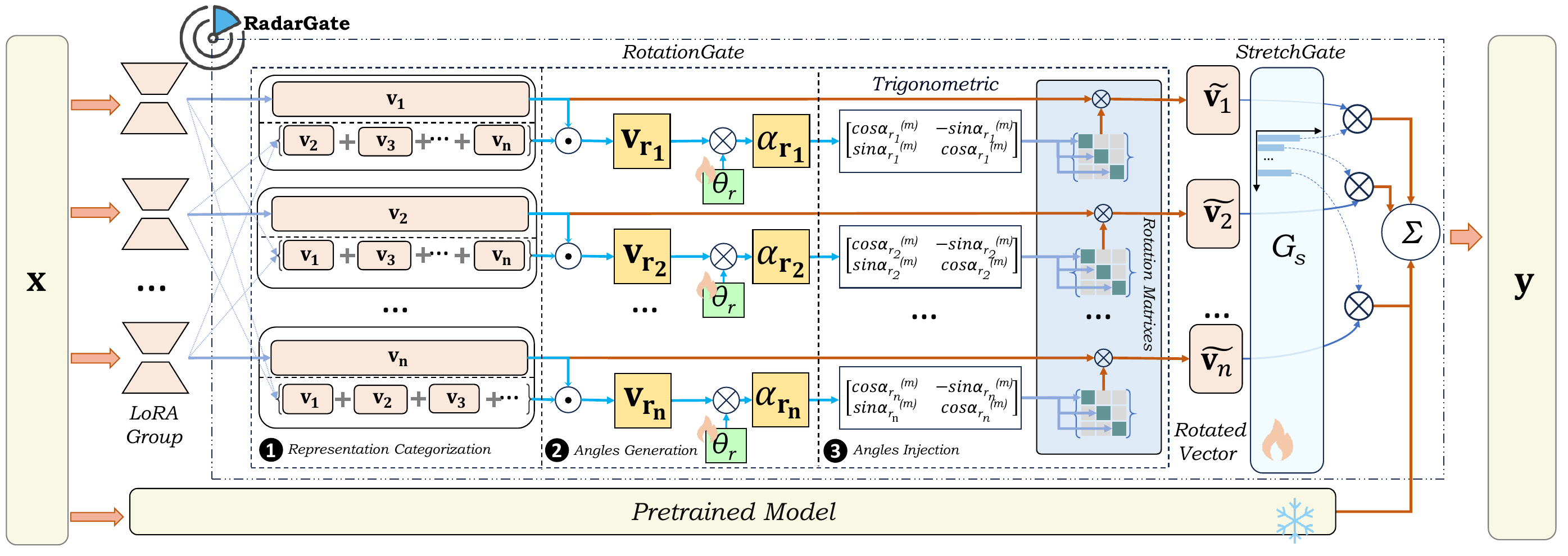}
\caption{\textbf{Workflow of the proposed \textit{RadarGate}.} Two key ingredients are \textit{RotationGate} and \textit{StretchGate}. \textit{RotationGate} takes LoRA representations as inputs and then proceeds to three steps, including 1) LoRA representation categorization, 2) rotation angles generation, and 3) angles injection. The rotated LoRA representation will be stretched in magnitude by \textit{StretchGate} to get the output.}

\label{fig:workflow}
\vskip -0.2in
\end{figure}

\vspace{-10pt}
\section{Our \textit{RadarGate} Method}
\vspace{-5pt}

Figure \ref{fig:workflow} illustrates the workflow of our \textit{RadarGate}. The proposed \textit{RadarGate} consists of two key components including \textit{RotationGate} and \textit{StretchGate}. In this section, we will introduce the workflow of our \textit{RadarGate} (sec~\ref{subsec:workflow}) and theoretically explain how \textit{RadarGate} improves fitting and generalization when LoRA modules scale up (sec~\ref{subsec:theory}). Subsequently, we will demonstrate that the computational and memory overhead incurred by our novel module is negligible (sec~\ref{subsec:complex}).   
\vspace{-5pt}
\subsection{Workflow}\label{subsec:workflow}
\vspace{-5pt}
This subsection details \textit{RadarGate}'s integration of LoRA submodule outputs per layer through angle and magnitude adjustments, as defined below:
\vspace{-3pt}
\begin{equation}
\mathbf{y} = \mathbf{x} W + \sum_{i=1}^n g_i \tilde{\mathbf{v}}_i, \tilde{\mathbf{v}}_i=  G\left( Map(\mathcal{L}_i), \mathbf{x}\right),
\label{eq:9}
\end{equation}
\vspace{-3pt}

here $g_i$ is the output of  \textit{StretchGate}, as shown in Equation \eqref{eq:long_output}. The operator $Map(\mathcal{L}_i)$ is defined as:
\begin{equation}
Map(\mathcal{L}_i) = (A_iB_i, \mathcal{L} - \{A_iB_i\}),\mathcal{L}=\{A_iB_i | i=1,2,...,n\}.
\label{eq:10}
\end{equation}
Here, the minus sign $-$ in $\mathcal{L} - \{A_iB_i\}$ denotes the set difference operation between two sets. $Map(\mathcal{L}_i)$ constructs binary relations between submodule $A_iB_i$ and its reference set $\mathcal{L} - \{A_i B_i\}$.
The function $G(\cdot)$ rotates each submodule’s output $\mathbf{v}_i$ using their relations.
\begin{equation}
G(Map(\mathcal{L}_i),\mathbf{x}) = \mathbf{x}A_iB_i \times \mathcal{R}_i(Map(\mathcal{L}_i))\triangleq \mathbf{v}_i \times \mathcal{R}_i.
\label{eq:gating}
\end{equation}
Here, $\times$ denotes the multiplication operation between matrices, for each LoRA representation $\mathbf{v}_i$, we compute a rotation matrix $\mathcal{R}_i$ with angles governed by $\theta_r$ to learn relative angular relationships. If we denote $d_{in}\triangleq d$, then the rotation matrix is:
\begin{equation}
\mathcal{R}_i = 
\begin{pmatrix}
R_i^{(0)} & 0 & \cdots & 0 \\ 
0 & R_i^{(1)} & \cdots & 0 \\ 
\vdots & \vdots & \ddots & \vdots \\ 
0 & 0 & \cdots & R_i^{(\frac{d-1}{2})} \\ 
\end{pmatrix},
R_i^{(m)} = 
\begin{pmatrix}
\cos \alpha_{r_i}^{(m)} & -\sin \alpha_{r_i}^{(m)} \\ 
\sin \alpha_{r_i}^{(m)} & \cos \alpha_{r_i}^{(m)} \\ 
\end{pmatrix},
\label{eq:11}
\end{equation}
Here, $\alpha_{r_i}^{(m)}$ is the $m$-th component of the rotational control factor $\alpha_{r_i} \in \mathbb{R}^{\frac{d}{2}}$, which is calculated as:
\begin{equation}
\alpha_{r_i} = \left(\mathbf{x}\times Map(\mathcal{L}_i)^{(0)}\right) \odot \left(\mathbf{x}\times \sum_{A_jB_j \in S_i} A_jB_j\right) \times \theta_r, S_i = \{A_jB_j | A_jB_j \in Map(\mathcal{L}_i)^{(1)}\}
\label{eq:13}
\end{equation}
where $\odot$ denotes the element-wise Hadamard product, $Map(\mathcal{L}_i)^{(t)}$ is the $t$-th submodule output in $A_iB_i$'s reference set. Submodule outputs and references undergo Hadamard product, then matrix multiplication with learnable $\theta_r$, injecting relative angular information to update $\mathbf{v}_i$. It should be noted that we use the map to construct the binary relation for the rotating reference frame because the absolute value of the angle is meaningless, and only the relative value of the rotation angle matters. We provide more details about the workflow of the proposed \textit{RadarGate} in Appendix \ref{app:wor}.

\vspace{-5pt}
\subsection{Theoretical Demonstration}\label{subsec:theory}
In this subsection, we will explain the reasons why our \textit{RadarGate} can effectively improve the underfitting phenomenon and enhance the generalization ability from the theoretical perspective.\footnote{Proofs of Lemma 1 and 2 are provided in Appendix ~\ref{lem_proof}}

\textbf{Mitigating Underfitting.}

\obsbox{
\begin{lemma}\label{lem:1}
    For nested function hypothesis spaces $\mathcal{K}_1 \subseteq \mathcal{K}_2$, the optimal fitting error $\mathcal{E}_t = \inf_{f \in \mathcal{K}_t} L(f, g^*)$ of the target function $g^*$ under the loss function $L$ necessarily satisfies $\mathcal{E}_2 \leq \mathcal{E}_1$.
\end{lemma}

}
    
Treating the gating architecture as a function space $\mathcal{K}_G$, existing gating $\mathcal{K}_{\text{gate}}$ and our \textit{RadarGate} $\mathcal{K}_{\text{ours}}$ satisfy $\mathcal{K}_{\text{gate}} \subseteq \mathcal{K}_{\text{ours}}$. Due to the added \textit{RotationGate} module, \textit{RadarGate} has an advantage for ideal mapping $g^*$ (Lemma~\ref{lem:1}). Specifically, during inference, LoRA representations are adjusted via magnitude scaling and vector rotation. During optimization, taking MSE loss as an example,
\begin{equation}
\mathcal{L}(x) = \left\| \Delta \mathbf{y}_{\text{target}}(\mathbf{x}) - \sum_{i=1}^n g_i(\mathbf{x}) \left( \mathbf{v}_i \mathcal{R}_i(\mathbf{x}; \theta_r) \right) \right\|^2,
\end{equation}
\textit{RadarGate} optimizes by adjusting gradients of scaling parameter $\theta$ (i.e., $\frac{\partial \mathcal{L}}{\partial g_i} \frac{\partial g_i}{\partial \theta}$) and rotation parameter $\theta_r$ (i.e., $\frac{\partial \mathcal{L}}{\partial \mathcal{R}_i} \frac{\partial \mathcal{R}_i}{\partial \alpha_{r_i}} \frac{\partial \alpha_{r_i}}{\partial \theta_r}$). Higher freedom and extra optimization paths enhance model flexibility in fitting data, alleviating underfitting from $\mathcal{K}_{\text{gate}}$'s insufficient expressiveness.

\textbf{Improving Generalization.}

\obsbox{
\begin{lemma}\label{lem:2}
    Define the fixed output space $\mathcal{H} = \left\{ \sum_{i=1}^n \alpha_i v_i \mid \alpha_i \ge 0,\ \sum_{i=1}^n \alpha_i = 1 \right\}$ with fixed basis vectors $\{v_i\}$. Transforming $v_i$ via input-dependent rotation $R_i(x)$ gives $\tilde{v}_i(x) = v_i R_i(x)$. Define dynamic output space $\mathcal{H}'(x) = \left\{ \sum_{i=1}^n \alpha_i \tilde{v}_i(x) \mid \alpha_i \ge 0,\ \sum_{i=1}^n \alpha_i = 1 \right\}$. The union $\mathcal{S} = \bigcup_x \mathcal{H}'(x)$ strictly contains $\mathcal{H}$, i.e., $\mathcal{S} \supset \mathcal{H}$.

\end{lemma}
}    
According to Equation~\ref{eq:generalization_conv}, existing gating outputs are confined to fixed convex cone $\mathcal{H}$, so $\Delta y_{\text{target}} \notin \mathcal{H}$ cannot be fitted. \textit{RadarGate} introduces input-dependent rotation $R_i(x)$ (Lemma~\ref{lem:2}) to expand the space to dynamic convex cones $\mathcal{H}'(x)$. According to Lemma~\ref{lem:2}, we have $\bigcup_x \mathcal{H}'(x) \supset \mathcal{H}$. Thus $\Delta y_{\text{target}} \in \bigcup_x \mathcal{H}'(x) \setminus \mathcal{H}$ (outside $\mathcal{H}$) can still be fitted. This rotation-induced basis alteration and space expansion improve the generalization of gating modules on various tasks outside the cone $\mathcal{H}$.

\textbf{Enhancing Scaling.} As the LoRA modules scale up, the complexity of approximating the ideal weights $g_i^*(\mathbf{x})$ grows, and the limitations of the fixed convex cone $\mathcal{H}$ become more pronounced. Combining Lemmas~\ref{lem:1} and~\ref{lem:2} with the preceding analyses, we can summarize our insights as follows: 
\obsbox{
\textbf{Summary.} \textit{RadarGate}'s rotational mechanism $\mathcal{R}_i(\mathbf{x})$ expands the hypothesis space to $\mathcal{K}_{\text{ours}} \supset \mathcal{K}_{\text{gate}}$ and the effective output space to $\bigcup_x \mathcal{H}'(x) \supset \mathcal{H}$, providing the necessary flexibility to better fit complex $g_i^*(\mathbf{x})$ and generalize to a wider range of target outputs, thereby sustaining performance at larger scales.
}
\vspace{-5pt}
\subsection{Computational and Memory Complexity}\vspace{-5pt}\label{subsec:complex}

For a sequence input dimension of \( L \times d_{\text{in}} \), we decompose the parameters of \(\textit{RotationGate}\) into two matrices of \( d_{\text{in}} \times r_{\text{a}} \) and \( r_{\text{a}} \times d_{\text{in}} \) through low-rank factorization. When these parameters satisfy \( n, r, k, r_{\text{a}} \ll \min\{d_{\text{in}}, d_{\text{out}}\} \), the computational and memory complexity $O$ and $M$ can be simplified as:  
\begin{equation}
O_{\text{s}} = O\left(L \cdot \min\{d_{\text{in}}, d_{\text{out}}\}\right) = O_{\text{r}}, \quad M_{\text{s}} \approx M_{\text{r}}
\end{equation}
This result indicates that the computational and memory complexities of \(\textit{RadarGate}\) are asymptotically of the same order of magnitude as those of existing gating methods.\footnote{For detailed theoretical derivation, please refer to Appendix \ref{app:complex}.}

\vspace{-10pt}
\section{Experiments}
\vspace{-5pt}
\subsection{Experimental Setup}
\vspace{-5pt}
\textbf{Environment.}  

All experiments are conducted on an Ubuntu 20.04.5 LTS server with PyTorch, featuring 64GB RAM, an Intel Xeon Silver 4210 CPU, and dual NVIDIA A40 GPUs (48GB each).

\textbf{Datasets.}  
We use nine datasets from the v1 version of the FLAN dataset \cite{weifinetuned} as the base LoRA module training and test set for the LoRA module independence experiments. In later experiments, the v2 FLAN dataset is categorized by language, mathematics, reasoning, and translation for LoRA module training. Our method \textit{RadarGate} is compared with baselines on six large-scale comprehensive benchmarks, including NLP benchmarks: GLUE \cite{socher2013recursive}, MMLU \cite{hendryckstest2021}, WMT14 \cite{bojar-EtAl:2014:W14-33}; mathematics benchmarks: MATH \cite{lightman2023lets}, GSM8K \cite{deepseekai2025deepseekr1incentivizingreasoningcapability} and the science benchmark GPQA \cite{rein2024gpqa}.

\textbf{Baselines.}  

\textit{RadarGate} is compared with multi-LoRA/gating architectures across two categories: 1) Rule-based (LoraHub \cite{huang2024lorahub}, PEMs \cite{Zhang2023ComposingPM}, Arrow \cite{DBLP:conf/icml/OstapenkoSPCRCS24}, direct LoRA) and 2) Learnable (HydraLoRA \cite{tian2024hydralora}, MoLE \cite{DBLP:conf/iclr/WuHW24}, OMoE \cite{Feng2025OMoEDM}). Evaluated gating mechanisms include Stretch-Only (MoLE \cite{DBLP:conf/iclr/WuHW24}), Rotation-Only (from \textit{RadarGate}), Nexus \cite{DBLP:journals/corr/abs-2408-15901}, and Tutel \cite{hwang2023tutel}. \textit{RadarGate} integrates Stretch-Only Gate and Rotation-Only Gate to jointly improve generalization and scalability.

\textbf{Metric.}  
Accuracy serves as the primary metric. Predictions are evaluated on the aforementioned benchmarks using their standard protocols. Overall accuracy is calculated by matching predictions against reference answers.

\vspace{-5pt}
\subsection{Training Details}  
\vspace{-5pt}

\textit{RadarGate} employs frozen pretrained weights with LoRA standard initialization (rank $8$, LoRA \( \alpha = 32 \), learning rate \( 1e{-4} \), batch size $4$, dropout = $0.1$). Gate training uses identical hyperparameters but fewer parameters than baselines, with frozen pretrained/LoRA weights isolating gating effects. Inference maintains frozen parameters and benchmark consistency \(\text{max\_new\_tokens} = 512 \).

\vspace{-5pt}
\subsection{Performance}
\vspace{-5pt}
\label{sect-comparison}
We validate \textit{RadarGate} across multi-LoRA architectures and provide experimental insights\footnote{Details about experiments are given in Appendix \ref{app:exp_gen} and \ref{app:exp_sca}.}. Figure \ref{fig:ablation_studies}(\subref{fig:fitting_a})  confirms it has excellent fitting capability under matched training/test conditions, while Table \ref{fig:gen} demonstrates superior generalization. Figure \ref{fig:main} reveals scalability improvements with increasing module and parameter, and Figure \ref{fig:ablation_studies}(\subref{fig:ablation_a})(\subref{fig:ablation_b})  shows the results of the ablation study. 

\begin{figure}[htpb]
\setlength{\abovecaptionskip}{0cm}
\setlength{\belowcaptionskip}{0cm}
  \centering
  \subfloat[]{\includegraphics[width=0.37\textwidth,keepaspectratio]{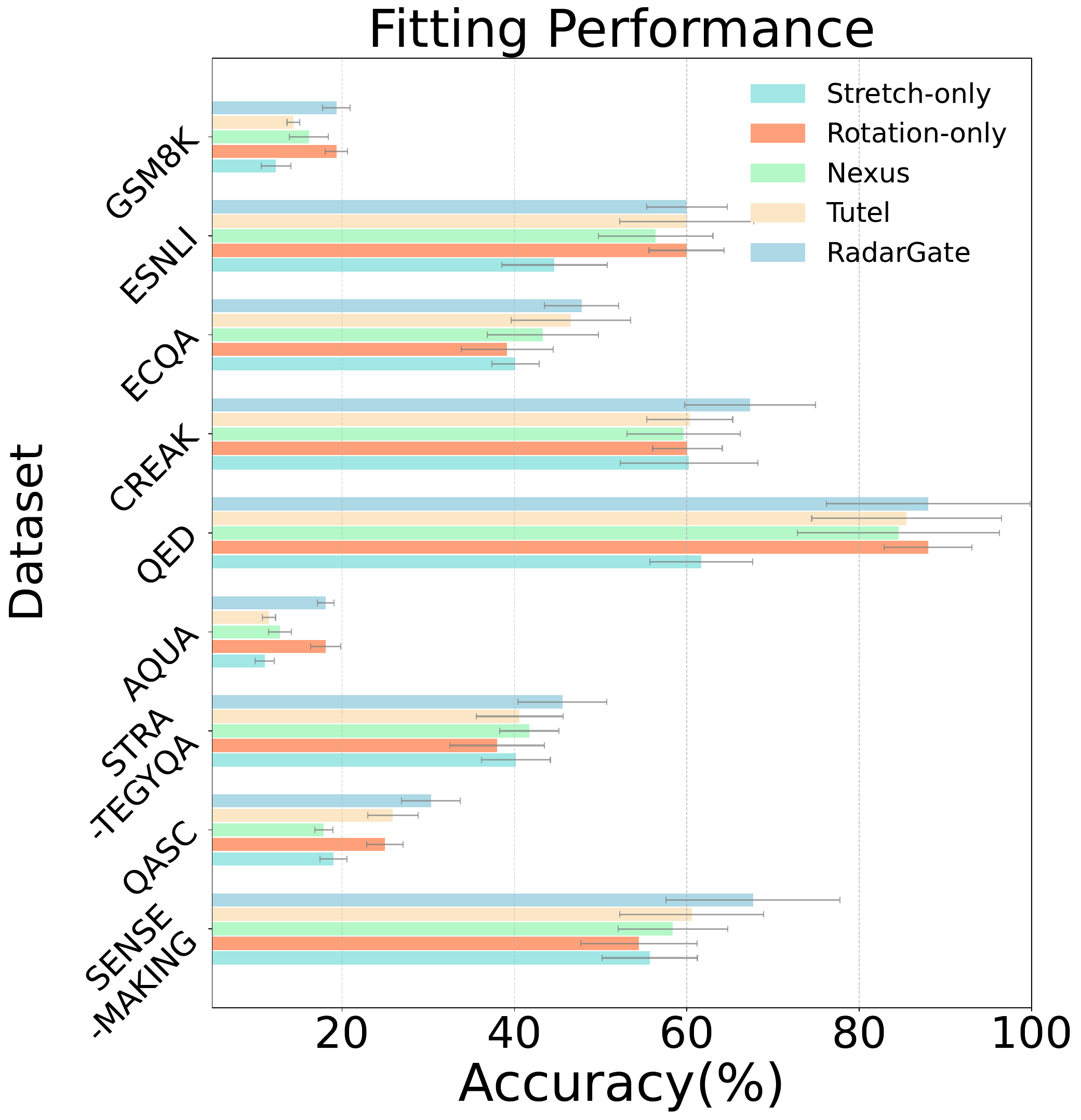}\label{fig:fitting_a}} 
  \hfill
  \subfloat[]{\includegraphics[width=0.31\textwidth]{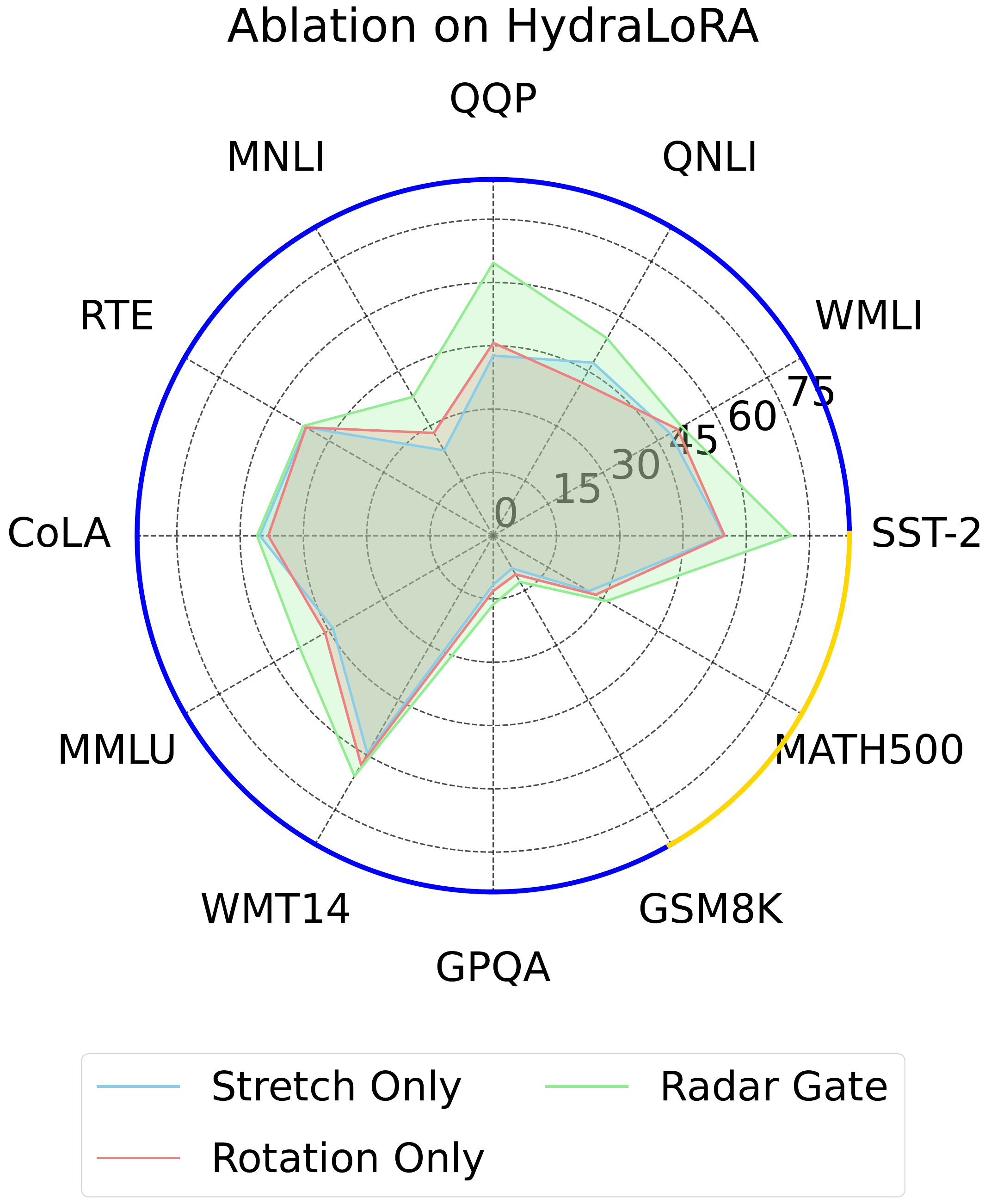}\label{fig:ablation_a}}
  \hfill
  \subfloat[]{\includegraphics[width=0.31\textwidth]{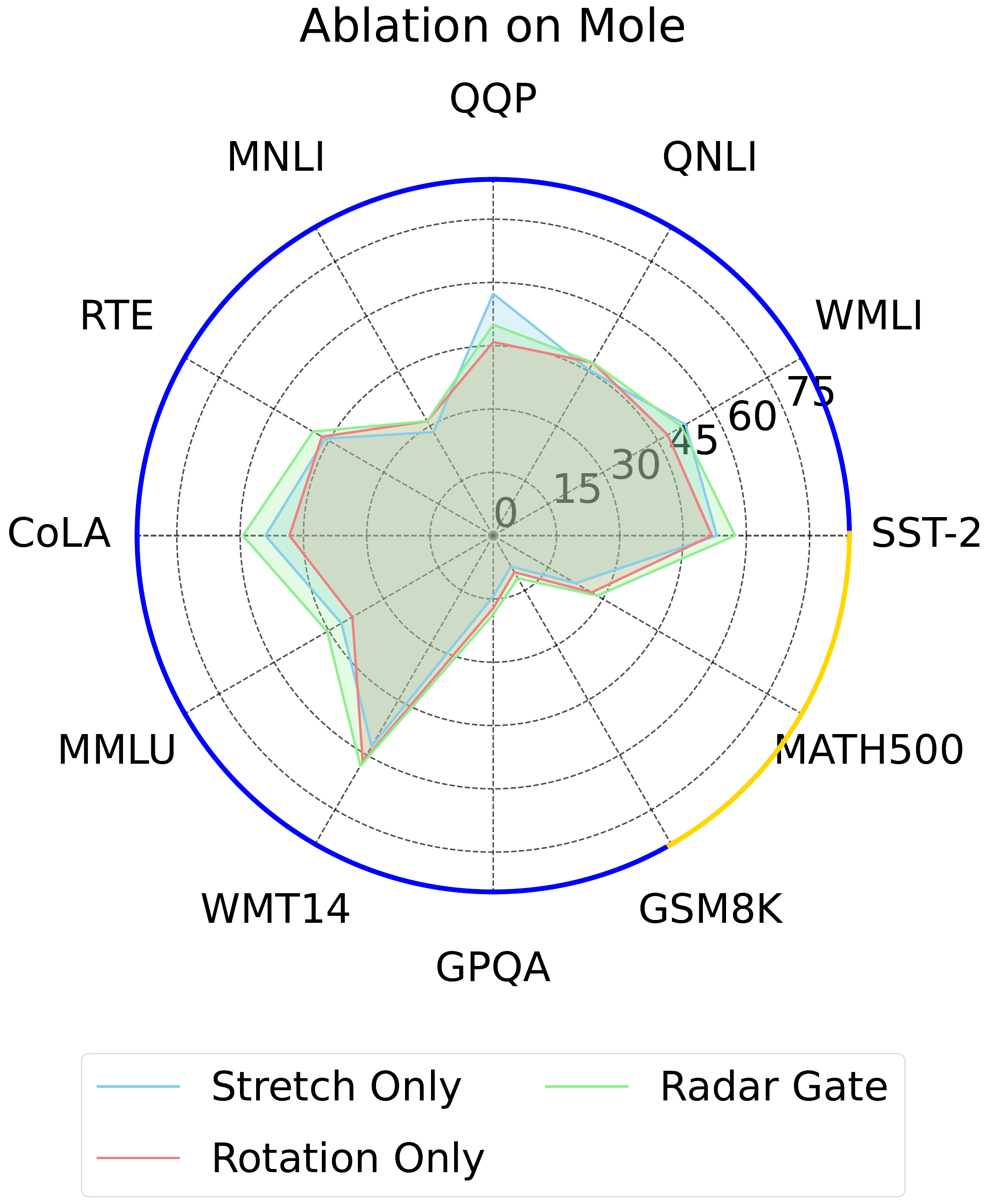}\label{fig:ablation_b}}
  \caption{\textbf{Performance on Fitting and Ablation.} 
Figure (a) shows performance of fitting capability on same-source training/test sets, while Figures (b) and (c) show ablation results for \textit{RadarGate}'s \textit{StretchGate} and \textit{RotationGate} components.} 
  \label{fig:ablation_studies}
\end{figure}

\textbf{Experiments on Fitting.}
We validate \textit{RadarGate} on nine tasks to assess the fitting capability under three frameworks. Figure \ref{fig:ablation_studies}(\subref{fig:fitting_a}) shows \textit{RadarGate} outperforms baselines in MoLE across nine tasks, we make the following observation:

\textbf{Obs.\ding{182} Baseline methods exhibit lower fitting performance than \textit{RadarGate} when training and test sets are from the same source.} As shown in Figure \ref{fig:ablation_studies}(\subref{fig:fitting_a}), our method achieves state-of-the-art performance in over 90\% of tasks compared to other baseline methods, and in some tasks, it even outperforms certain baselines by more than 20\% in performance. This indicates that \textit{RadarGate} has a stronger fitting capability.

\begin{table*}[t]
\setlength{\abovecaptionskip}{-0.1cm}
\setlength{\belowcaptionskip}{-0cm}
\caption{\textbf{Performance on Generalization.} Comparsions of our proposed \textit{RadarGate} with existing 4 baseline methods regarding generalization under 2 composable LoRAs architectures.}
\label{fig:gen}
\vskip 0.15in
\renewcommand{\arraystretch}{1.2}
\begin{center}
\begin{LARGE}
\begin{sc}
\resizebox{\textwidth}{!}{ 
\begin{tabular}{l@{\hspace{-12pt}}ccccccccccc} 
\toprule
{\fontsize{18}{16}\selectfont \textbf{\textsc{\multirow{3}{*}[-1.2ex]{Benchmark}}}} & {\fontsize{18}{16}\selectfont \multirow{3}{*}[-1.2ex]{\textbf{Task}}} & \multicolumn{2}{c}{ \textbf{Rule-based}} & \multicolumn{8}{c}{\textbf{Learnable}} \\
\cmidrule(r){3-4}\cmidrule(r){5-12}
& & \multirow{2}{*}[-1.0ex]{\textbf{Lorahub}} & \multirow{2}{*}[-1.0ex]{\textbf{Arrow}} & \multicolumn{4}{c}{\textbf{HydraLoRa}} & \multicolumn{4}{c}{\textbf{MoLE}}  \\
\cmidrule(r){5-8}\cmidrule(r){9-12}
& & & & \textbf{\makecell{Stretch-\\Only Gate}} & \textbf{Nexus} & \textbf{Tutel} & \textbf{\makecell{\textit{RadarGate}\\(ours)}} & \textbf{\makecell{Stretch-\\Only Gate}} & \textbf{Nexus} & \textbf{Tutel} & \textbf{\makecell{\textit{RadarGate}\\(ours)}} \\
\toprule
{\fontsize{18}{16}\selectfont\multirow{7}{*}{\textbf{GLUE}}}
& SST-2 &  \underline{16.67\%} & \textbf{18.33\%}  & 54.67\% & 53.35\% & \underline{58.33\%} & \textbf{70.67\%} &  53.01\% & 48.69\% & \underline{56.30\%} & \textbf{57.33\%} \\
& WNLI &  \textbf{22.62\%} & \underline{20.40\%}  & \underline{48.42\%}  & 44.89\% & 41.68\% & \textbf{51.58\%} &  \textbf{52.63\%} & 42.16\% & 48.11\% & \underline{51.58\%} \\
& QNLI &  \underline{17.11\%} & \textbf{19.42\%}  & \underline{47.33\%} & 46.16\% & 40.74\% & \textbf{54.00\%} &  \underline{45.33\%} & 36.92\% & 37.52\% & \textbf{47.33\%} \\
& QQP &  \textbf{21.92\%} & \underline{17.21\%}  & 42.67\%  & \underline{45.47\%} & 45.41\% & \textbf{64.67\%} &  \textbf{57.33\%} & 47.99\% & 47.87\% & \underline{50.00\%} \\
& MNLI &  \underline{10.59\%} & \textbf{9.61}\%  & 23.33\% & 22.98\% & \underline{24.06\%} & \textbf{38.00\%} &  28.33\% & 24.74\% & \underline{30.37\%} & \textbf{31.33\%} \\
& RTE &  \underline{18.67\%} & \textbf{21.33\%}  & \underline{51.33\%} & 48.68\% & 49.82\% & \textbf{52.00\%} &  46.00\%  & 40.32\% & \underline{47.12\%} & \textbf{49.33\%} \\
& CoLA &  \underline{19.67\%} & \textbf{22.33\%}  & \underline{55.33\%} & 50.34\% & 53.72\% & \textbf{56.00\%} &  54.00\% & \underline{55.18\%} & 49.12\% & \textbf{59.33\%} \\
\midrule
{\fontsize{18}{16}\selectfont \multirow{7}{*}{\textbf{MMLU}}}
& ARC-HARD &  \underline{12.33\%} & \textbf{19.67\%} & \underline{38.67\%} & 37.24\% & 36.37\% & \textbf{41.33\%} &  \underline{35.33\%} & \textbf{40.00\%} & 33.68\% & \textbf{40.00\%} \\
& SCI-MID &  \textbf{21.33\%} & \underline{19.33\%} & 55.56\% & \underline{57.76\%} & \textbf{58.89\%} & \textbf{58.89\%} &  \underline{41.11\%} & 37.16\% & \textbf{50.00\%} & \textbf{50.00\%} \\
& RACE &  \textbf{13.67\%} & \underline{13.00\%} & 32.00\% & 37.64\% & \underline{37.93\%} & \textbf{40.00\%} &  34.67\% & \textbf{40.00\%} & \underline{34.92\%} & \textbf{40.00\%} \\
& OBQA &  \textbf{14.67\%} & \underline{10.21\%} & 34.67\% & \textbf{41.33\%} & \underline{39.28\%} & \textbf{41.33\%} &  30.00\% & 33.88\% & \underline{34.31\%} & \textbf{35.33\%} \\
& MC-TEST &  \textbf{25.33\%} & \underline{21.67\%} & 57.33\% & 70.89\% & \underline{71.76\%} & \textbf{78.00\%} &  \underline{58.67\%}  & 54.51\% & 58.66\% & \textbf{59.33\%} \\
& AUX-LAW-90S &  \underline{11.67\%} & \textbf{13.00\%} & 26.00\% & 31.01\% & \underline{31.13\%} & \textbf{33.33\%} &  32.67\% & \underline{35.20\%} & 30.38\% & \textbf{36.67\%} \\
& ARC-EASY &  \underline{11.67\%} & \textbf{12.33\%} & \textbf{56.67\%}  & \underline{48.08\%} & 47.77\% & \underline{52.67\%} &  36.67\% & \underline{39.35\%} & 36.88\% & \textbf{42.00\%} \\
& SCI-ELEM & \underline{15.33\%} & \textbf{16.00\%} & 52.17\%  & \underline{54.59\%} & \textbf{60.87\%} & \textbf{60.87\%} &  \textbf{55.43\%} & 50.20\% & 50.84\% & \underline{52.17\%} \\
\midrule
{\fontsize{18}{16}\selectfont \multirow{3}{*}{\textbf{WMT14}}}
& EN-CD &  \underline{35.33\%} & \textbf{39.11\%} & 60.67\% & \textbf{63.33\%} & \underline{63.04\%} & \textbf{63.33\%} &  \underline{59.33\%} & 57.69\% & 56.72\% & \textbf{62.67\%} \\
& EN-RU &  \underline{41.67\%} & \textbf{40.00\%} & 57.33\% & 57.37\% & \underline{57.95\%} & \textbf{65.67\%} &  56.67\% & 55.41\% & \underline{60.55\%} & \textbf{61.33\%} \\
& EN-DE &  \textbf{35.00\%} & \underline{34.47\%} & 61.00\% & \textbf{68.33\%} & \underline{68.22\%} & \textbf{68.33\%} &  57.00\% & \textbf{65.00\%} & \underline{62.35\%} & \textbf{65.00\%} \\
\midrule
{\fontsize{18}{16}\selectfont \multirow{1}{*}{\textbf{GPQA}}}
& GPQA &  \underline{3.58\%} & \textbf{4.39\%} & 9.67\% & \underline{13.10\%} & 12.15\% & \textbf{15.33\%} &  11.00\% & \underline{17.20\%} & 14.53\% & \textbf{18.67\%} \\
\midrule
{\fontsize{18}{16}\selectfont \multirow{1}{*}{\textbf{MATH}}}
& MATH &  \underline{3.35\%} & \textbf{4.11\%} & 7.33\% & 7.82\% & \underline{8.80\%} & \textbf{12.67\%} &  6.33\% & \underline{10.02\%} & \textbf{11.67\%} & \textbf{11.67\%} \\
\midrule
{\fontsize{18}{16}\selectfont \multirow{1}{*}{\textbf{GSM8K}}}
& GSM8K &  \textbf{18.67\%} & \underline{15.11\%}  & 23.33\% & \underline{29.47\%} & 26.08\% & \textbf{31.00\%} &  21.58\%  & 21.40\% & \underline{26.92\%} & \textbf{28.47\%} \\
\bottomrule
\end{tabular}
} 
\end{sc}
\end{LARGE}
\end{center}
\vskip -0.1in
\end{table*}

\textbf{Experiments on Generalization.}
We evaluate \textit{RadarGate} on six well-known benchmarks to assess its generalization ability under three frameworks, with the settings involving different source test and training datasets. Table \ref{fig:gen} demonstrates the generalization performance of \textit{RadarGate} (top-k = $2$). Observations include:

\textbf{Obs.\ding{183} 
\textit{RadarGate} demonstrates optimal generalization performance when the training and test sets are from different sources.} Our \textit{RadarGate} achieves 30\%-50\% higher accuracy than rule-based methods and 5\%–10\% improvements over learnable baselines. This shows that the proposed \textit{RadarGate} adapts to different LoRA architectures, surpassing baselines in generalization while preserving module independence.

\textbf{Experiments on Scaling.}
We analyze LoRA modules and parameter scalability's impact on performance and can make the following observations.

\begin{figure}[htbp]
\centering
\setlength{\abovecaptionskip}{3pt}
\setlength{\belowcaptionskip}{-3pt}  
\begin{minipage}{\textwidth}
    \centering
    
    \begin{subfigure}[t]{0.32\textwidth}
        \includegraphics[width=\textwidth]{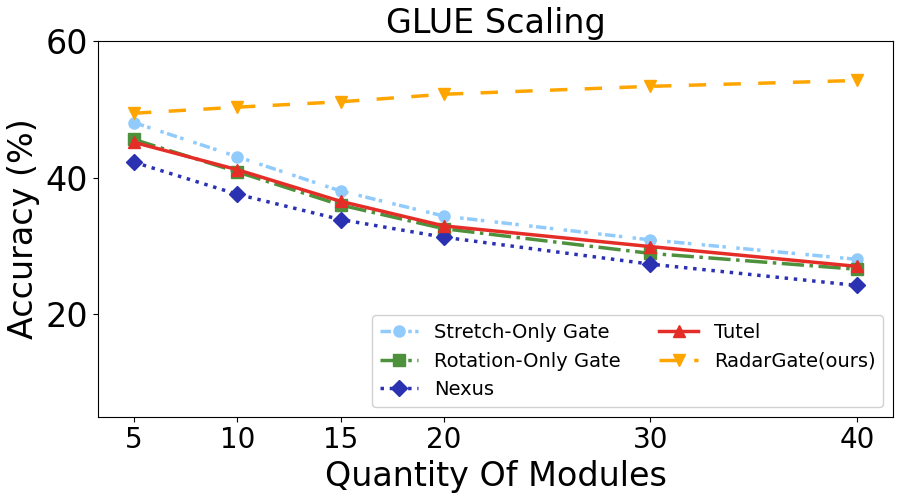}
        \caption{}
        \label{fig:subfig1}
    \end{subfigure}
    \hfill
    \begin{subfigure}[t]{0.32\textwidth}
        \includegraphics[width=\textwidth]{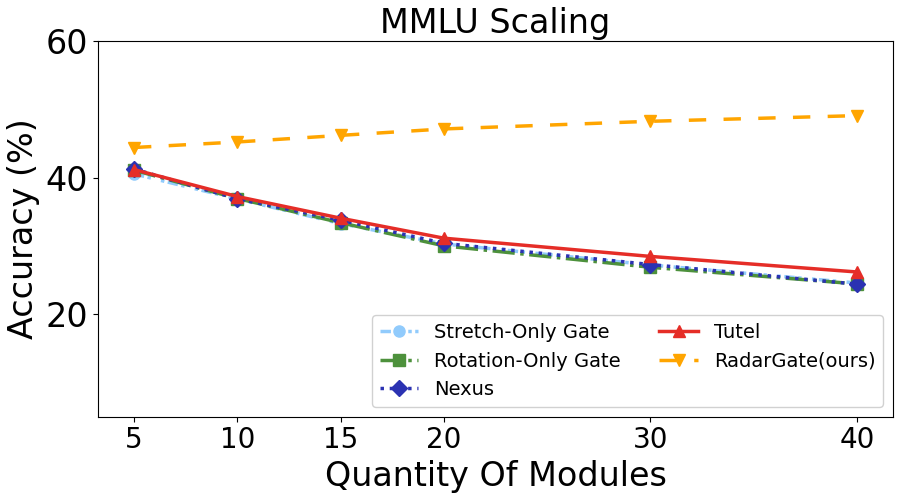}
        \caption{}
        \label{fig:subfig2}
    \end{subfigure}
    \hfill
    \begin{subfigure}[t]{0.32\textwidth}
        \includegraphics[width=\textwidth]{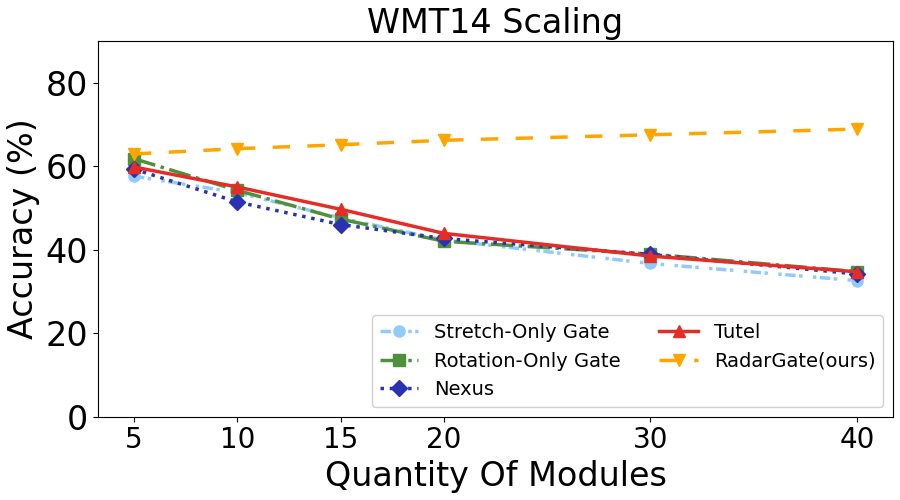}
        \caption{}
        \label{fig:subfig3}
    \end{subfigure}
    \hfill
    \begin{subfigure}[t]{0.32\textwidth}
        \includegraphics[width=\textwidth]{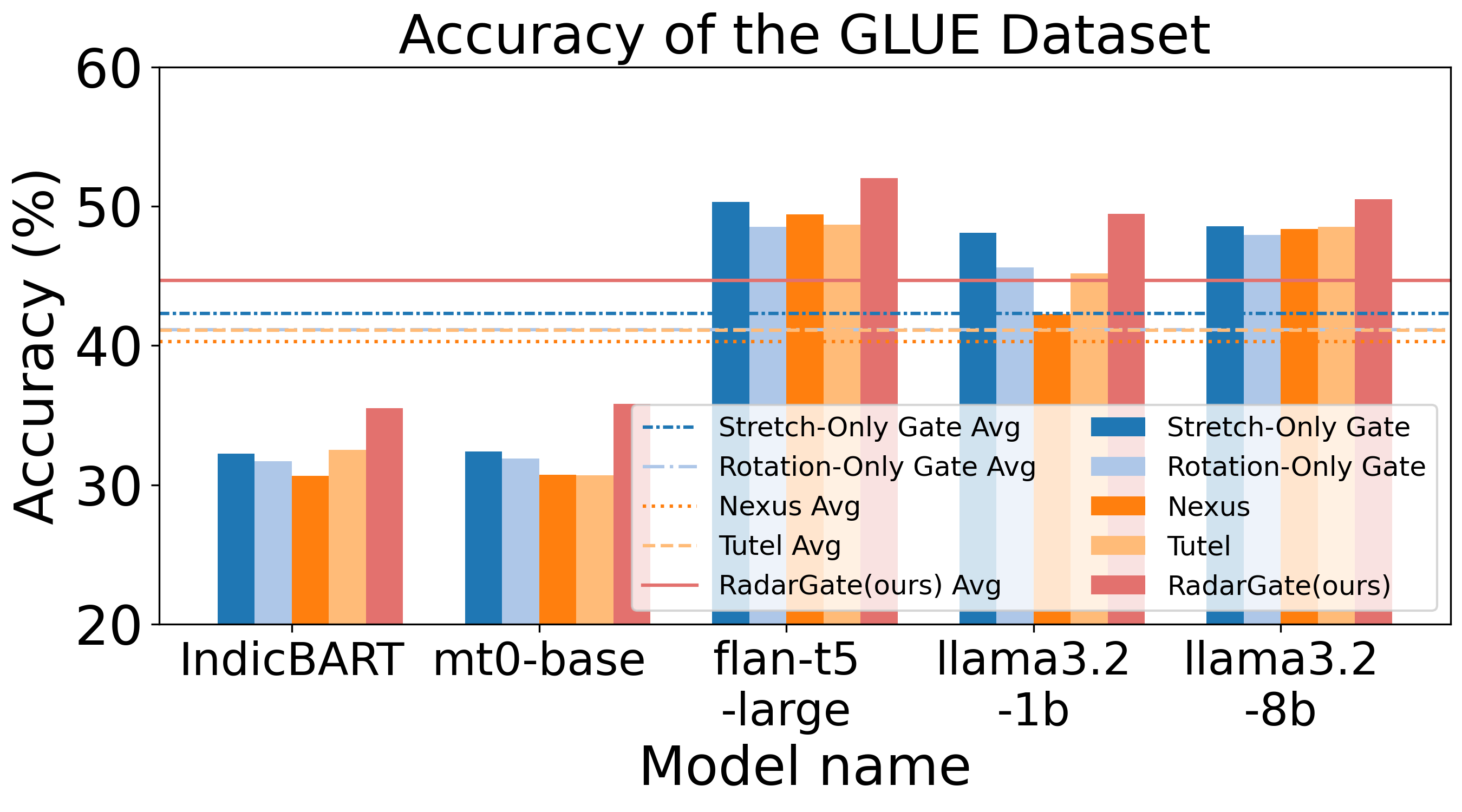}
        \caption{}
        \label{fig:subfig4}
    \end{subfigure}
    \hfill
    \begin{subfigure}[t]{0.32\textwidth}
        \includegraphics[width=\textwidth]{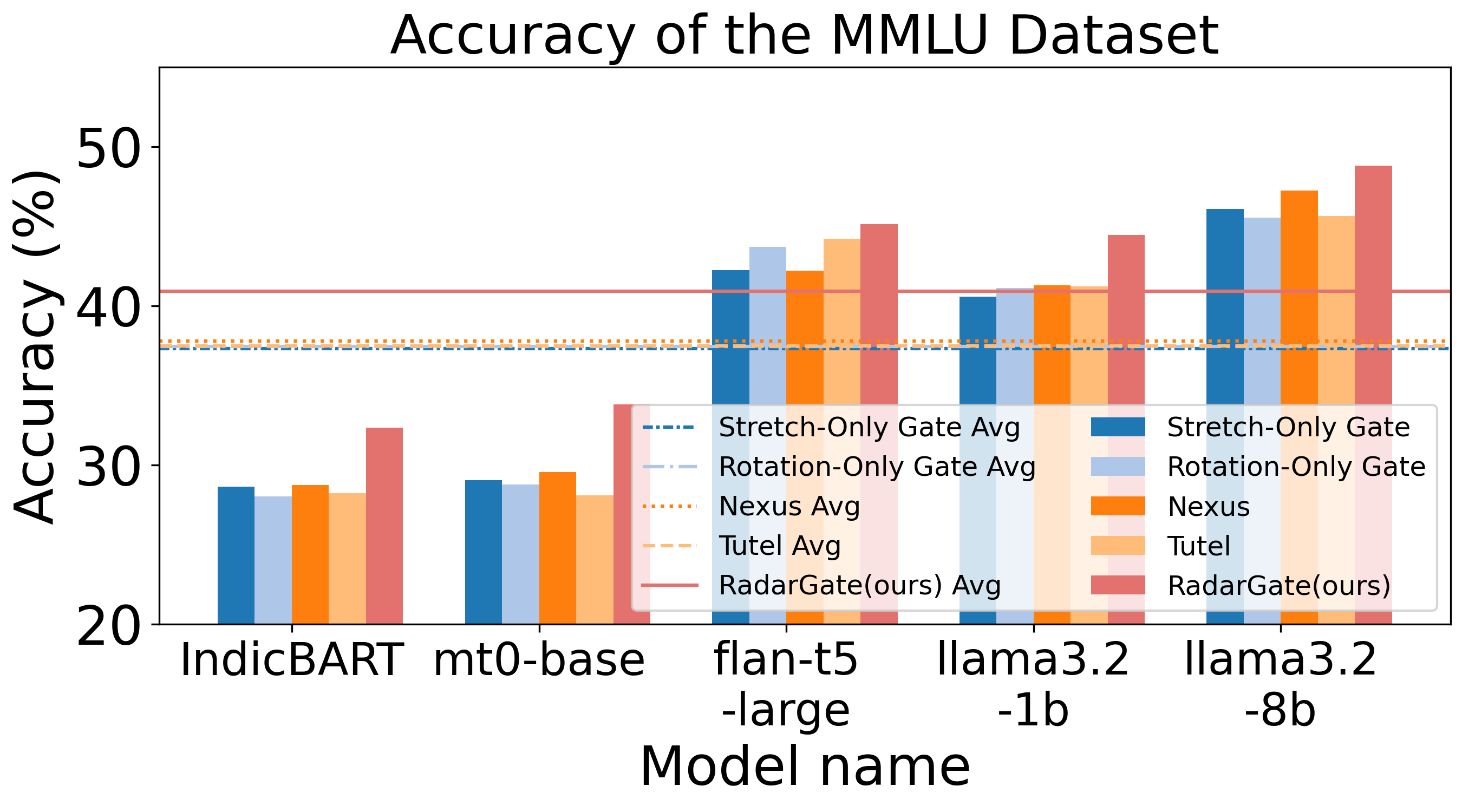}
        \caption{}
        \label{fig:subfig5}
    \end{subfigure}
    \hfill
    \begin{subfigure}[t]{0.32\textwidth}
        \includegraphics[width=\textwidth]{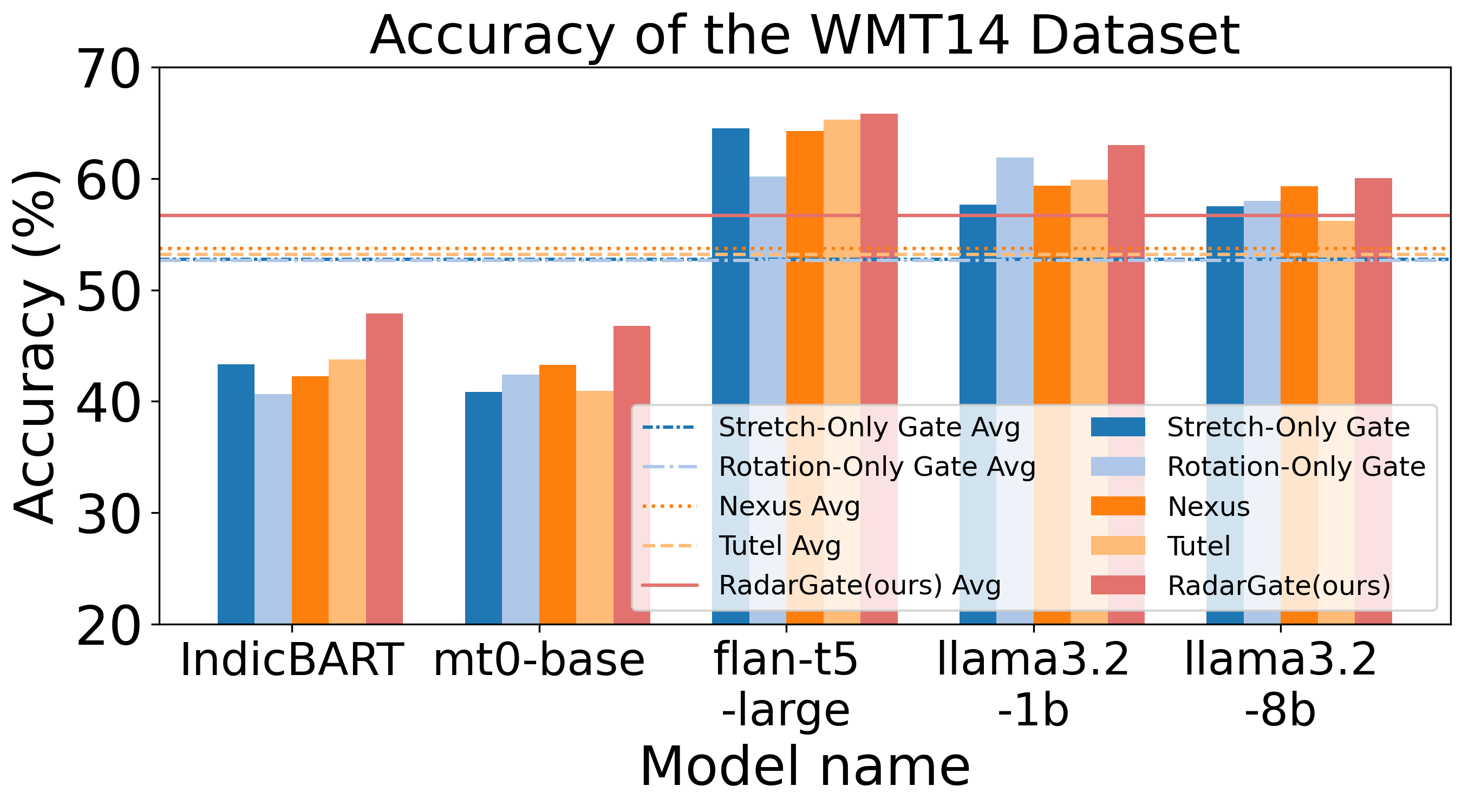}
        \caption{}
        \label{fig:subfig6}
    \end{subfigure}
\end{minipage}

\caption{\textbf{Scaling performance comparison on three benchmarks.} Figures (a), (b), and (c) compare the scaling performance of \textit{RadarGate} with four baselines as the number of modules increases. Figure (d), (e), and (f) show performance across different model sizes.}
\label{fig:main}
\end{figure}

\textbf{Obs.\ding{184} 
When the number of LoRA modules in the composable LoRA architecture scales up, our \textit{RadarGate} demonstrates superior performance.} Figure \ref{fig:main}(\subref{fig:subfig1})(\subref{fig:subfig2})(\subref{fig:subfig3}) evaluate NLP performance with 5–40 LoRA modules under MoLE: baselines exhibit inverted U-shaped trends, whereas our method achieves near-monotonic improvement (with an 8\% maximum gain), demonstrating sustained superiority. 

\textbf{Obs.\ding{185} 
When the parameter count of the base model in the composable LoRA architecture scales up (with the LoRA module parameters increasing accordingly), our \textit{RadarGate} demonstrates superior performance.} Figure \ref{fig:main}(\subref{fig:subfig4})(\subref{fig:subfig5})(\subref{fig:subfig6}) show consistent 5\%–10\% advantages over baselines across the 110M, 580M, 770M, 1B, and 8B parameters, confirming method scalability.

\textbf{Experiments on Ablation.}
To verify the necessity of the \textit{StretchGate} and \textit{RotationGate} components in RadaGate, we conducted ablation experiments on 12 tasks, leading to the following observations:

\textbf{Obs.\ding{186} 
\textit{RadarGate} achieves the best performance when all components are included, and \textit{StretchGate} and \textit{RotationGate} mutually reinforce each other.} As shown in Figures \ref{fig:ablation_studies}(\subref{fig:ablation_a})(\subref{fig:ablation_b}), the complete \textit{RadarGate} covers the largest area and achieves the maximum value in over 90\% of tasks, outperforming the second-best method by even 20\% in some cases. The standalone \textit{StretchGate} and \textit{RotationGate} exhibit significantly lower performance than the full \textit{RadarGate}, indicating that both components are crucial for optimizing the gating performance.

\begin{figure}[htbp]

\includegraphics[width=\columnwidth]{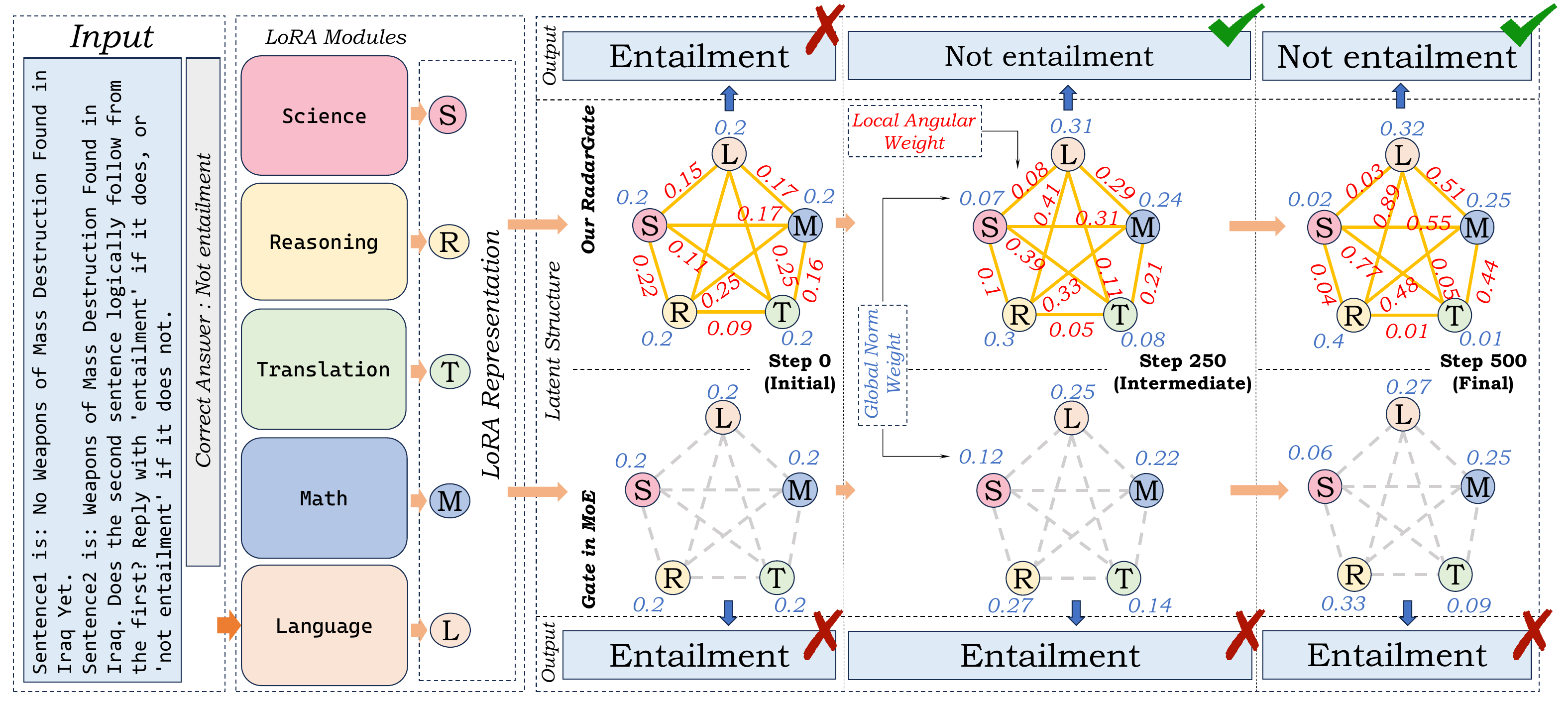}
\caption{\textbf{A GLUE case study visualizes \textit{RadarGate}.} The proposed \textit{RadarGate} can correctly integrate representations along global norm weight and local angular weight by mid-training, yielding the correct answers,while the Gate in the previous MoLE method fails to generate the correct answer.}
\label{fig:case_study}
\vspace{-6mm}
\end{figure}

\subsection{Case Study}
\vspace{-5pt}

Figure \ref{fig:case_study} demonstrates \textit{RadarGate}'s ability to capture latent architectures among LoRA representations in a reasoning task. Initially, both \textit{RadarGate} and MoE Gate fail due to averaged norm weights (global magnitude) and negligible angular weights (local angle dependencies). During training, MoE erroneously amplifies irrelevant Translation (T) and Science (S) LoRA modules, assigning their weights to 0.14 and 0.12, respectively. After 500 steps, MoE still fails. 
In contrast, we observe that \textit{RadarGate} suppresses interference by driving the angular weights between unrelated modules (module S and R, as well as module R and T) toward zero and those between related modules (module L and R, as well as module S and T) toward one, 
achieving correct answers after 250 steps. This suggests that the rotations to each pair of representations are contrastive, encouraging closer alignment of semantically similar representations while pushing distant ones further apart, which can help to converge the representations.

\vspace{-10pt}
\section{Discussion}
\vspace{-5pt}
Our experiments shows \textit{RadarGate}'s advantages in fitting, generalization and scalability. In this section, we present key observations and discuss future research directions. We investigate the convergence of \textit{RadarGate}, visualize it and analyze its performance in low-sample settings.

\begin{figure}[htpb]
\centering
\begin{minipage}{\textwidth}
    \centering
    
    \begin{subfigure}[t]{0.3\textwidth}
        \includegraphics[width=\textwidth]{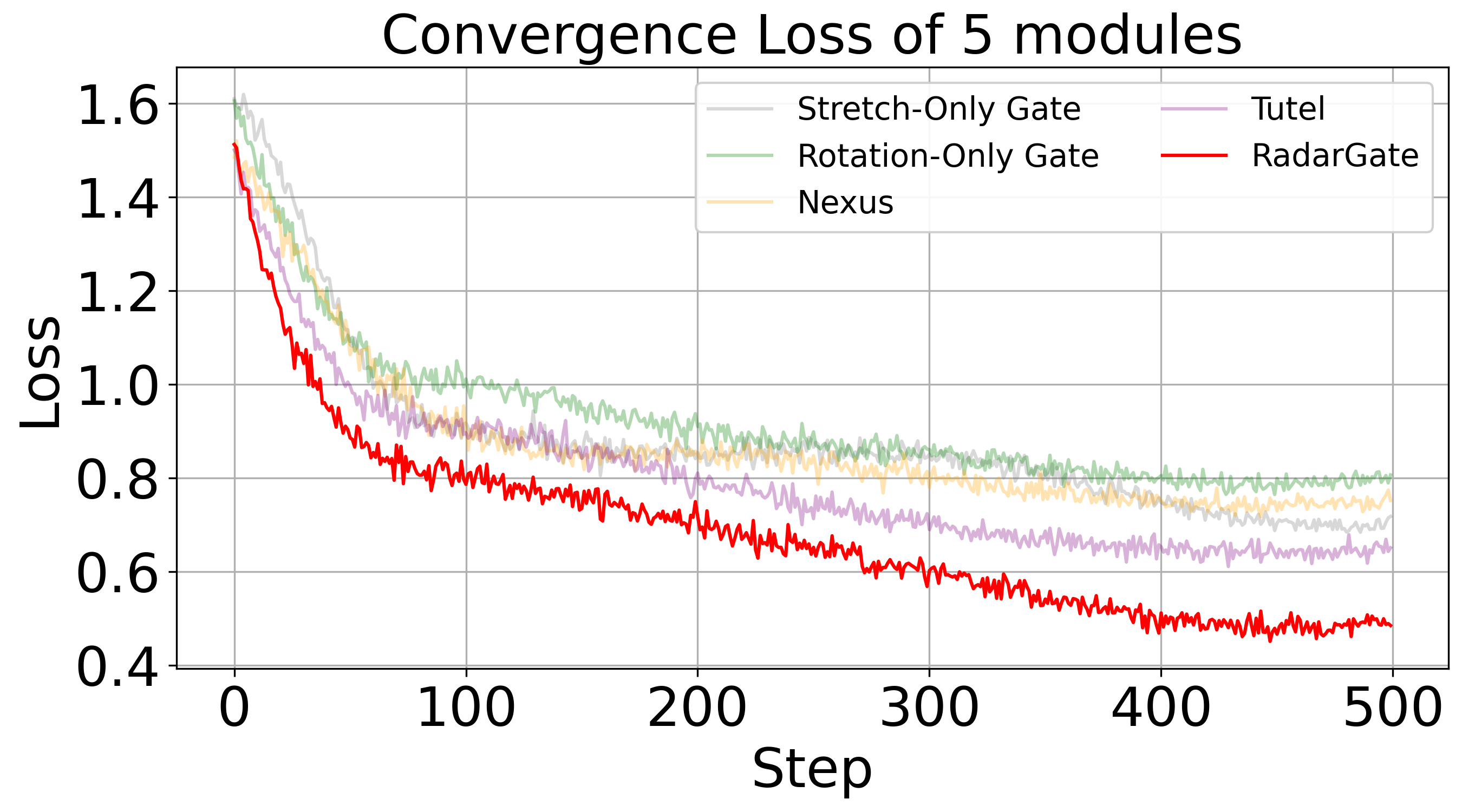}
        \caption{}
        \label{fig:subfig11}
    \end{subfigure}
    \hfill
    \begin{subfigure}[t]{0.37\textwidth}
        \includegraphics[width=\textwidth]{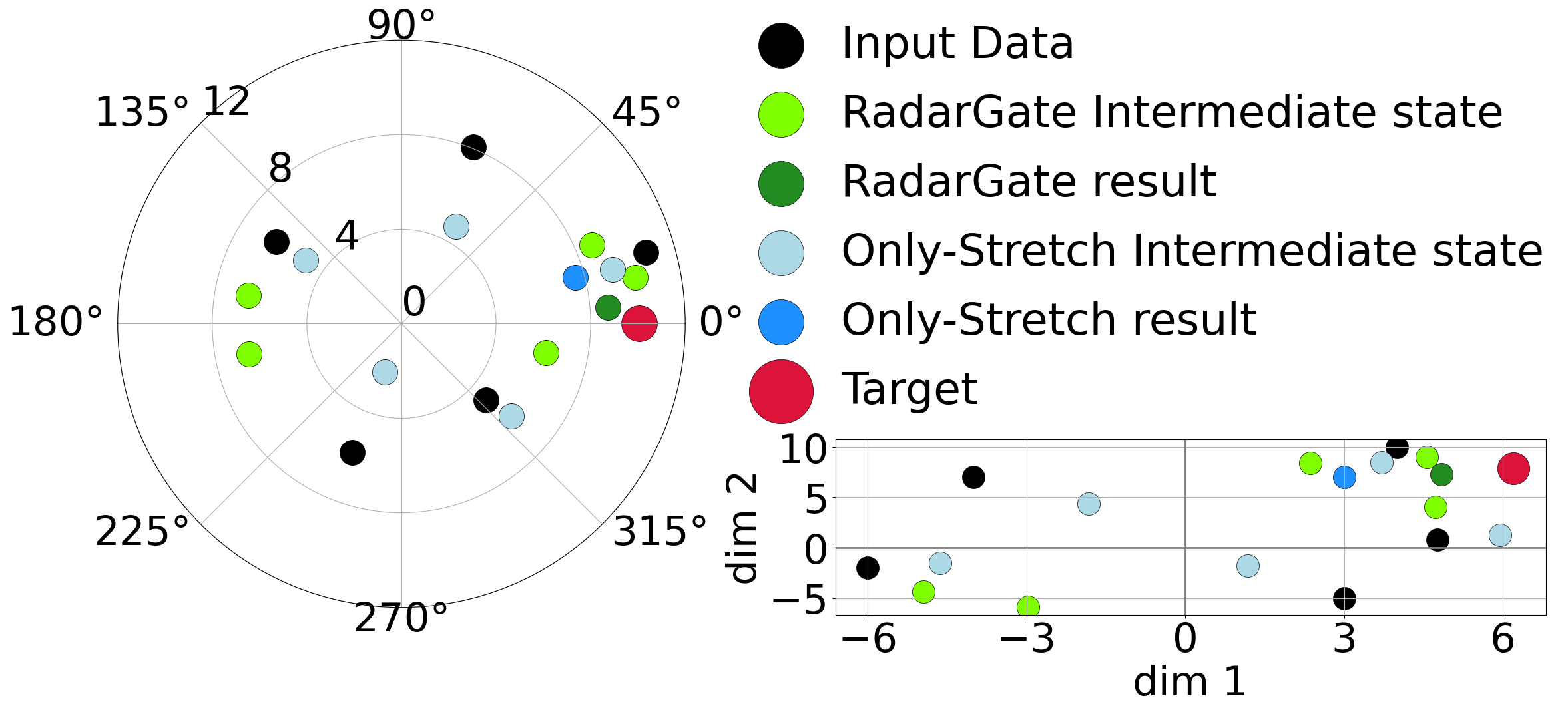}
        \caption{}
        \label{fig:subfig22}
    \end{subfigure}
    \hfill
    \begin{subfigure}[t]{0.3\textwidth}
        \includegraphics[width=\textwidth]{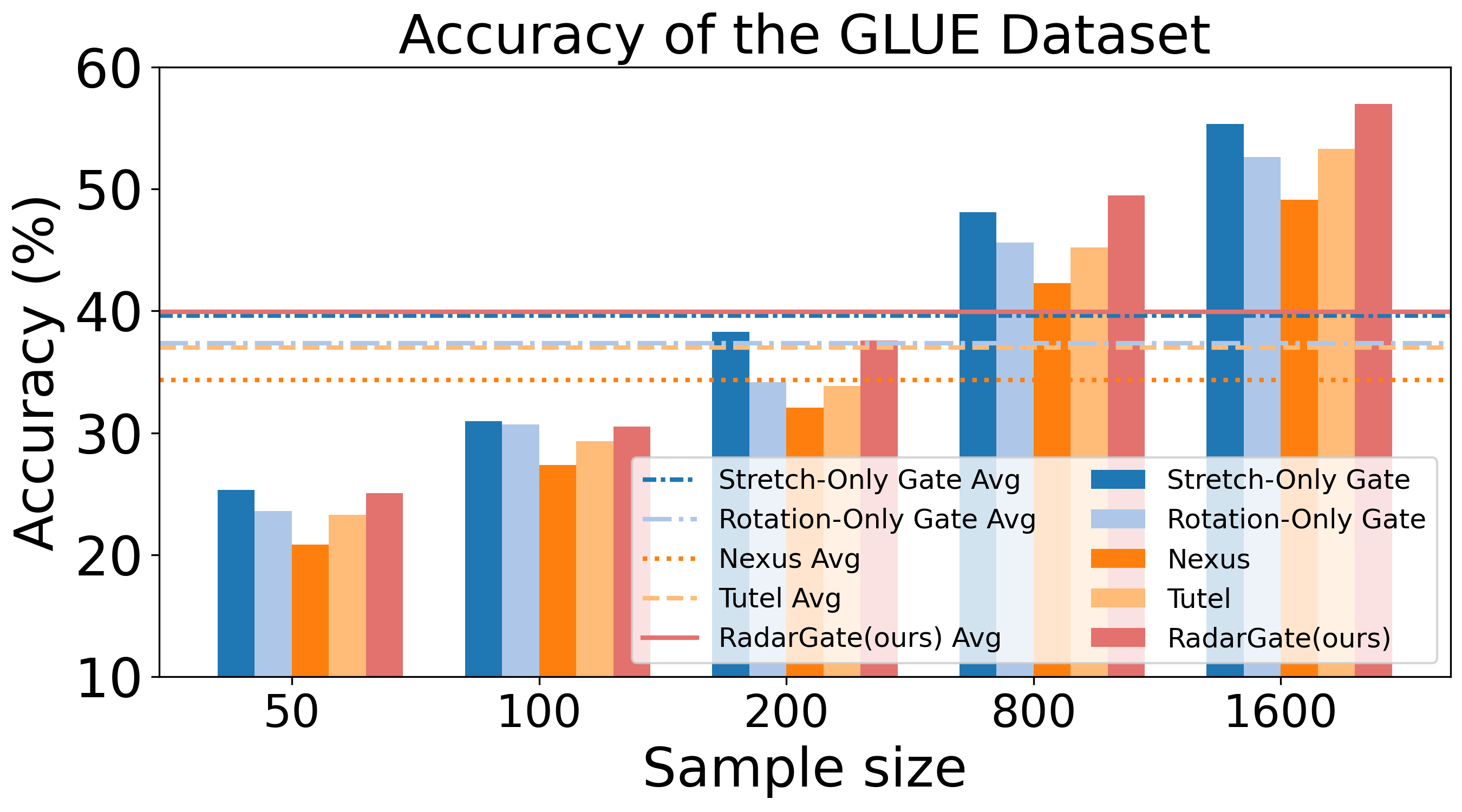}
        \caption{}
        \label{fig:subfig33}
    \end{subfigure}
\end{minipage}

\caption{\textbf{Discussion.} Figure (a) shows the convergence speed of different gating methods in the same MoLE architecture with five modules. Figure (b) presents the visualization of our proposed \textit{RadarGate}. Figure (c) compares method performance in MoLE under different sample sizes.}

\label{fig:main2}
\vskip -0.2in
\end{figure}

\vspace{3pt}
\textbf{Convergence during the Training steps.}

Training loss analysis reveals baseline methods suffer from early erratic oscillations, persistent convergence issues (suboptimal minima), and late-stage fluctuations indicating robustness limitations. As shown in Figure \ref{fig:main2}(\subref{fig:subfig11}), \textit{RadarGate} achieves faster convergence with lower loss and sustained stability, demonstrating superior convergence dynamics. Complete convergence experiment details are in Appendix \ref{app:exp_con}.

\textbf{Visualization of the Rotation Process.}
Figure \ref{fig:main2}(\subref{fig:subfig22}) shows the positional changes of LoRA representations before and after rotation and stretching, with PCA reducing the input from 2048 to 2 dimensions, visualized in both Cartesian and polar coordinates. In the Stretch-Only method, the magnitude of input vectors changes with little angular variation, 

leading to underfitting. In contrast, \textit{RadarGate} 

pulls similar vectors closer and rotates less correlated ones toward the target, improving feature learning and aligning the data more effectively with the target.

\textbf{Performance under Varying Sample Sizes.}

Figure \ref{fig:main2}(\subref{fig:subfig33}) shows our method achieves 5\% GLUE benchmark superiority over baselines at 50 samples, maintaining $\ge$5\% advantage with increasing data. This validates exceptional low-sample performance and resource-constrained applicability. Full sample size experiments are detailed in Appendix \ref{app:exp_sam}.

\vspace{-10pt}
\section{Conclusion}
\vspace{-5pt}
This paper introduces \textit{RadarGate}, a novel gating method that enhances the performance of composable LoRA architecture through two degrees of freedom: rotation and stretching. Our proposed \textit{RadarGate} achieves better fitting capability, generalization, and scalability for LoRA-MoE. Our \textit{RadarGate} consists of two key components: a \textit{RotationGate}  and a \textit{StretchGate}. The \textit{RotationGate} dynamically injects relative angular rotation information through a learnable parameter module. The output of the \textit{RotationGate} is fed into the \textit{StretchGate}, which stretches the magnitude of each LoRA representation by assigning weights to them. 
Extensive experiments on 6 public benchmarks across 21 tasks show the effectiveness of the proposed method. In the future, we aim to extend the application of \textit{RadarGate} to multimodal scenarios for image and video data.

\bibliography{example_paper}
\bibliographystyle{neurips_2025}

\appendix

\section{Motivation}
\label{app:lim}

We provide a more detailed mathematical explanation for the two observations presented in Section \ref{subsec:obs}. We adopt the notation from the main text. Specifically, $\mathbf{v}_i(\mathbf{x}) = \mathbf{x} A_i B_i$ represents the output of the $i$-th LoRA module for a given input $\mathbf{x}$. The gating weights, denoted by $g_i(\mathbf{x}; \mathbf{\theta}, \tau) = \tilde{S}(\epsilon_i)$, are produced by the gating mechanism described in Section \ref{sec:composable_loras_architecture_long_eq}. The LoRA component of the final output is $\Delta\mathbf{y}(\mathbf{x}) = \sum_{i=1}^n g_i(\mathbf{x}; \mathbf{\theta}, \tau) \mathbf{v}_i(\mathbf{x})$.

\subsection{Underfitting}

Existing gating mechanisms struggle to capture complex patterns of ideal $g_i^*(\mathbf{x})$ distribution within the convex hull $\mathcal{H}(\mathbf{x})$, resulting in an underfitting ensemble of multiple LoRAs.

    The ideal weights $g_i^*(\mathbf{x})$ are defined such that $\Delta\mathbf{y}_{\text{target}}(\mathbf{x}) \approx \sum_{i=1}^{n} g_i^*(\mathbf{x}) \mathbf{v}_i(\mathbf{x})$, as per Equation \eqref{eq:ideal_combination}. The core of the underfitting observation is that the function class for $g_i(\mathbf{x}; \mathbf{\theta}, \tau)$ might not be rich enough to approximate $g_i^*(\mathbf{x})$ for all training samples $\mathbf{x}$ using a single, shared $\theta$.

    The critical step in generating $g_i$ is the linear projection $\mathbf{z}(\mathbf{x})\theta$ to obtain the logits $\epsilon(\mathbf{x}; \theta)$.  the final mapping from this representation to the logits is linear, governed by the fixed matrix $\theta$. The subsequent softmax and top-$k$ operations are fixed non-linear transformations.

    If the true mapping from $\mathbf{z}(\mathbf{x})$ to the ideal logits $\epsilon^*(\mathbf{x})$ (i.e., logits that would perfectly generate $g_i^*(\mathbf{x})$ after softmax and top-$k$) is highly non-linear or varies in a complex manner that cannot be captured by a single linear transformation $\theta$, then the model will struggle.

    The model is trained by minimizing a loss function, typically of the form:
    $$\mathcal{L}_{\text{total}} = \sum_{(\mathbf{x}, \mathbf{y}_{\text{target}}) \in \text{TrainingSet}} \| \mathbf{y}_{\text{target}} - (\mathbf{x}W + \sum_{i=1}^n g_i(\mathbf{x}; \mathbf{\theta}, \tau) \mathbf{v}_i(\mathbf{x})) \|^2$$
    If the gating mechanism $g_i(\mathbf{x}; \mathbf{\theta}, \tau)$ cannot adequately approximate the ideal weights $g_i^*(\mathbf{x})$ across the diverse inputs $\mathbf{x}$ in the training set, the term $\sum_{i=1}^n g_i(\mathbf{x}; \mathbf{\theta}, \tau) \mathbf{v}_i(\mathbf{x})$ will fail to reconstruct $\Delta\mathbf{y}_{\text{target}}(\mathbf{x})$ effectively. This leads to a high training loss, which is characteristic of underfitting. The "magnitude scaling" mentioned in the observation refers to the fact that $g_i$ controls the contribution magnitude of $\mathbf{v}_i(\mathbf{x})$; the underfitting arises because the mechanism for determining these $g_i$ values is not sufficiently expressive.

\subsection{Poor Generalization}

Existing gating methods heavily rely on a weighted sum of LoRA representation, thus degrading generalization as the LoRA output is limited in the convex hull $\mathcal{H}(\mathbf{x})$.

For a given input $\mathbf{x}$, the LoRA modules produce a set of basis vectors $\mathbf{v}_i(\mathbf{x}) = \mathbf{x} A_i B_i$. The gating mechanism computes weights $g_i(\mathbf{x}; \mathbf{\theta}, \tau) = \tilde{S}(\epsilon_i)$.
    According to the definition of $\tilde{S}(\epsilon_i)$ in Equation (5):
\begin{equation}
    g_i(\mathbf{x}; \mathbf{\theta}, \tau) \ge 0,\sum_{i=1}^n g_i(\mathbf{x}; \mathbf{\theta}, \tau)  = 1. 
\end{equation}

    Given the properties of $g_i$, $\Delta\mathbf{y}_{\text{model}}(\mathbf{x})$ is a convex combination of the vectors $\{\mathbf{v}_i(\mathbf{x})\}_{i \in \epsilon_{\text{topk}}}$. This means $\Delta\mathbf{y}_{\text{model}}(\mathbf{x})$ must lie within the convex hull of the selected top-$k$ LoRA module outputs.
    Let $\mathcal{V}(\mathbf{x}) = \{\mathbf{v}_1(\mathbf{x}), \dots, \mathbf{v}_n(\mathbf{x})\}$. The generated LoRA output $\Delta\mathbf{y}_{\text{model}}(\mathbf{x})$ is always an element of $\text{conv}(\mathcal{V}_{\text{topk}}(\mathbf{x}))$, where $\mathcal{V}_{\text{topk}}(\mathbf{x})$ are the top-k selected vectors. This is necessarily a subset of $\mathcal{H}(\mathbf{x}) = \text{conv}(\mathcal{V}(\mathbf{x}))$.

Thus, $\mathbf{v}_i(\mathbf{x})$ are determined by the fixed LoRA matrices $A_i, B_i$ and the input $\mathbf{x}$. The gating coefficients $g_i$ only scale the magnitudes of these directional vectors $\mathbf{v}_i(\mathbf{x})$ and sum them up; they do not alter the directions themselves.

Consider a target output $\mathbf{y}_{\text{target}}(\mathbf{x})$. The corresponding target LoRA contribution is $\Delta\mathbf{y}_{\text{target}}(\mathbf{x}) = \mathbf{y}_{\text{target}}(\mathbf{x}) - \mathbf{x}W$.

If, for a given input $\mathbf{x}$ (especially one from a test set or unseen distribution), the true required modification $\Delta\mathbf{y}_{\text{target}}(\mathbf{x})$ lies outside the convex hull $\mathcal{H}(\mathbf{x})$, i.e.,
\begin{equation}
    \Delta\mathbf{y}_{\text{target}}(\mathbf{x}) \notin \mathcal{H}(\mathbf{x}) = \left\{ \sum_{i=1}^{n} \alpha_i \mathbf{v}_i(\mathbf{x}) \mid \alpha_i \ge 0, \sum_{i=1}^{n} \alpha_i = 1 \right\}
\end{equation}

then it is mathematically impossible for the model's LoRA output $\Delta\mathbf{y}_{\text{model}}(\mathbf{x})$ to exactly match $\Delta\mathbf{y}_{\text{target}}(\mathbf{x})$, because $\Delta\mathbf{y}_{\text{model}}(\mathbf{x})$ is constrained to be within $\mathcal{H}(\mathbf{x})$.Therefore,the best the model can do is to output the point in $\mathcal{H}(\mathbf{x})$ that is closest to $\Delta\mathbf{y}_{\text{target}}(\mathbf{x})$. This point is the projection of $\Delta\mathbf{y}_{\text{target}}(\mathbf{x})$ onto the set $\mathcal{H}(\mathbf{x})$, let's call it $\text{proj}_{\mathcal{H}(\mathbf{x})}(\Delta\mathbf{y}_{\text{target}}(\mathbf{x}))$.
The error resulting from this input $\mathbf{x}$, usually in the form of MSE, will be:
    $$\| \Delta\mathbf{y}_{\text{target}}(\mathbf{x}) - \text{proj}_{\mathcal{H}(\mathbf{x})}(\Delta\mathbf{y}_{\text{target}}(\mathbf{x})) \|^2$$

This inherent limitation, stemming from the fact that the LoRA outputs are restricted to the convex hull of component LoRA outputs (scaled by non-negative weights summing to unity), can lead to poor generalization performance when the model encounters inputs $\mathbf{x}$ for which the desired output $\Delta\mathbf{y}_{\text{target}}(\mathbf{x})$ requires stepping outside this convex hull.

\section{Workflow of \textit{RadarGate}}\label{app:wor}
\subsection{\textit{RadarGate} Inference}\label{subsec:inference}

Our proposed \textit{RadarGate} method combines the output of the original pretrained model layer with dynamically adjusted outputs from multiple low-rank adaptation (LoRA) submodules. The overall output $\mathbf{y}$ of a layer utilizing \textit{RadarGate} is computed as:
\begin{equation}
\mathbf{y} = \mathbf{x} W + \sum_{i=1}^n g_i \tilde{\mathbf{v}}_i,
\label{eq:radar_output}
\end{equation}
where $\mathbf{x} \in \mathbb{R}^{d_{in}}$ is the input vector, $W \in \mathbb{R}^{d_{in} \times d_{out}}$ is the frozen weight matrix from the pretrained model, $n$ is the number of LoRA submodules. Here, $\tilde{\mathbf{v}}_i \in \mathbb{R}^{d_{out}}$ is the transformed output of the $i$-th LoRA submodule produced by the \textit{RotationGate}, and $g_i \in \mathbb{R}$ is the scalar gating factor for the $i$-th submodule's transformed output, produced by the \textit{StretchGate}.

This approach reformulates the traditional gating mechanism often used with multiple LoRA modules. In typical setups, as illustrated in Section \ref{sec:composable_loras_architecture_long_eq}, the individual LoRA outputs $\mathbf{v}_i = \mathbf{x}A_iB_i$ are often concatenated, normalized, and projected to generate scalar gating weights $g_i$ (e.g., via softmax and top-$k$), and the final output is a sum over scaled original LoRA outputs: $\mathbf{y} = \mathbf{x}W + \sum_{i=1}^n g_i \mathbf{v}_i$. \textit{RadarGate} differs fundamentally by introducing an angular adjustment via the \textit{RotationGate} before applying the scaling factor from the \textit{StretchGate}, operating on $\tilde{\mathbf{v}}_i$ instead of $\mathbf{v}_i$.

To achieve this, \textit{RadarGate}'s gating function $G$, as defined in Equation \eqref{eq:9}, operates on the structural relationships between LoRA modules captured by $\text{Map}(\mathcal{L})$ and the input $\mathbf{x}$, parameterized by $\theta_r$ and $\theta_s$:
\begin{equation}
\tilde{\mathbf{v}}_i = G\left( \text{Map}(\mathcal{L}_i), \mathbf{x}; \theta_r \right).
\end{equation}
As shown, the function $G$ here specifically represents the \textit{RotationGate}, producing the rotated vector $\tilde{\mathbf{v}}_i$. The \textit{StretchGate} is a separate component that computes the scalar $g_i$. The operator $\text{Map}(\mathcal{L}_i)$ provides the $i$-th submodule $A_iB_i$ and its reference set $\mathcal{L} - \{A_iB_i\}$, consistent with Equation \eqref{eq:10}:
\begin{equation}
\text{Map}(\mathcal{L}_i) = (A_iB_i, \mathcal{L} - \{A_iB_i\}),\quad\mathcal{L}=\{A_iB_i | i=1,2,...,n\}.
\end{equation}

The \textit{RotationGate} computes $\tilde{\mathbf{v}}_i$ by applying a rotation matrix $\mathcal{R}_i$ to the original LoRA output $\mathbf{v}_i$. The gating function applies a transformation analogous to Rotary Position Embedding (RoPE) to each submodule output $\mathbf{v}_i = \mathbf{x}A_iB_i$, following the principle in Equation \eqref{eq:gating}:
\begin{equation}
\tilde{\mathbf{v}}_i = \mathbf{v}_i \times \mathcal{R}_i = (\mathbf{x}A_iB_i) \times \mathcal{R}_i \quad (\tilde{\mathbf{v}}_i \in \mathbb{R}^{d_{out}}).
\end{equation}
Here, $\times$ denotes matrix multiplication. The rotation matrix $\mathcal{R}_i \in \mathbb{R}^{d_{out} \times d_{out}}$ is a block-diagonal matrix. Due to the sparsity of the rotation matrix, we can optimize the computation into the following element-wise formula, avoiding explicit matrix multiplication:
\begin{equation}
\tilde{\mathbf{v}}_i =
\begin{pmatrix}
\mathbf{v}_i^{(0)} \\ \mathbf{v}_i^{(1)} \\ \mathbf{v}_i^{(2)} \\ \mathbf{v}_i^{(3)} \\ \vdots \\ \mathbf{v}_i^{({d_{out}-2})} \\ \mathbf{v}_i^{({d_{out}-1})}
\end{pmatrix}
\odot
\begin{pmatrix}
\cos \alpha_{r_i}^{(0)} \\ \cos \alpha_{r_i}^{(0)} \\ \cos \alpha_{r_i}^{(1)} \\ \cos \alpha_{r_i}^{(1)} \\ \vdots \\ \cos \alpha_{r_i}^{(d_{out}/2-1)} \\ \cos \alpha_{r_i}^{(d_{out}/2-1)}
\end{pmatrix}
+
\begin{pmatrix}
-\mathbf{v}_i^{(1)} \\ \mathbf{v}_i^{(0)} \\ -\mathbf{v}_i^{(3)} \\ \mathbf{v}_i^{(2)} \\ \vdots \\ -\mathbf{v}_i^{({d_{out}-1})} \\ \mathbf{v}_i^{({d_{out}-2})}
\end{pmatrix}
\odot
\begin{pmatrix}
\sin \alpha_{r_i}^{(0)} \\ \sin \alpha_{r_i}^{(0)} \\ \sin \alpha_{r_i}^{(1)} \\ \sin \alpha_{r_i}^{(1)} \\ \vdots \\ \sin \alpha_{r_i}^{(d_{out}/2-1)} \\ \sin \alpha_{r_i}^{(d_{out}/2-1)}
\end{pmatrix}
\end{equation}
Specifically, for each pair of dimensions $(2m, 2m+1)$ where $m=0,1,\dots,\frac{d_{out}}{2}-1$, the rotation is applied as:
\begin{equation}
\begin{aligned}
\tilde{\mathbf{v}}_i^{(2m)} &= \mathbf{v}_i^{(2m)} \cos \alpha_{r_i}^{(m)} - \mathbf{v}_i^{(2m+1)} \sin\alpha_{r_i}^{(m)},\\
\tilde{\mathbf{v}}_i^{(2m+1)} &= \mathbf{v}_i^{(2m)} \sin \alpha_{r_i}^{(m)} + \mathbf{v}_i^{(2m+1)} \cos \alpha_{r_i}^{(m)}.
\end{aligned}
\end{equation}
This computation corresponds to the application of the block-diagonal matrix $\mathcal{R}_i$:
\begin{equation}
\mathcal{R}_i =
\begin{pmatrix}
R_i^{(0)} & 0 & \cdots & 0 \\
0 & R_i^{(1)} & \cdots & 0 \\
\vdots & \vdots & \ddots & \vdots \\
0 & 0 & \cdots & R_i^{(\frac{d_{out}}{2}-1)} \\
\end{pmatrix},
\quad R_i^{(m)} =
\begin{pmatrix}
\cos \alpha_{r_i}^{(m)} & -\sin \alpha_{r_i}^{(m)} \\
\sin \alpha_{r_i}^{(m)} & \cos \alpha_{r_i}^{(m)} \\
\end{pmatrix},
\label{eq:rotation_matrix_inference}
\end{equation}
where $\alpha_{r_i}^{(m)}$ is the $m$-th component of the rotational control factor vector $\alpha_{r_i} \in \mathbb{R}^{\frac{d_{out}}{2}}$. The vector $\alpha_{r_i}$ governs the rotation angles for $\tilde{\mathbf{v}}_i$ and is computed dynamically based on the input $\mathbf{x}$ and the relative relationships defined by $Map(\mathcal{L}_i)^{(0)}$ and $Map(\mathcal{L}_i)^{(1)}$, using the learnable parameter $\theta_r \in \mathbb{R}^{d_{out} \times \frac{d_{out}}{2}}$:
\begin{equation}
\alpha_{r_i} = \left(\mathbf{x}\times Map(\mathcal{L}_i)^{(0)}\right) \odot \left(\mathbf{x}\times \sum_{A_jB_j \in S_i} A_jB_j\right) \times \theta_r,\quad S_i = \{A_jB_j | A_jB_j \in Map(\mathcal{L}_i)^{(1)}\}.
\end{equation}
Here, $Map(\mathcal{L}_i)^{(0)} = A_iB_i$ and $Map(\mathcal{L}_i)^{(1)} = \mathcal{L} - \{A_iB_i\}$. The Hadamard product ($\odot$) is applied element-wise to the resulting vectors. The orthogonal nature of $\mathcal{R}_i$ ensures that this operation rotates the vector $\mathbf{v}_i$ without changing its magnitude, effectively representing relative angular relationships.

The \textit{StretchGate} is responsible for computing the scalar gating factor $g_i$ for each rotated submodule output $\tilde{\mathbf{v}}_i$. This computation leverages parameters related to $\theta_s$ and dynamically scales the contribution of each rotated LoRA output (details of the \textit{StretchGate} computation for $g_i$ are provided in Equation \eqref{eq:long_output}). The final layer output $\mathbf{y}$ is then the sum of the original pretrained model output $\mathbf{x}W$ and the scaled, rotated LoRA outputs $g_i \tilde{\mathbf{v}}_i$, as expressed in Equation \eqref{eq:radar_output}.

\subsection{\textit{RadarGate} Parameter Updating}
This section details the computation of gradients for the learnable parameters $\theta_s$ and $\theta_r$ to minimize the loss function $\mathcal{L}$. The gradients are derived using the chain rule based on the forward pass defined by Equation \eqref{eq:radar_output}: $\mathbf{y} = \mathbf{x} W + \sum_{i=1}^n g_i \tilde{\mathbf{v}}_i$.

\textbf{Derivative of Loss w.r.t. $\theta_s$}

The parameter matrix $\theta_s \in \mathbb{R}^{nd_{out} \times n}$ parameterizes the \textit{StretchGate}, which computes the scalar gating weights $g_i$. Following the structure described in Section \ref{sec:composable_loras_architecture_long_eq} for generating scalar gates, we assume $g_i$ are derived from the concatenated and normalized original LoRA outputs projected by $\theta_s$. Let $\mathbf{v}_i = \mathbf{x}A_iB_i$, $\mathbf{x}_{\text{cat}} = \mathbf{v}_1 \oplus \dots \oplus \mathbf{v}_n$, and $\mathbf{z} = \mathrm{Normalization}(\mathbf{x}_{\text{cat}}) \in \mathbb{R}^{1 \times nd_{out}}$. The gating logits are $\epsilon = \mathbf{z} \theta_s \in \mathbb{R}^{1 \times n}$, and $g_i$ is derived from $\epsilon_i$ (e.g., via temperature-scaled softmax and top-$k$).

The gradient of the loss $\mathcal{L}$ with respect to $\theta_s$ is given by the chain rule. A convenient form for the matrix gradient $\frac{\partial \mathcal{L}}{\partial \theta_s}$ is:
\begin{equation}
\frac{\partial \mathcal{L}}{\partial \theta_s} = \mathbf{z}^\intercal \left(\frac{\partial \mathcal{L}}{\partial \epsilon}\right),
\end{equation}
where $\frac{\partial \mathcal{L}}{\partial \epsilon}$ is the row vector $\begin{bmatrix} \frac{\partial \mathcal{L}}{\partial \epsilon_1} & \dots & \frac{\partial \mathcal{L}}{\partial \epsilon_n} \end{bmatrix} \in \mathbb{R}^{1 \times n}$.

We compute the elements of $\frac{\partial \mathcal{L}}{\partial \epsilon}$ using the chain rule through the scalar gates $g_i$:
\begin{equation}
\frac{\partial \mathcal{L}}{\partial \epsilon_j} = \sum_{i=1}^n \frac{\partial \mathcal{L}}{\partial g_i} \frac{\partial g_i}{\partial \epsilon_j}.
\end{equation}
From the \textit{RadarGate} output definition $\mathbf{y} = \mathbf{x}W + \sum_{k=1}^n g_k \tilde{\mathbf{v}}_k$, the derivative of the scalar loss $\mathcal{L}$ with respect to the scalar gate $g_i$ is:
\begin{equation}
\frac{\partial \mathcal{L}}{\partial g_i} = \frac{\partial \mathcal{L}}{\partial \mathbf{y}} \left(\frac{\partial \mathbf{y}}{\partial g_i}\right)^\intercal = \frac{\partial \mathcal{L}}{\partial \mathbf{y}} (\tilde{\mathbf{v}}_i)^\intercal.
\end{equation}
The term $\frac{\partial g_i}{\partial \epsilon_j}$ is the derivative of the function mapping logits $\epsilon$ to sparse gating weights $g_i$. Using the common approximation for sparse gating gradients (considering $\frac{\partial g_i}{\partial \epsilon_j} \approx 0$ for $i \neq j$ and using the standard softmax derivative for $i=j$):
\begin{equation}
\frac{\partial g_i}{\partial \epsilon_j} \approx \delta_{ij} \frac{1}{\tau} S(\epsilon_i)(1 - S(\epsilon_i)).
\end{equation}
Substituting these into the expression for $\frac{\partial \mathcal{L}}{\partial \epsilon_j}$:
\begin{equation}
\frac{\partial \mathcal{L}}{\partial \epsilon_j} \approx \sum_{i=1}^n \left(\frac{\partial \mathcal{L}}{\partial \mathbf{y}} (\tilde{\mathbf{v}}_i)^\intercal\right) \left(\delta_{ij} \frac{1}{\tau} S(\epsilon_i)(1 - S(\epsilon_i))\right) = \left(\frac{\partial \mathcal{L}}{\partial \mathbf{y}} (\tilde{\mathbf{v}}_j)^\intercal\right) \frac{1}{\tau} S(\epsilon_j)(1 - S(\epsilon_j)).
\end{equation}
Combining with $\frac{\partial \mathcal{L}}{\partial \theta_s} = \mathbf{z}^\intercal \left(\frac{\partial \mathcal{L}}{\partial \epsilon}\right)$, the full gradient matrix is:
\begin{equation}
\frac{\partial \mathcal{L}}{\partial \theta_s} \approx \mathbf{z}^\intercal \begin{bmatrix} \left(\frac{\partial \mathcal{L}}{\partial \mathbf{y}} (\tilde{\mathbf{v}}_1)^\intercal\right) \frac{1}{\tau} S(\epsilon_1)(1 - S(\epsilon_1)), & \dots, & \left(\frac{\partial \mathcal{L}}{\partial \mathbf{y}} (\tilde{\mathbf{v}}_n)^\intercal\right) \frac{1}{\tau} S(\epsilon_n)(1 - S(\epsilon_n)) \end{bmatrix}.
\end{equation}

\textbf{Derivative of Loss w.r.t. $\theta_r$}
The parameter matrix $\theta_r \in \mathbb{R}^{d_{out} \times \frac{d_{out}}{2}}$ parameterizes the \textit{RotationGate}, influencing the rotation angles $\alpha_{r_i}$ which determine $\tilde{\mathbf{v}}_i$. The gradient of $\mathcal{L}$ with respect to $\theta_r$ is given by the chain rule:
\begin{equation}
\frac{\partial \mathcal{L}}{\partial \theta_r} = \sum_{i=1}^n \frac{\partial \mathcal{L}}{\partial \tilde{\mathbf{v}}_i} \frac{\partial \tilde{\mathbf{v}}_i}{\partial \theta_r}.
\end{equation}
From the output definition, the derivative of the loss with respect to the rotated vector $\tilde{\mathbf{v}}_i$ is:
\begin{equation}
\frac{\partial \mathcal{L}}{\partial \tilde{\mathbf{v}}_i} = \frac{\partial \mathcal{L}}{\partial \mathbf{y}} \left(\frac{\partial \mathbf{y}}{\partial \tilde{\mathbf{v}}_i}\right)^\intercal = \frac{\partial \mathcal{L}}{\partial \mathbf{y}} (g_i \mathbf{I})^\intercal = g_i \frac{\partial \mathcal{L}}{\partial \mathbf{y}}.
\end{equation}
Now we need $\frac{\partial \tilde{\mathbf{v}}_i}{\partial \theta_r}$. The rotated vector $\tilde{\mathbf{v}}_i$ depends on $\alpha_{r_i}$, which depends on $\theta_r$. We decompose this derivative:
\begin{equation}
\frac{\partial \tilde{\mathbf{v}}_i}{\partial \theta_r} = \frac{\partial \tilde{\mathbf{v}}_i}{\partial \alpha_{r_i}} \frac{\partial \alpha_{r_i}}{\partial \theta_r}.
\end{equation}
$\frac{\partial \tilde{\mathbf{v}}_i}{\partial \alpha_{r_i}}$ is a Jacobian matrix $\in \mathbb{R}^{d_{out} \times \frac{d_{out}}{2}}$. Its elements are $\frac{\partial \tilde{\mathbf{v}}_i^{(k)}}{\partial \alpha_{r_i}^{(m)}}$. From the element-wise rotation formulas:
\begin{align}
\frac{\partial \tilde{\mathbf{v}}_i^{(2m)}}{\partial \alpha_{r_i}^{(m)}} &= -\mathbf{v}_i^{(2m)} \sin \alpha_{r_i}^{(m)} - \mathbf{v}_i^{(2m+1)} \cos \alpha_{r_i}^{(m)}, \\
\frac{\partial \tilde{\mathbf{v}}_i^{(2m+1)}}{\partial \alpha_{r_i}^{(m)}} &= \mathbf{v}_i^{(2m)} \cos \alpha_{r_i}^{(m)} - \mathbf{v}_i^{(2m+1)} \sin \alpha_{r_i}^{(m)}.
\end{align}
For $k \notin \{2m, 2m+1\}$, $\frac{\partial \tilde{\mathbf{v}}_i^{(k)}}{\partial \alpha_{r_i}^{(m)}} = 0$.

For $\frac{\partial \alpha_{r_i}}{\partial \theta_r}$, from Equation \eqref{eq:13}, we have:
\begin{equation}
\alpha_{r_i} = U_i \times \theta_r
\end{equation}
where 
\begin{equation}
U_i = \left(\mathbf{x}\times Map(\mathcal{L}_i)^{(0)}\right) \odot \left(\mathbf{x}\times \sum_{A_jB_j \in S_i} A_jB_j\right) \in \mathbb{R}^{1 \times d_{out}},\quad  S_i = \{A_jB_j | A_jB_j \in Map(\mathcal{L}_i)^{(1)}\}.
\end{equation}

The derivative of the vector $\alpha_{r_i} \in \mathbb{R}^{1 \times \frac{d_{out}}{2}}$ with respect to the matrix $\theta_r \in \mathbb{R}^{d_{out} \times \frac{d_{out}}{2}}$ can be computed element-wise. The element $(j, k)$ of the matrix gradient $\frac{\partial \mathcal{L}}{\partial \theta_r}$ is:
\begin{equation}
\frac{\partial \mathcal{L}}{\partial \theta_r^{(j, k)}} = \sum_{i=1}^n \sum_{m=0}^{d_{out}/2-1} \frac{\partial \mathcal{L}}{\partial \alpha_{r_i}^{(m)}} \frac{\partial \alpha_{r_i}^{(m)}}{\partial \theta_r^{(j, k)}}.
\end{equation}

The term $\frac{\partial \alpha_{r_i}^{(m)}}{\partial \theta_r^{(j, k)}}$ from $\alpha_{r_i}^{(m)} = \sum_{p=1}^{d_{out}} U_i^{(p)} \theta_r^{(p, m)}$ is:
\begin{equation}
\frac{\partial \alpha_{r_i}^{(m)}}{\partial \theta_r^{(j, k)}} = U_i^{(j)} \delta_{mk}.
\end{equation}

The term $\frac{\partial \mathcal{L}}{\partial \alpha_{r_i}^{(m)}}$ is:
\begin{equation}
\frac{\partial \mathcal{L}}{\partial \alpha_{r_i}^{(m)}} = \frac{\partial \mathcal{L}}{\partial \tilde{\mathbf{v}}_i} \frac{\partial \tilde{\mathbf{v}}_i}{\partial \alpha_{r_i}^{(m)}},
\end{equation}
where $\frac{\partial \tilde{\mathbf{v}}_i}{\partial \alpha_{r_i}^{(m)}}$ is the row vector of partials. Thus:
\begin{equation}
\frac{\partial \mathcal{L}}{\partial \alpha_{r_i}^{(m)}} = g_i \frac{\partial \mathcal{L}}{\partial \mathbf{y}} \begin{pmatrix} \frac{\partial \tilde{\mathbf{v}}_i^{(0)}}{\partial \alpha_{r_i}^{(m)}}, & \dots, & \frac{\partial \tilde{\mathbf{v}}_i^{(d_{out}-1)}}{\partial \alpha_{r_i}^{(m)}} \end{pmatrix}^\intercal.
\end{equation}
The only non-zero terms in this vector are at indices $2m, 2m+1$:
\begin{equation}
\frac{\partial \mathcal{L}}{\partial \alpha_{r_i}^{(m)}} = g_i \left( \left(\frac{\partial \mathcal{L}}{\partial \mathbf{y}}\right)^{(2m)} \frac{\partial \tilde{\mathbf{v}}_i^{(2m)}}{\partial \alpha_{r_i}^{(m)}} + \left(\frac{\partial \mathcal{L}}{\partial \mathbf{y}}\right)^{(2m+1)} \frac{\partial \tilde{\mathbf{v}}_i^{(2m+1)}}{\partial \alpha_{r_i}^{(m)}} \right).
\end{equation}

Substituting back into the gradient element $\frac{\partial \mathcal{L}}{\partial \theta_r^{(j, k)}}$:
\begin{equation}
\frac{\partial \mathcal{L}}{\partial \theta_r^{(j, k)}} = \sum_{i=1}^n \frac{\partial \mathcal{L}}{\partial \alpha_{r_i}^{(k)}} U_i^{(j)} = \sum_{i=1}^n g_i \left( \left(\frac{\partial \mathcal{L}}{\partial \mathbf{y}}\right)^{(2k)} \frac{\partial \tilde{\mathbf{v}}_i^{(2k)}}{\partial \alpha_{r_i}^{(k)}} + \left(\frac{\partial \mathcal{L}}{\partial \mathbf{y}}\right)^{(2k+1)} \frac{\partial \tilde{\mathbf{v}}_i^{(2k+1)}}{\partial \alpha_{r_i}^{(k)}} \right) U_i^{(j)}.
\end{equation}

This element-wise expression defines the gradient matrix $\frac{\partial \mathcal{L}}{\partial \theta_r} \in \mathbb{R}^{d_{out} \times \frac{d_{out}}{2}}$.

\section{Theoretical Demonstration}
\subsection{Proof of Lemmas in 4.2}\label{lem_proof}
\setcounter{lemma}{0} 
\begin{lemma}
For nested function hypothesis spaces $\mathcal{K}_1 \subseteq \mathcal{K}_2$, the optimal fitting error $\mathcal{E}_t = \inf_{f \in \mathcal{K}_t} L(f, g^*)$ of the target function $g^*$ under the loss function $L$ necessarily satisfies $\mathcal{E}_2 \leq \mathcal{E}_1$.
\end{lemma}

\begin{proof}
Let $L(f, g^*)$ be a non-negative loss function that measures the discrepancy between a hypothesis function $f$ and the target function $g^*$.

Since $\mathcal{K}_1 \subseteq \mathcal{K}_2$, every $f \in \mathcal{K}_1$ is also in $\mathcal{K}_2$. Therefore,
\[
\inf_{f \in \mathcal{K}_1} L(f, g^*) \ge \inf_{f \in \mathcal{K}_2} L(f, g^*).
\]
By definition, $\mathcal{E}_1 = \inf_{f \in \mathcal{K}_1} L(f, g^*)$ and $\mathcal{E}_2 = \inf_{f \in \mathcal{K}_2} L(f, g^*)$.
Hence, the optimal fitting error under the larger hypothesis space $\mathcal{K}_2$ is less than or equal to that under the smaller space $\mathcal{K}_1$, i.e., $\mathcal{E}_2 \le \mathcal{E}_1$.
\end{proof}

\vspace{1em}

\begin{lemma}\label{lem:22}
Define the fixed output space $\mathcal{H} = \left\{ \sum_{i=1}^n \alpha_i v_i \mid \alpha_i \ge 0,\ \sum_{i=1}^n \alpha_i = 1 \right\}$ with fixed basis vectors $\{v_i\}$. Transforming $v_i$ via input-dependent rotation $R_i(x)$ gives $\tilde{v}_i(x) = v_i R_i(x)$. Define dynamic output space $\mathcal{H}'(x) = \left\{ \sum_{i=1}^n \alpha_i \tilde{v}_i(x) \mid \alpha_i \ge 0,\ \sum_{i=1}^n \alpha_i = 1 \right\}$. The union $\mathcal{S} = \bigcup_x \mathcal{H}'(x)$ strictly contains $\mathcal{H}$, i.e., $\mathcal{S} \supset \mathcal{H}$.
\end{lemma}

\begin{proof}
Consider that when $R_i(x)$ is the identity rotation for all $i$, we have $\tilde{v}_i(x) = v_i$ and therefore $\mathcal{H}'(x) = \mathcal{H}$. Thus, for some $x$, we have $\mathcal{H} \subseteq \mathcal{H}'(x)$. Taking the union over all $x$, we obtain:
\[
\mathcal{H} \subseteq \bigcup_x \mathcal{H}'(x) = \mathcal{S}.
\]
To show the inclusion is strict, note that for generic input-dependent rotations $R_i(x)$, the rotated vectors $\tilde{v}_i(x)$ can span directions outside the original convex hull of the fixed vectors $\{v_i\}$. Therefore, for some $x$, the space $\mathcal{H}'(x)$ contains points not in $\mathcal{H}$.

As a concrete example, suppose $v_1, v_2 \in \mathbb{R}^2$ are linearly independent. If $v_1$ is rotated by an angle $\theta \ne 0$, then $\tilde{v}_1(x)$ does not lie on the line spanned by $v_1$ (assuming the rotation is applied appropriately, e.g., $v_1$ is treated as a column vector and multiplied by the rotation matrix). The convex combination $\sum_i \alpha_i \tilde{v}_i(x)$ can exit the original convex region defined by $\mathcal{H}$.

Hence, there exists $\mathbf{y} \in \mathcal{H}'(x)$ such that $\mathbf{y} \notin \mathcal{H}$. This implies that the union $\mathcal{S}$ strictly contains $\mathcal{H}$:
\[
\mathcal{S} = \bigcup_x \mathcal{H}'(x) \supset \mathcal{H}.
\]
\end{proof}

\subsection{Mitigating Underfitting}

The LoRA output hypothesis space of the existing gating mechanism is $\mathcal{K}_{\text{gate}}$. A function $f \in \mathcal{K}_{\text{gate}}$ has the form:
\[
f(x; \theta_s) = \sum_{i=1}^{n} g_i(x; \theta_s) v_i(x)
\]
Here, $v_i(x) = xA_iB_i \in \mathbb{R}^{1 \times d_{\text{out}}}$ is the output of the $i$-th LoRA module for input $x$. The gating weight $g_i(x; \theta_s)$ is determined by $x$ and parameters $\theta_s \in \mathbb{R}^{n d_{\text{out}} \times n}$. Specifically, $g_i$ comes from applying a non-linear function to $[v_1(x), \dots, v_n(x)]\theta_s$.

The LoRA output hypothesis space of \textit{RadarGate} is $\mathcal{K}_{\text{ours}}$. A function $f \in \mathcal{K}_{\text{ours}}$ has the form:
\[
f(x; \theta_s, \theta_r) = \sum_{i=1}^{n} G_i(x; \theta_s, \theta_r) v_i(x)
\]
Here, $G_i(x; \theta_s, \theta_r)$ is the gating weight of \textit{RadarGate}, determined by $x$ and parameters $\theta_s \in \mathbb{R}^{n d_{\text{out}} \times n}$ and $\theta_r \in \mathbb{R}^{d_{\text{out}} \times d_{\text{out}}/2}$. Specifically, $G_i$ comes from applying a non-linear function to $[\tilde{x}_1(x; \theta_r), \dots, \tilde{x}_n(x; \theta_r)]\theta_s$, where $\tilde{x}_j(x; \theta_r) = v_j(x)R_j(x; \theta_r)$. The rotation $R_j$ depends on the angle $\alpha_{rj}(x; \theta_r)$, which is calculated from $x$ and $\theta_r$.

To prove $\mathcal{K}_{\text{gate}} \subseteq \mathcal{K}_{\text{ours}}$, it suffices to show that any $f \in \mathcal{K}_{\text{gate}}$ can also be represented as a function in $\mathcal{K}_{\text{ours}}$.
Take any $f(x; \theta_s^{\text{gate}}) = \sum_{i} g_i(x; \theta_s^{\text{gate}}) v_i(x) \in \mathcal{K}_{\text{gate}}$.
Choose the parameters of \textit{RadarGate} as $\theta_s = \theta_s^{\text{gate}}$ and $\theta_r = 0$ (the zero matrix).
When $\theta_r = 0$, the formula for calculating $\alpha_{rj}(x; 0)$ (Equation (8)) results in a zero vector, and the corresponding rotation matrix $R_j(x; 0)$ is the identity matrix $I$. At this point:
\[
\tilde{x}_j(x; 0) = v_j(x)I = v_j(x)
\]
Therefore, the calculation of \textit{RadarGate}'s gating weights $G_i(x; \theta_s^{\text{gate}}, 0)$ (based on $[\tilde{x}_1(x; 0), \dots, \tilde{x}_n(x; 0)]\theta_s^{\text{gate}})$ is exactly the same as the calculation of the existing gating weights $g_i(x; \theta_s^{\text{gate}})$ (based on $[v_1(x), \dots, v_n(x)]\theta_s^{\text{gate}}$). This means:
\[
G_i(x; \theta_s^{\text{gate}}, 0) = g_i(x; \theta_s^{\text{gate}}) \quad \text{for all } i, x
\]
Thus:
\[
f(x; \theta_s^{\text{gate}}) = \sum_{i} g_i(x; \theta_s^{\text{gate}}) v_i(x) = \sum_{i} G_i(x; \theta_s^{\text{gate}}, 0) v_i(x)
\]
which shows that $f(x; \theta_s^{\text{gate}}) \in \mathcal{K}_{\text{ours}}$.

We have shown that $\mathcal{K}_{\text{gate}} \subseteq \mathcal{K}_{\text{ours}}$. According to Lemma 1, for the target function $g^* = \Delta y_{\text{target}} - xW$, the optimal fitting error of \textit{RadarGate} $E_{\text{ours}}$ and the existing method $E_{\text{gate}}$ satisfy:
\[
E_{\text{ours}} = \inf_{f \in \mathcal{K}_{\text{ours}}} \mathcal{L}(f, g^*) \leq \inf_{f \in \mathcal{K}_{\text{gate}}} \mathcal{L}(f, g^*) = E_{\text{gate}}
\]
A larger hypothesis space and lower optimal error potential alleviate underfitting.

\subsection{ Enhancing Generalization Ability}

For a given input $x$, the output of the existing gating mechanism is $\Delta y_{\text{model}}(x) = \sum_{i} g_i(x; \theta_s) v_i(x)$. This is a convex combination of the vector set $\{v_i(x)\}_{i=1}^{n}$ (given $g_i \geq 0, \sum g_i = 1$). Its output space is limited within the convex hull $H(x)$:
\[
H(x) = \text{conv}(\{v_i(x)\}_{i=1}^{n})
\]
\textit{RadarGate}'s gating weights $G_i(x; \theta_s, \theta_r)$ are calculated based on the rotated vectors $\tilde{x}_j(x; \theta_r)$. To align with Lemma 2 regarding the convex hull of rotated vectors, we consider the effective output space of \textit{RadarGate} to be the convex hull formed by the rotated vectors. Define $\mathcal{H}'(x; \theta_r)$ as the convex hull spanned by the rotated vectors $\{\tilde{x}_i(x; \theta_r)\}_{i=1}^{n}$:
\[
\mathcal{H}'(x; \theta_r) = \text{conv}(\{\tilde{x}_i(x; \theta_r)\}_{i=1}^{n})
\]

Lemma 2 states that the union of convex hulls after applying input-dependent rotations $R_i(x)$ to a fixed set of basis vectors $\{v_i\}$ strictly contains the original convex hull:
\[
\bigcup_{x} \text{conv}(\{v_i R_i(x)\}) \supset \text{conv}(\{v_i\})
\]

Consider the vector set $\{v_i(x)\}_{i=1}^{n}$ corresponding to each input $x$ as the "basis vectors" and apply the input-dependent rotation $R_i(x; \theta_r)$. Although the original statement of the lemma is for fixed $v_i$, the core idea is that rotation changes the vector direction and expands the reachable space. \textit{RadarGate} uses $x$ and $\theta_r$ to calculate $R_i(x; \theta_r)$ to rotate $v_i(x)$ to obtain $\tilde{x}_i(x; \theta_r)$. The effective output space of \textit{RadarGate} is the set of outputs under all possible inputs $x$ and all possible parameters $\theta_s, \theta_r$. This set includes points within the dynamic convex hull $\mathcal{H}'(x; \theta_r)$. Consider the union of $\mathcal{H}'(x; \theta_r)$ over all possible $\theta_r$ and $x$:
\[
\bigcup_{\theta_r} \bigcup_{x} \mathcal{H}'(x; \theta_r) = \bigcup_{\theta_r} \bigcup_{x} \text{conv}(\{\tilde{x}_i(x; \theta_r)\}_{i=1}^{n})
\]
According to the conclusion of Lemma 2, due to the input dependence of the rotation $R_i(x; \theta_r)$, this union space strictly contains the union of convex hulls spanned only by the original $v_i(x)$, i.e., $\bigcup_{x} H(x) = \bigcup_{x} \text{conv}(\{v_i(x)\}_{i=1}^{n})$.

\textit{RadarGate}'s effective output space $\bigcup_{\theta_r} \bigcup_{x} \mathcal{H}'(x; \theta_r)$ is larger than the output space of existing methods. This means \textit{RadarGate} can generate vectors that lie outside the original convex hull $H(x)$ corresponding to any single input $x$. This allows the model to match target outputs $\Delta y_{\text{target}}(x)$ that need to fall outside of $H(x)$, thereby improving generalization ability.

Through the above mathematical definitions and application of the lemmas, it can be rigorously explained how \textit{RadarGate}'s rotation mechanism enhances the model's fitting and generalization abilities from the perspectives of hypothesis space and effective output space, respectively.

\subsection{Computational and Memory Complexity}\vspace{-5pt}\label{app:complex}
In this subsection, we will analyze the Computational and Memory Complexity of \textit{RadarGate} and existing gating architectures in detail. Theoretically, we will prove that our method is on par with existing ones in terms of complexity, and the newly added module incurs minimal extra cost.  

Let the input $x$ be of size $L \times d_{\text{in}}$, where $L$ is the sequence length. In the code implementation, we applied low-rank factorization to the \textit{RotationGate} parameters, splitting them into two matrices of size $d_{\text{in}} \times r_{\text{a}}$ and $r_{\text{a}} \times d_{\text{in}}$, respectively.

\textbf{Computational Complexity.}
From Equation \eqref{eq:long_output} and \eqref{eq:11}, the computational complexities $O_{\text{s}}$ of existing gating architectures and $O_{\text{r}}$ of our \textit{RadarGate} can be derived as follows:
\begin{equation}
O_{\text{s}} = L\left[ n d_{\text{in}} + k r (d_{\text{in}} + d_{\text{out}}) + k d_{\text{out}} \right],
O_{\text{r}} = L\left[ (2n + 2r_{\text{a}})d_{\text{in}} + k r (d_{\text{in}} + d_{\text{out}}) + k d_{\text{out}} \right].
\end{equation}

\textbf{Memory Complexity.}
Following the same reasoning, the memory complexities $M_{\text{s}}$ of existing gating architectures and $M_{\text{r}}$ of our \textit{RadarGate} can be derived as follows:
\begin{equation}
M_{\text{s}} = n\left[ d_{\text{in}} + r(d_{\text{in}} + d_{\text{out}}) \right] + L \cdot (n + k r),
M_{\text{r}} = n\left[ (2r_{\text{a}} + 1)d_{\text{in}} + r(d_{\text{in}} + d_{\text{out}}) \right] + L \cdot (n + k r).
\end{equation}

Given that $n, r, k, r_{\text{a}} \ll \min\{d_{\text{in}}, d_{\text{out}}\}$, the complexities simplify to:
\begin{equation}
O_{\text{s}} = O(L \cdot \min\{d_{\text{in}}, d_{\text{out}}\}) = O_{\text{r}}, M_{\text{s}} \approx M_{\text{r}}
\end{equation}
This indicates that the computational and memory complexities of \textit{RadarGate} and existing gating methods are asymptotically equivalent, belonging to the same order of magnitude.
\section{Fitting Capability of LoRAs}

Table \ref{tab:app_exp_ind_le} presents the performance of different gating structures within the three learnable gating architectures—HydraLoRA, MoLE, and OMoE—when the training and test sets are from the same distribution. The underlying LoRA blocks are composed of the nine tasks listed in the table. Notably, our approach achieves the highest accuracy in over 90\% of cases under the fitting experiments, significantly outperforming existing methods. Table \ref{tab:app_exp_ind_rb} reports the independence of LoRA modules under existing rule-based methods, showing that their overall performance is substantially lower than that of learnable methods and consistently inferior to the single LoRA approach in all cases. Based on all the presented results, it is evident that \textit{RadarGate} demonstrates a superior ability to preserve the independence of LoRA modules.
\begin{table}[htbp]
\setlength{\abovecaptionskip}{-0.1cm}
\setlength{\belowcaptionskip}{0cm}
\caption{Fitting Capability Experiments of Learnable Gate Architecture.}
\label{tab:app_exp_ind_le}
\vskip 0.15in
\renewcommand{\arraystretch}{1.2}
\begin{center}
\begin{large}
\begin{sc}
\resizebox{\textwidth}{!}{
\begin{tabular}{lp{2cm}<{\centering}ccccccccccccccc}
\toprule
{\fontsize{14}{16}\selectfont \multirow{2}{*}[-1.2ex]{\textbf{TASK}}} & {\fontsize{14}{16}\selectfont \multirow{2}{*}[-1.2ex]{\textbf{TOPK}}}& \multicolumn{5}{c}{\textbf{HydraLoRa}} & \multicolumn{5}{c}{\textbf{MoLE}} & \multicolumn{5}{c}{\textbf{OMoE}}  \\
\cmidrule(r){3-7}\cmidrule(r){8-12}\cmidrule(r){13-17}
 & & \textbf{\makecell{Stretch-\\Only Gate}} & \textbf{\makecell{Rotation-\\Only Gate}} & \textbf{Nexus} & \textbf{Tutel} & \textbf{\makecell{\textit{RadarGate}\\(ours)}} & \textbf{\makecell{Stretch-\\Only Gate}} & \textbf{\makecell{Rotation-\\Only Gate}} & \textbf{Nexus} & \textbf{Tutel} & \textbf{\makecell{\textit{RadarGate}\\(ours)}} & \textbf{\makecell{Stretch-\\Only Gate}} & \textbf{\makecell{Rotation-\\Only Gate}} & \textbf{Nexus} & \textbf{Tutel} & \textbf{\makecell{\textit{RadarGate}\\(ours)}} \\
\toprule
\Large \multirow{6}{*}{\textbf{SENSEMAKING}}
& \Large 1 & 52.95\% & 53.28\% & 49.38\% & \underline{58.24\%} & \textbf{62.26\%} &  44.99\% & 41.01\% & \underline{47.11\%} & 40.97\% & \textbf{50.64\%} &  58.66\% & 59.76\% & 48.86\% & \underline{62.82\%} & \textbf{65.71\%} \\
& \Large 2 & 55.54\% & \textbf{72.21\%} & \underline{70.64\%} & 51.30\% & \textbf{72.21\%} &  \underline{44.72\%} & 36.15\% & 38.85\% & 39.29\% & \textbf{51.60\%} &  70.43\% & 71.03\% & \underline{75.00\%} & 59.64\% & \textbf{79.65\%} \\
& \Large 3 & 67.78\% & \textbf{87.78\%} & \underline{84.90\%} & 84.16\% & \textbf{87.78\%} &  55.70\% & 54.44\% & 58.37\% & \underline{60.57\%} & \textbf{67.67\%} &  77.78\% & 83.45\% & \underline{84.24\%} & 84.05\% & \textbf{85.67\%} \\
& \Large 4 & 77.96\% & 65.71\% & \underline{78.26\%} & 66.03\% & \textbf{80.00\%} &  \textbf{70.27\%} & 51.87\% & 59.60\% & 55.02\% & \underline{64.20\%} &  \textbf{79.37\%} & 74.76\% & 75.46\% & 72.58\% & \underline{77.20\%} \\
& \Large 5 & 76.26\% & 75.99\% & 68.24\% & \underline{78.79\%} & \textbf{80.00\%} &  70.78\% & \underline{74.99\%} & 69.01\% & \textbf{75.08\%} & \textbf{75.08\%} &  72.29\% & \underline{79.36\%} & 62.64\% & 65.76\% & \textbf{80.00\%} \\
& \Large 6 & 72.07\% & \underline{75.76\%} & 66.46\% & 64.12\% & \textbf{76.17\%} &  \underline{77.02\%} & 69.01\% & \textbf{80.00\%} & 70.83\% & \textbf{80.00\%} &  \underline{68.45\%} & \textbf{68.46\%} & 65.82\% & 64.14\% & \textbf{68.46\%} \\
\cmidrule{1-17}
\Large \multirow{6}{*}{\textbf{QASC}}
& \Large 1 & 18.88\% & \underline{20.53\%} & 20.04\% & 18.93\% & \textbf{23.25\%} &  15.92\% & 18.54\% & 18.33\% & \underline{20.12\%} & \textbf{21.84\%} &  15.49\% & 17.03\% & 15.73\% & \underline{20.03\%} & \textbf{20.49\%} \\
& 2 & 20.92\% & 18.22\% & 20.30\% & \underline{20.94\%} & \textbf{28.42\%} &  21.64\% & \underline{21.89\%} & \textbf{28.61\%} & 18.86\% & \textbf{28.61\%} &  15.28\% & 20.46\% & \underline{23.78\%} & 16.52\% & \textbf{25.04\%} \\
& \Large 3 & 26.15\% & \textbf{32.31\%} & 29.54\% & \underline{30.37\%} & \textbf{32.31\%} &  19.00\% & 24.96\% & 17.87\% & \underline{25.89\%} & \textbf{30.31\%} &  21.54\% & 20.52\% & 24.32\% & \underline{26.99\%} & \textbf{28.33\%} \\
& \Large 4 & 24.31\% & 21.78\% & \underline{25.19\%} & 20.95\% & \textbf{25.62\%} &  23.76\% & \underline{27.32\%} & 26.24\% & 25.65\% & \textbf{29.59\%} &  23.38\% & 21.52\% & 19.30\% & \underline{24.03\%} & \textbf{24.77\%} \\
& \Large 5 & \textbf{29.97\%} & 26.62\% & 24.73\% & 27.46\% & \underline{27.66\%} &  \textbf{32.70\%} & 30.25\% & 27.85\% & 31.45\% & \underline{32.19\%} &  \textbf{23.88\%} & 21.94\% & 20.26\% & 21.41\% & \underline{23.24\%} \\
& \Large 6 & \textbf{26.42\%} & 25.55\% & 24.24\% & \underline{26.19\%} & \underline{26.19\%} &  \textbf{27.84\%} & \underline{27.08\%} & 26.71\% & 23.61\% & \underline{27.08\%} &  \textbf{21.14\%} & 19.33\% & 18.65\% & 17.73\% & \underline{20.65\%} \\
\cmidrule{1-17}
\Large \multirow{6}{*}{\textbf{STRATEGYQA}}
& \Large 1 & 32.59\% & 26.37\% & 27.62\% & \underline{36.10\%} & \textbf{36.40\%} &  28.99\% & 23.25\% & 24.10\% & \underline{31.00\%} & \textbf{32.33\%} &  34.14\% & 27.40\% & 33.74\% & \underline{43.66\%} & \textbf{43.93\%} \\
& \Large 2 & 29.98\% & 24.01\% & 31.35\% & \underline{31.75\%} & \textbf{35.05\%} &  \underline{28.19\%} & \textbf{30.93\%} & 27.14\% & 26.77\% & \textbf{30.93\%} &  \underline{49.22\%} & 39.53\% & 48.61\% & 47.75\% & \textbf{52.29\%} \\
& \Large 3 & 41.88\% & 42.30\% & 41.56\% & \underline{43.55\%} & \textbf{45.56\%} &  40.16\% & 37.98\% & \underline{41.71\%} & 40.60\% & \textbf{45.56\%} &  \underline{45.56\%} & 43.96\% & \textbf{55.56\%} & 44.01\% & \textbf{55.56\%} \\
& \Large 4 & \textbf{42.02\%} & 38.05\% & 37.42\% & 37.41\% & \underline{38.31\%} &  \underline{40.70\%} & 32.32\% & 34.13\% & 34.84\% & \textbf{42.19\%} &  \textbf{60.01\%} & 40.90\% & 50.44\% & 55.21\% & \underline{59.05\%} \\
& \Large 5 & 40.19\% & 47.83\% & \underline{48.95\%} & 36.27\% & \textbf{52.80\%} &  \underline{42.00\%} & 35.07\% & \textbf{53.41\%} & 34.36\% & \textbf{53.41\%} &  \underline{61.18\%} & 49.76\% & \textbf{61.83\%} & 50.44\% & \textbf{61.83\%} \\
& \Large 6 & 40.10\% & \textbf{50.18\%} & 42.04\% & \underline{45.30\%} & \textbf{50.18\%} &  \underline{40.41\%} & 39.07\% & 38.43\% & 38.21\% & \textbf{43.18\%} &  57.79\% & 57.50\% & \underline{62.07\%} & 50.71\% & \textbf{66.09\%} \\
\cmidrule{1-17}
\Large \multirow{6}{*}{\textbf{AQUA}}
& \Large 1 & 7.87\% & 7.72\% & 12.45\% & \underline{15.78\%} & \textbf{16.12\%} &  8.21\% & 9.41\% & 7.49\% & \underline{11.29\%} & \textbf{12.78\%} &  8.60\% & \underline{12.79\%} & 7.13\% & 7.54\% & \textbf{13.67\%} \\
& \Large 2 & 8.89\% & \underline{13.38\%} & 10.37\% & 12.24\% & \textbf{14.90\%} &  7.98\% & 7.49\% & \underline{11.88\%} & 11.87\% & \textbf{12.74\%} &  11.43\% & 12.61\% & 10.37\% & \underline{13.48\%} & \textbf{13.51\%} \\
& \Large 3 & 11.11\% & \underline{19.08\%} & 11.10\% & 16.71\% & \textbf{21.11\%} &  11.03\% & \textbf{18.11\%} & \underline{12.81\%} & 11.51\% & \textbf{18.11\%} &  10.00\% & \textbf{17.19\%} & 11.44\% & \underline{14.14\%} & \textbf{17.19\%} \\
& \Large 4 & 14.16\% & 10.50\% & \underline{15.60\%} & 10.54\% & \textbf{16.43\%} &  11.56\% & 10.31\% & \underline{13.77\%} & 10.28\% & \textbf{14.21\%} &  12.59\% & 11.68\% & \underline{14.84\%} & 9.65\% & \textbf{15.61\%} \\
& \Large 5 & 14.05\% & 12.48\% & \textbf{18.84\%} & \underline{14.77\%} & \textbf{18.84\%} &  \underline{16.58\%} & 16.52\% & 16.18\% & 13.86\% & \textbf{18.44\%} &  14.63\% & 12.66\% & \underline{17.05\%} & \textbf{20.55\%} & \textbf{20.55\%} \\
& \Large 6 & \underline{18.60\%} & \textbf{19.68\%} & 15.34\% & 15.17\% & \textbf{19.68\%} &  \underline{18.95\%} & 18.76\% & 16.26\% & 16.17\% & \textbf{20.20\%} &  \underline{16.45\%} & 13.79\% & 14.41\% & 13.71\% & \textbf{18.36\%} \\
\Large \multirow{6}{*}{\textbf{QED}}
& \Large 1 & 53.62\% & 46.10\% & 57.36\% & \underline{58.32\%} & \textbf{63.80\%} &  44.35\% & 59.84\% & \underline{62.89\%} & 44.67\% & \textbf{66.59\%} &  \underline{55.00\%} & 50.19\% & 50.66\% & 53.57\% & \textbf{63.59\%} \\
& \Large 2 & 55.90\% & \underline{57.08\%} & 56.46\% & 54.97\% & \textbf{58.84\%} &  45.85\% & \underline{57.09\%} & 53.91\% & \textbf{69.42\%} & \textbf{69.42\%} &  56.16\% & \underline{66.59\%} & 49.43\% & 65.82\% & \textbf{69.37\%} \\
& \Large 3 & 68.89\% & \underline{89.44\%} & \textbf{90.00\%} & 89.36\% & \textbf{90.00\%} &  61.67\% & \textbf{88.00\%} & 84.55\% & \underline{85.48\%} & \textbf{88.00\%} &  66.67\% & \underline{84.02\%} & 83.81\% & 66.43\% & \textbf{88.33\%} \\
& \Large 4 & \textbf{66.57\%} & 58.45\% & 59.82\% & 59.35\% & \underline{65.84\%} &  61.42\% & 62.08\% & 49.79\% & \underline{63.44\%} & \textbf{80.00\%} &  \textbf{80.00\%} & \underline{71.75\%} & 69.72\% & 70.50\% & \textbf{80.00\%} \\
& \Large 5 & 65.90\% & \underline{76.34\%} & 75.86\% & 66.05\% & \textbf{78.17\%} &  \underline{78.00\%} & 66.58\% & 66.65\% & \textbf{78.48\%} & \textbf{78.48\%} &  \textbf{80.00\%} & 72.87\% & 69.75\% & \underline{77.51\%} & \underline{77.51\%} \\
& \Large 6 & 63.41\% & \underline{63.61\%} & 57.39\% & 61.98\% & \textbf{65.27\%} &  \underline{71.80\%} & 68.54\% & 58.23\% & 57.57\% & \textbf{80.00\%} &  \underline{72.53\%} & 64.59\% & \textbf{74.20\%} & 65.33\% & \textbf{74.20\%} \\
\cmidrule{1-17}
\Large \multirow{6}{*}{\textbf{CREAK}}
& \Large 1 & 46.76\% & 46.82\% & \underline{48.41\%} & 44.57\% & \textbf{51.24\%} &  43.83\% & \underline{49.13\%} & 42.18\% & 48.97\% & \textbf{51.57\%} &  46.06\% & 38.61\% & 46.30\% & \underline{49.60\%} & \textbf{54.05\%} \\
& \Large 2 & 56.11\% & 49.11\% & 48.41\% & \underline{61.29\%} & \textbf{62.15\%} &  43.01\% & 36.57\% & \underline{51.64\%} & 36.06\% & \textbf{53.90\%} &  51.95\% & 65.30\% & \underline{68.52\%} & 63.73\% & \textbf{69.09\%} \\
& \Large 3 & \underline{70.00\%} & 66.32\% & 68.76\% & 66.32\% & \textbf{71.11\%} &  60.23\% & 60.07\% & 59.63\% & \underline{60.34\%} & \textbf{67.33\%} &  63.33\% & \textbf{69.67\%} & \underline{66.52\%} & 65.47\% & \textbf{69.67\%} \\
& \Large 4 & \textbf{63.84\%} & 57.05\% & 58.51\% & 61.11\% & \underline{63.04\%} &  \textbf{56.64\%} & 51.78\% & 50.42\% & 51.51\% & \underline{51.87\%} &  \underline{57.45\%} & 51.95\% & 55.36\% & 51.20\% & \textbf{63.19\%} \\
& \Large 5 & 62.74\% & \textbf{80.00\%} & 51.81\% & \underline{79.94\%} & \textbf{80.00\%} &  \underline{53.45\%} & 49.41\% & \textbf{64.49\%} & 52.89\% & \textbf{64.49\%} &  \underline{55.38\%} & 48.90\% & 47.43\% & 53.87\% & \textbf{80.00\%} \\
& \Large 6 & 57.85\% & \underline{65.66\%} & \textbf{69.45\%} & 60.41\% & \textbf{69.45\%} &  51.49\% & 55.31\% & \underline{58.54\%} & \textbf{69.55\%} & \textbf{69.55\%} &  73.56\% & \underline{79.66\%} & 76.21\% & 60.47\% & \textbf{80.00\%} \\
\cmidrule{1-17}
\Large \multirow{6}{*}{\textbf{ECQA}}
& \Large 1 & 34.92\% & 34.93\% & \underline{41.59\%} & 36.08\% & \textbf{41.87\%} &  28.37\% & 23.37\% & \underline{31.39\%} & 29.86\% & \textbf{35.67\%} &  33.88\% & 29.35\% & 28.08\% & \underline{38.58\%} & \textbf{40.90\%} \\
& \Large 2 & \underline{41.31\%} & 34.29\% & 39.27\% & \textbf{45.44\%} & \textbf{45.44\%} &  43.61\% & 38.91\% & 39.92\% & \underline{45.69\%} & \textbf{46.44\%} &  36.50\% & 42.67\% & 32.84\% & \underline{45.29\%} & \textbf{48.63\%} \\
& \Large 3 & 46.67\% & \underline{57.56\%} & 43.55\% & \textbf{58.89\%} & \textbf{58.89\%} &  40.12\% & 39.15\% & 43.29\% & \underline{46.54\%} & \textbf{47.78\%} &  47.78\% & 56.54\% & \underline{57.50\%} & 47.13\% & \textbf{57.78\%} \\
& \Large 4 & \textbf{45.85\%} & 43.76\% & 37.42\% & 44.80\% & \underline{44.87\%} &  38.57\% & 41.99\% & 41.12\% & \underline{46.57\%} & \textbf{46.60\%} &  \underline{50.51\%} & 46.52\% & 48.22\% & 41.28\% & \textbf{50.89\%} \\
& \Large 5 & 44.42\% & \underline{65.85\%} & 42.76\% & 49.84\% & \textbf{66.04\%} &  \underline{40.83\%} & 37.23\% & \textbf{43.57\%} & 39.28\% & \textbf{43.57\%} &  \textbf{67.57\%} & 62.91\% & 62.94\% & 55.19\% & \underline{64.83\%} \\
& \Large 6 & 41.05\% & 41.18\% & \underline{58.32\%} & 49.07\% & \textbf{59.20\%} &  \underline{38.73\%} & 31.63\% & \textbf{39.64\%} & 32.26\% & \textbf{39.64\%} &  64.71\% & \textbf{67.69\%} & 64.92\% & \underline{67.25\%} & \textbf{67.69\%} \\
\cmidrule{1-17}
\Large \multirow{6}{*}{\textbf{ESNLI}}
& \Large 1 & 37.11\% & 39.54\% & 33.14\% & \underline{40.52\%} & \textbf{42.96\%} &  36.77\% & 41.62\% & \underline{41.87\%} & 41.26\% & \textbf{42.26\%} &  42.09\% & \underline{45.60\%} & 37.32\% & 35.75\% & \textbf{46.88\%} \\
& \Large 2 & 50.43\% & 41.64\% & \underline{50.69\%} & 49.06\% & \textbf{51.75\%} &  41.87\% & \underline{42.86\%} & 41.07\% & \textbf{47.15\%} & \textbf{47.15\%} &  42.90\% & \underline{53.65\%} & 46.94\% & 50.99\% & \textbf{54.87\%} \\
& \Large 3 & 50.00\% & \textbf{60.00\%} & 55.31\% & \underline{57.32\%} & \textbf{60.00\%} &  44.64\% & \underline{59.96\%} & 56.37\% & \textbf{60.00\%} & \textbf{60.00\%} &  54.44\% & 55.62\% & \underline{57.04\%} & \textbf{60.00\%} & \textbf{60.00\%} \\
& \Large 4 & \underline{53.14\%} & 42.08\% & 52.03\% & 40.98\% & \textbf{54.37\%} &  43.10\% & \underline{49.89\%} & 36.03\% & 44.63\% & \textbf{51.23\%} &  55.86\% & 44.41\% & \underline{58.07\%} & 52.90\% & \textbf{58.18\%} \\
& \Large 5 & \textbf{60.21\%} & 52.27\% & \underline{56.03\%} & 50.32\% & \underline{56.03\%} &  41.83\% & \underline{54.14\%} & \textbf{65.10\%} & 44.06\% & \textbf{65.10\%} &  \underline{54.67\%} & \textbf{56.46\%} & 44.89\% & 45.97\% & \textbf{56.46\%} \\
& \Large 6 & \textbf{53.97\%} & 48.61\% & 48.76\% & 51.38\% & \underline{53.06\%} &  44.99\% & \underline{58.31\%} & 45.69\% & 36.42\% & \textbf{64.64\%} &  \underline{49.53\%} & 46.86\% & 48.09\% & \textbf{50.36\%} & \textbf{50.36\%} \\
\cmidrule{1-17}
\Large \multirow{6}{*}{\textbf{GSM8K}}
& \Large 1 & 11.96\% & \underline{13.23\%} & 12.60\% & 12.32\% & \textbf{15.00\%} &  \underline{12.63\%} & 11.38\% & 10.75\% & 11.01\% & \textbf{14.36\%} &  9.62\% & 8.91\% & \underline{11.33\%} & 9.11\% & \textbf{15.66\%} \\
& 2 & \underline{13.25\%} & 11.38\% & 13.13\% & \textbf{13.90\%} & \textbf{13.90\%} &  12.26\% & 11.75\% & \underline{13.66\%} & \textbf{13.88\%} & \textbf{13.88\%} &  11.07\% & 9.18\% & 12.67\% & \underline{13.57\%} & \textbf{14.44\%} \\
& \Large 3 & 15.33\% & \textbf{20.67\%} & \underline{18.05\%} & 16.29\% & \textbf{20.67\%} &  12.33\% & \textbf{19.33\%} & \underline{16.15\%} & 14.34\% & \textbf{19.33\%} &  12.67\% & 17.44\% & 18.06\% & \underline{18.28\%} & \textbf{21.33\%} \\
& 4 & 15.48\% & \underline{20.21\%} & 18.61\% & 14.13\% & \textbf{22.55\%} &  11.24\% & 11.09\% & \underline{13.39\%} & 12.06\% & \textbf{13.58\%} &  \underline{14.66\%} & 11.28\% & 14.11\% & 12.42\% & \textbf{16.51\%} \\
& 5 & 17.42\% & 15.56\% & 20.01\% & \underline{20.43\%} & \textbf{21.61\%} &  13.22\% & \textbf{18.23\%} & 13.97\% & \underline{17.55\%} & \textbf{18.23\%} &  13.48\% & \textbf{18.72\%} & \underline{15.53\%} & 12.31\% & \textbf{18.72\%} \\
& \Large 6 & 17.11\% & \underline{17.66\%} & 16.24\% & 13.88\% & \textbf{19.04\%} &  \underline{14.49\%} & 14.12\% & 12.16\% & 11.53\% & \textbf{15.59\%} &  12.57\% & 13.94\% & 13.05\% & \underline{16.97\%} & \textbf{18.11\%} \\
\bottomrule
\end{tabular}
} 
\end{sc}
\end{large}
\end{center}
\vskip -0.1in
\end{table}

\begin{table}[htbp]
\setlength{\abovecaptionskip}{-0.1cm}
\setlength{\belowcaptionskip}{0cm}
    \caption{Fitting Capability Experiments of Rule-based Gate Architecture.}
    \label{tab:app_exp_ind_rb}
    \vskip 0.15in
    \renewcommand{\arraystretch}{1.2}
    \begin{center}
    \begin{small}
    \begin{sc}
    \scalebox{1}{
    \begin{tabular}{l|c|ccc} 
    \hline
    \fontsize{10pt}{12pt}\selectfont
    \textbf{Task} & \textbf{direct lora} & \textbf{Lorahub}& \textbf{Arrow}& \textbf{PEMs} \\
    \hline
         \textbf{SENSEMAKING} & \textbf{45.99\%} & 34.24\% & 29.56\% & \underline{41.33\%} \\
    
          \textbf{QASC} & \textbf{30.01\%} & \underline{22.34\%} & 19.62\% & 20.40\% \\
    
          \textbf{STRATEGYQA} & \textbf{35.34\%} & 23.95\% & \underline{27.11\%} & 15.42\% \\
    
          \textbf{AUQA} & \textbf{19.55\%} & \underline{14.02\%} & 13.92\% & 9.21\% \\
    
          \textbf{QED} & \textbf{39.63\%} & 21.85\% & \underline{35.59\%} & 32.61\% \\
    
          \textbf{CREAK} & \textbf{46.00\%} & 27.53\% & 28.52\% & \underline{39.42\%} \\
    
          \textbf{ECQA} & \textbf{33.89\%} & \underline{26.57\%} & 19.67\% & 20.52\% \\
    
          \textbf{ESNLI} & \textbf{41.76\%} & \underline{28.57\%} & 19.67\% & 22.13\% \\
    
          \textbf{GSM8K} & \textbf{17.14\%} & 10.57\% & 8.67\% & \underline{11.12\%} \\
    \hline
    \end{tabular}
   }
    \end{sc}
    \end{small}
    \end{center}
    \vskip -0.1in
    \end{table}

\section{Generalization Experiments}
\label{app:exp_gen}
This section presents the results of two key experiments in the Generalization study, focusing on the independence and generalization capabilities of LoRA blocks. Tables \ref{tab:app_exp_ind_rb} and \ref{tab:app_exp_ind_le} correspond to the experiments on independence, showcasing the performance of the rule-based and learnable methods, respectively. Tables \ref{tab:app_exp_gen1} and \ref{tab:app_exp_gen2} report the generalization experiments related to the learnable gate, while Table \ref{tab:app_exp_gen_rb1} and \ref{tab:app_exp_gen_rb2} present the results for the rule-based gate.

Tables \ref{tab:app_exp_gen1} and \ref{tab:app_exp_gen2} present the generalization performance of five different gating structures within the three learnable gating architectures—HydraLoRA, MoLE, and OMoE—when the training and test sets are from different distributions. The experiments are conducted on a base model consisting of five LoRA blocks. Overall, our approach significantly outperforms existing methods across nearly all tasks, demonstrating consistently strong performance across different top-k settings. In over 80\% of cases, our method achieves the best results. Table \ref{tab:app_exp_gen_rb1} and \ref{tab:app_exp_gen_rb2} report the generalization performance of existing rule-based methods, which remain substantially inferior to learnable approaches. Based on the results from all tables, it is evident that \textit{RadarGate} exhibits the strongest generalization capability across a wide range of experiments.
\begin{table}[htbp]
\setlength{\abovecaptionskip}{-0.1cm}
\setlength{\belowcaptionskip}{0cm}
\caption{Generalization Experiments of Learnable Gate Architecture.}
\label{tab:app_exp_gen1}
\vskip 0.15in
\renewcommand{\arraystretch}{1.2}
\begin{center}
\begin{large}
\begin{sc}
\resizebox{\textwidth}{!}{ 
\begin{tabular}{lcp{1.7cm}<{\centering}ccccccccccccccc} 
\toprule
{\fontsize{15}{16}\selectfont  \textbf{\textsc{\multirow{2}{*}[-1.2ex]{Benchmark}}}} &   {\fontsize{15}{16}\selectfont \multirow{2}{*}[-1.2ex]{\textbf{Task}}} & {\fontsize{14}{16}\selectfont \multirow{2}{*}[-1.2ex]{\textbf{TopK}}}& \multicolumn{5}{c}{\textbf{HydraLoRa}} & \multicolumn{5}{c}{\textbf{MoLE}} & \multicolumn{5}{c}{\textbf{OMoE}}  \\
\cmidrule(r){4-8}\cmidrule(r){9-13}\cmidrule(r){14-18}
& & & \textbf{\makecell{Stretch-\\Only Gate}} & \textbf{\makecell{Rotation-\\Only Gate}} & \textbf{Nexus} & \textbf{Tutel} & \textbf{\makecell{\textit{RadarGate}\\(ours)}} & \textbf{\makecell{Stretch-\\Only Gate}} & \textbf{\makecell{Rotation-\\Only Gate}} & \textbf{Nexus} & \textbf{Tutel} & \textbf{\makecell{\textit{RadarGate}\\(ours)}} & \textbf{\makecell{Stretch-\\Only Gate}} & \textbf{\makecell{Rotation-\\Only Gate}} & \textbf{Nexus} & \textbf{Tutel} & \textbf{\makecell{\textit{RadarGate}\\(ours)}} \\
\toprule
{\fontsize{15}{16}\selectfont\ \textbf{\multirow{35}{*}{GLUE}}}
& \Large \multirow{5}{*}{SST-2}
& \Large 1 & \underline{46.00\%} & 42.35\% & 42.29\% & 45.07\% & \textbf{56.00\%} &  35.33\% & 31.20\% & 35.86\% & \underline{46.38\%} & \textbf{50.00\%}& 31.82\% & 33.33\% & 30.24\% & \underline{35.37\%} & \textbf{36.68\%} \\
&  & \Large 2 & 54.67\% & 54.83\% & 53.35\% & \underline{58.33\%} & \textbf{70.67\%} &  \underline{55.77\%} & 51.84\% & 48.69\% & 55.45\% & \textbf{57.33\%} &  \underline{37.29\%} & 33.68\% & \textbf{42.57\%} & 32.07\% & \textbf{42.57\%} \\
&  & \Large 3 & 52.00\% & 51.59\% & \underline{55.82\%} & 50.87\% & \textbf{72.00\%} &  \textbf{56.67\%} & \underline{53.33\%} & 52.25\% & 51.35\% & \underline{53.33\%} &  43.67\% & \underline{44.21\%} & 39.83\% & 43.83\% & \textbf{48.77\%} \\
&  & \Large 4 & 56.67\% & 52.66\% & \underline{60.47\%} & 56.78\% & \textbf{67.33\%} &  \textbf{56.00\%} & 49.03\% & \underline{55.33\%} & 53.83\% & \underline{55.33\%} &  \underline{61.86\%} & 55.81\% & 52.84\% & \textbf{62.95\%} & 58.47\% \\
&  & \Large 5 & \underline{52.00\%} & 48.32\% & 51.57\% & 49.29\% & \textbf{58.00\%} &  \textbf{55.33\%} & 46.73\% & 45.33\% & 45.29\% & \underline{49.33\%} &  \underline{64.08\%} & 63.21\% & 60.02\% & 55.78\% & \textbf{65.22\%} \\
\cmidrule{2-18}
& \Large \multirow{5}{*}{WNLI}
& 1 & 34.74\% & 35.83\% & \underline{39.38\%} & 33.98\% & \textbf{49.47\%} &  29.47\% & \underline{47.04\%} & 27.22\% & 35.51\% & \textbf{50.53\%}& 37.99\% & 38.04\% & 37.42\% & \underline{39.88\%} & \textbf{39.89\%} \\
&  & \Large 2 & 48.42\% & \underline{50.46\%} & 44.89\% & 41.68\% & \textbf{51.58\%} &  \textbf{52.63\%} & 47.76\% & 42.16\% & 48.11\% & \underline{51.58\%} &  \underline{41.04\%} & 35.41\% & \textbf{42.28\%} & 36.85\% & \textbf{42.28\%} \\
&  & \Large 3 & \textbf{52.63\%} & 47.74\% & 46.95\% & 45.39\% & \underline{49.47\%} &  \underline{48.42\%} & \textbf{50.53\%} & 48.32\% & 44.06\% & \textbf{50.53\%} &  47.19\% & 46.75\% & 42.19\% & \textbf{48.85\%} & \underline{47.75\%} \\
&  & \Large 4 & \textbf{54.74\%} & 46.23\% & \underline{49.69\%} & 48.17\% & 47.37\% &  50.53\% & \underline{53.26\%} & 43.12\% & 51.43\% & \textbf{55.79\%} &  46.53\% & 52.66\% & 42.34\% & \underline{52.94\%} & \textbf{55.19\%} \\
&  & \Large 5 & 47.37\% & 48.87\% & \underline{50.62\%} & 47.32\% & \textbf{53.68\%} &  \textbf{53.68\%} & 41.62\% & 47.43\% & 48.22\% & \underline{52.63\%} &  \underline{56.32\%} & 52.91\% & 53.64\% & \underline{56.32\%} & \textbf{57.27\%} \\
\cmidrule{2-18}
& \Large \multirow{5}{*}{QNLI}
& \Large 1 & 25.33\% & \underline{28.67\%} & 28.02\% & 24.75\% & \textbf{50.00\%} &  29.33\% & 31.19\% & \underline{32.71\%} & 24.73\% & \textbf{46.67\%}& 26.16\% & 27.71\% & 22.70\% & \underline{29.54\%} & \textbf{30.28\%} \\
&  & \Large 2 & \underline{47.33\%} & 41.94\% & 46.16\% & 40.74\% & \textbf{54.00\%} &  \underline{45.33\%} & \textbf{47.33\%} & 36.92\% & 37.52\% & \textbf{47.33\%} &  28.53\% & \underline{30.47\%} & 26.85\% & 29.34\% & \textbf{34.93\%} \\
&  & \Large 3 & \underline{44.00\%} & 43.58\% & 43.57\% & 41.32\% & \textbf{56.67\%} &  \underline{49.33\%} & \textbf{50.00\%} & 39.01\% & 39.75\% & \textbf{50.00\%} &  37.63\% & 42.98\% & \underline{44.98\%} & \textbf{45.53\%} & 43.17\% \\
&  & \Large 4 & 44.67\% & \underline{48.22\%} & 42.33\% & 44.53\% & \textbf{55.33\%} &  \underline{46.00\%} & 44.37\% & 43.75\% & 45.05\% & \textbf{46.57\%} &  47.75\% & 42.60\% & 48.45\% & \textbf{54.21\%} & \underline{53.28\%} \\
&  & 5 & 45.33\% & \underline{47.45\%} & 45.17\% & 44.09\% & \textbf{54.67\%} &  \textbf{50.00\%} & 44.26\% & \underline{49.92\%} & 40.54\% & \textbf{50.00\%} &  \textbf{58.11\%} & 53.35\% & 50.23\% & 52.11\% & \underline{53.66\%} \\
\cmidrule{2-18}
& \Large \multirow{5}{*}{QQP}
& \Large 1 & 28.00\% & \underline{30.77\%} & 29.58\% & 30.03\% & \textbf{46.00\%} &  14.67\% & 18.19\% & \underline{35.16\%} & 20.05\% & \textbf{39.33\%}& 21.71\% & 22.57\% & \underline{35.18\%} & 32.73\% & \textbf{35.23\%} \\
&  & \Large 2 & 42.67\% & \underline{45.70\%} & 45.47\% & 45.41\% & \textbf{64.67\%} &  \textbf{57.33\%} & 45.88\% & 46.84\% & 46.87\% & \underline{50.00\%} &  25.49\% & \underline{39.22\%} & 34.11\% & 27.25\% & \textbf{40.75\%} \\
&  & \Large 3 & 42.00\% & 42.00\% & 43.69\% & \underline{45.02\%} & \textbf{63.33\%} &  \underline{57.33\%} & 50.88\% & 53.28\% & \textbf{61.33\%} & \textbf{61.33\%} &  35.03\% & 40.56\% & \underline{42.09\%} & \textbf{42.84\%} & 40.14\% \\
&  & \Large 4 & \underline{45.33\%} & 41.22\% & 44.10\% & 48.71\% & \textbf{54.67\%} &  48.00\% & 46.55\% & \underline{54.07\%} & 53.28\% & \textbf{56.00\%} &  34.07\% & \underline{44.39\%} & 37.87\% & \textbf{46.75\%} & 43.16\% \\
&  & \Large 5 & \underline{48.00\%} & 46.65\% & 43.21\% & 47.20\% & \textbf{62.67\%} &  \textbf{54.67\%} & 40.06\% & 48.39\% & 41.88\% & \underline{52.00\%} &  \textbf{45.50\%} & 41.24\% & 39.40\% & 44.23\% & \underline{45.11\%} \\
\cmidrule{2-18}
& \Large \multirow{5}{*}{MNLI}
& \Large 1 & 12.00\% & 9.50\% & \underline{16.05\%} & 12.24\% & \textbf{27.33\%} &  10.00\% & 18.76\% & 10.65\% & \underline{23.10\%} & \textbf{25.33\%}& \underline{17.07\%} & 16.49\% & 15.43\% & 16.14\% & \textbf{17.28\%} \\
&  & \Large 2 & 23.33\% & \underline{28.11\%} & 22.98\% & 24.06\% & \textbf{38.00\%} &  28.33\% & \textbf{31.33\%} & 24.74\% & \underline{30.37\%} & \textbf{31.33\%} &  \underline{18.93\%} & 18.18\% & 17.28\% & 18.58\% & \textbf{19.02\%} \\
&  & \Large 3 & 24.67\% & 22.72\% & \underline{27.40\%} & 27.06\% & \textbf{37.33\%} &  \textbf{28.67\%} & 26.42\% & 26.61\% & 25.91\% & \underline{27.33\%} &  18.81\% & 19.27\% & \underline{21.99\%} & \textbf{23.04\%} & 21.71\% \\
&  & \Large 4 & 26.00\% & \underline{29.47\%} & 23.77\% & 26.84\% & \textbf{44.00\%} &  25.33\% & \textbf{28.67\%} & 25.94\% & \underline{27.84\%} & \textbf{28.67\%} &  20.72\% & \underline{22.02\%} & 20.99\% & \textbf{22.68\%} & 21.45\% \\
&  & \Large 5 & \underline{25.33\%} & 24.85\% & 24.45\% & 25.08\% & \textbf{32.00\%} &  \underline{26.00\%} & 24.48\% & 23.81\% & 21.73\% & \textbf{31.33\%} &  \textbf{22.81\%} & 19.61\% & 19.79\% & 20.85\% & \underline{20.91\%} \\
\cmidrule{2-18}
& \Large \multirow{5}{*}{RTE}
& \Large 1 & \underline{43.33\%} & 41.18\% & 41.81\% & 39.18\% & \textbf{45.33\%} &  \textbf{51.33\%} & 44.95\% & 42.32\% & 42.75\% & \underline{45.33\%}& 28.26\% & 26.02\% & \underline{38.74\%} & 27.75\% & \textbf{43.74\%} \\
&  &\Large  2 & \underline{51.33\%} & 51.27\% & 48.68\% & 49.82\% & \textbf{52.00\%} &  46.00\% & 46.71\% & 40.32\% & \underline{46.77\%} & \textbf{49.33\%} &  31.51\% & \underline{44.71\%} & 37.51\% & 37.85\% & \textbf{46.23\%} \\
&  & \Large 3 & \textbf{45.33\%} & 43.18\% & 43.71\% & \underline{44.70\%} & \textbf{45.33\%} &  43.33\% & \underline{50.10\%} & 48.09\% & 46.65\% & \textbf{50.67\%} &  43.02\% & 36.72\% & \underline{45.84\%} & 41.64\% & \textbf{45.95\%} \\
&  & \Large 4 & \underline{43.44\%} & 43.08\% & 41.87\% & 41.85\% & \textbf{45.33\%} &  \textbf{54.00\%} & 50.68\% & 50.39\% & \underline{51.33\%} & \underline{51.33\%} &  42.38\% & 36.56\% & 44.81\% & \underline{52.93\%} & \textbf{54.92\%} \\
&  & 5 & \textbf{46.67\%} & 41.93\% & \underline{44.51\%} & 43.54\% & \textbf{46.67\%} &  \underline{49.33\%} & 44.99\% & 44.99\% & 48.30\% & \textbf{50.67\%} &  49.08\% & \underline{53.29\%} & 41.86\% & 50.05\% & \textbf{53.34\%} \\
\cmidrule{2-18}
& \Large \multirow{5}{*}{CoLA}
& \Large 1 & \textbf{58.00\%} & 39.91\% & 42.19\% & 38.72\% & \underline{44.67\%} &  53.33\% & 44.38\% & 54.11\% & \underline{55.44\%} & \textbf{58.00\%}& 49.96\% & 46.14\% & 42.95\% & \underline{53.19\%} & \textbf{61.98\%} \\
&  & \Large 2 & \underline{55.33\%} & 53.33\% & 50.34\% & 53.72\% & \textbf{56.00\%} &  54.00\% & 48.30\% & \underline{55.18\%} & 49.12\% & \textbf{59.33\%} &  \underline{56.19\%} & 51.31\% & \textbf{67.28\%} & 51.22\% & \textbf{67.28\%} \\
&  & \Large 3 & \textbf{58.00\%} & 54.36\% & \underline{55.46\%} & 45.26\% & \textbf{58.00\%} &  \textbf{58.00\%} & 53.99\% & 52.34\% & \underline{57.33\%} & \underline{57.33\%} &  65.00\% & 60.64\% & 59.93\% & \textbf{70.18\%} & \underline{67.63\%} \\
&  & \Large 4 & \textbf{53.33\%} & 38.14\% & 45.89\% & 41.50\% & \underline{45.33\%} &  \textbf{57.33\%} & 51.92\% & \underline{56.32\%} & 47.49\% & \textbf{57.33\%} &  66.18\% & 64.23\% & \textbf{75.52\%} & 57.06\% & \underline{71.55\%} \\
&  & \Large 5 & \textbf{56.67\%} & 44.91\% & 41.63\% & 46.69\% & \underline{47.33\%} &  \underline{58.00\%} & 48.88\% & 54.16\% & 46.44\% & \textbf{59.33\%} &  65.67\% & 66.52\% & 69.36\% & \underline{75.00\%} & \textbf{78.83\%} \\
\bottomrule
\end{tabular}
} 
\end{sc}
\end{large}
\end{center}
\vskip -0.1in
\end{table}

\begin{table}
\setlength{\abovecaptionskip}{-0.1cm}
\setlength{\belowcaptionskip}{0cm}
\caption{Generalization Experiments of Learnable Gate Architecture. (Continue)}
\label{tab:app_exp_gen2}
\vskip 0.15in
\renewcommand{\arraystretch}{1.2}
\begin{center}
\begin{large}
\begin{sc}
\resizebox{\textwidth}{!}{ 
\begin{tabular}{lcp{2cm}<{\centering}ccccccccccccccc} 
\toprule
{\fontsize{15}{16}\selectfont  \textbf{\textsc{\multirow{2}{*}[-1.2ex]{Benchmark}}}} &   {\fontsize{15}{16}\selectfont \multirow{2}{*}[-1.2ex]{\textbf{Task}}} & {\fontsize{14}{16}\selectfont \multirow{2}{*}[-1.2ex]{\textbf{TopK}}}& \multicolumn{5}{c}{\textbf{HydraLoRa}} & \multicolumn{5}{c}{\textbf{MoLE}} & \multicolumn{5}{c}{\textbf{OMoE}}  \\
\cmidrule(r){4-8}\cmidrule(r){9-13}\cmidrule(r){14-18}
& & & \textbf{\makecell{Stretch-\\Only Gate}} & \textbf{\makecell{Rotation-\\Only Gate}} & \textbf{Nexus} & \textbf{Tutel} & \textbf{\makecell{\textit{RadarGate}\\(ours)}} & \textbf{\makecell{Stretch-\\Only Gate}} & \textbf{\makecell{Rotation-\\Only Gate}} & \textbf{Nexus} & \textbf{Tutel} & \textbf{\makecell{\textit{RadarGate}\\(ours)}} & \textbf{\makecell{Stretch-\\Only Gate}} & \textbf{\makecell{Rotation-\\Only Gate}} & \textbf{Nexus} & \textbf{Tutel} & \textbf{\makecell{\textit{RadarGate}\\(ours)}} \\
\toprule
{\fontsize{15}{16}\selectfont\ \textbf{\multirow{35}{*}{MMLU}}}
& \Large \multirow{5}{*}{ARC-HARD}
& \Large 1 & 29.59\% & 24.47\% & \underline{30.33\%} & 27.37\% & \textbf{30.49\%} &  \underline{25.12\%} & 20.68\% & 21.52\% & 20.42\% & \textbf{26.39\%}& 25.33\% & \underline{33.31\%} & 31.65\% & 27.81\% & \textbf{33.91\%} \\
& & \Large 2 & 38.67\% & \underline{39.11\%} & 37.24\% & 39.04\% & \textbf{41.33\%} &  \underline{35.33\%} & 33.20\% & \textbf{40.00\%} & 33.68\% & \textbf{40.00\%} &  27.67\% & 28.28\% & 26.63\% & \underline{30.29\%} & \textbf{39.00\%} \\
& & \Large 3 & \underline{41.90\%} & 40.14\% & 37.37\% & 37.46\% & \textbf{43.30\%} &  33.02\% & \textbf{42.25\%} & 38.30\% & \underline{38.82\%} & \textbf{42.25\%} &  40.82\% & \underline{49.60\%} & \textbf{51.19\%} & 40.37\% & 49.29\% \\
& & \Large 4 & \textbf{44.77\%} & 36.34\% & 39.00\% & 37.31\% & \underline{43.97\%} &  31.21\% & 37.02\% & 37.57\% & \underline{40.82\%} & \textbf{47.04\%} &  \textbf{58.98\%} & 53.52\% & 52.62\% & \underline{54.53\%} & 54.40\% \\
& & \Large 5 & 52.05\% & 45.92\% & \underline{52.59\%} & 42.62\% & \textbf{53.08\%} &  38.39\% & \underline{44.46\%} & 31.43\% & 35.44\% & \textbf{47.82\%} &  \textbf{58.34\%} & 53.65\% & 51.02\% & 52.15\% & \underline{54.13\%} \\
\cmidrule{2-18}
& \Large \multirow{5}{*}{SCIENCE-MIDDLE}
& \Large 1 & 36.03\% & 30.23\% & 32.97\% & \underline{42.42\%} & \textbf{44.52\%} &  27.11\% & \underline{33.61\%} & 32.37\% & 26.60\% & \textbf{39.49\%}& 30.46\% & 30.17\% & \underline{35.11\%} & 34.88\% & \textbf{35.22\%} \\
& & \Large 2 & 55.56\% & 55.13\% & \underline{57.76\%} & \textbf{58.89\%} & \textbf{58.89\%} &  41.11\% & \underline{45.82\%} & 37.16\% & \textbf{50.00\%} & \textbf{50.00\%} &  32.67\% & \underline{37.91\%} & \textbf{40.00\%} & 33.15\% & \textbf{40.00\%} \\
& & \Large 3 & 53.92\% & 43.28\% & \underline{54.19\%} & 51.18\% & \textbf{54.20\%} &  40.93\% & \underline{43.65\%} & 36.24\% & 40.37\% & \textbf{45.07\%} &  \underline{35.28\%} & 31.12\% & 33.71\% & 33.41\% & \textbf{39.47\%} \\
& & \Large 4 & 58.59\% & \underline{59.50\%} & 53.26\% & \textbf{61.70\%} & \textbf{61.70\%} &  39.32\% & 32.47\% & \underline{42.13\%} & \textbf{50.69\%} & \textbf{50.69\%} &  40.87\% & \textbf{43.92\%} & 35.99\% & 41.49\% & \underline{43.89\%} \\
& & \Large 5 & \underline{56.52\%} & 48.24\% & 48.72\% & 53.33\% & \textbf{58.65\%} &  40.12\% & 35.30\% & 35.09\% & \underline{41.68\%} & \textbf{57.06\%} &  46.09\% & 42.18\% & \underline{46.53\%} & 40.29\% & \textbf{47.51\%} \\
\cmidrule{2-18}
& \Large \multirow{5}{*}{RACE}
& \Large 1 & 21.43\% & 19.48\% & 20.79\% & \underline{24.09\%} & \textbf{26.68\%} &  \underline{22.88\%} & 22.42\% & 19.27\% & 21.66\% & \textbf{24.45\%}& 23.71\% & \underline{35.27\%} & 27.74\% & 33.68\% & \textbf{43.93\%} \\
& & \Large 2 & 32.00\% & 38.10\% & 37.67\% & \underline{38.20\%} & \textbf{40.00\%} &  34.67\% & 32.32\% & \textbf{40.00\%} & \underline{34.92\%} & \textbf{40.00\%} &  26.33\% & 41.57\% & \underline{44.14\%} & 25.48\% & \textbf{45.67\%} \\
& & \Large 3 & 33.36\% & 31.85\% & \underline{33.81\%} & 28.47\% & \textbf{41.96\%} &  38.72\% & 37.16\% & \underline{40.39\%} & 37.84\% & \textbf{41.73\%} &  43.87\% & \underline{44.61\%} & 44.42\% & 41.49\% & \textbf{53.84\%} \\
& & \Large 4 & \textbf{51.61\%} & 46.87\% & 46.94\% & 44.97\% & \underline{47.05\%} &  \textbf{42.85\%} & 38.50\% & 36.40\% & \underline{40.11\%} & \underline{40.11\%} &  50.14\% & 42.94\% & \underline{51.92\%} & 51.05\% & \textbf{52.07\%} \\
& & \Large 5 & 47.86\% & 39.52\% & \underline{48.19\%} & 44.80\% & \textbf{48.22\%} &  \textbf{50.60\%} & 41.78\% & 43.07\% & 41.70\% & \underline{46.87\%} &  \textbf{50.74\%} & 49.70\% & 47.08\% & 43.21\% & \underline{49.85\%} \\
\cmidrule{2-18}
& \Large \multirow{5}{*}{OBQA}
& \Large 1 & 26.33\% & 24.37\% & 23.08\% & \underline{26.80\%} & \textbf{28.78\%} &  20.90\% & \underline{24.50\%} & 21.50\% & 17.33\% & \textbf{27.04\%}& 21.40\% & 18.54\% & 18.77\% & \underline{29.29\%} & \textbf{30.76\%} \\
& & \Large 2 & 34.67\% & 39.27\% & \textbf{41.33\%} & \underline{39.28\%} & \textbf{41.33\%} &  30.00\% & \underline{34.29\%} & 31.88\% & 33.65\% & \textbf{35.33\%} &  23.67\% & 27.06\% & \underline{30.13\%} & \textbf{34.33\%} & \textbf{34.33\%} \\
& & \Large 3 & 32.89\% & 34.23\% & 27.44\% & \underline{36.17\%} & \textbf{37.33\%} &  \underline{29.61\%} & 26.10\% & \textbf{37.30\%} & 26.44\% & \textbf{37.30\%} &  24.61\% & \textbf{35.41\%} & 29.38\% & 30.68\% & \underline{34.76\%} \\
& & \Large 4 & 37.81\% & \textbf{43.09\%} & 35.43\% & \underline{40.46\%} & \textbf{43.09\%} &  31.82\% & 32.62\% & 26.88\% & \underline{33.38\%} & \textbf{34.74\%} &  33.18\% & \textbf{35.24\%} & 32.23\% & \underline{35.14\%} & 34.47\% \\
& & \Large 5 & 42.48\% & \underline{47.21\%} & 46.83\% & 39.20\% & \textbf{50.50\%} &  \textbf{42.78\%} & 36.89\% & 34.56\% & 35.38\% & \underline{39.44\%} &  31.77\% & \underline{33.44\%} & 30.09\% & 28.69\% & \textbf{33.68\%} \\
\cmidrule{2-18}
& \Large \multirow{5}{*}{MC-TEST}
& \Large 1 & 34.98\% & 34.62\% & \underline{43.05\%} & 30.05\% & \textbf{50.93\%} &  36.06\% & 30.99\% & \underline{36.84\%} & 34.55\% & \textbf{38.32\%}& 36.32\% & 33.64\% & \underline{43.03\%} & 32.42\% & \textbf{45.37\%} \\
& & \Large 2 & 57.33\% & 58.34\% & 70.89\% & \underline{71.76\%} & \textbf{78.00\%} &  \underline{58.67\%} & \textbf{59.33\%} & 54.51\% & 58.66\% & \textbf{59.33\%} &  41.33\% & 47.01\% & \underline{49.85\%} & 47.79\% & \textbf{51.67\%} \\
& & \Large 3 & \underline{57.66\%} & 46.51\% & 54.19\% & 48.75\% & \textbf{72.48\%} &  \underline{53.90\%} & 53.68\% & 48.50\% & \textbf{53.99\%} & \textbf{53.99\%} &  48.22\% & 45.79\% & 58.40\% & \underline{60.35\%} & \textbf{64.33\%} \\
& & \Large 4 & \underline{60.00\%} & 59.59\% & 56.77\% & \textbf{68.58\%} & \textbf{68.58\%} &  \textbf{51.34\%} & 44.86\% & \underline{49.61\%} & 42.15\% & \underline{49.61\%} &  67.21\% & \textbf{69.91\%} & 66.76\% & 59.22\% & \underline{69.35\%} \\
& & \Large 5 & 60.00\% & 53.17\% & 58.61\% & \underline{63.41\%} & \textbf{67.78\%} &  \underline{60.00\%} & 51.57\% & 48.28\% & 54.24\% & \textbf{67.37\%} &  65.67\% & \underline{68.59\%} & 63.58\% & 63.23\% & \textbf{74.14\%} \\
\cmidrule{2-18}
& \Large \multirow{5}{*}{AUX-LAW-90S}
& \Large 1 & \underline{21.13\%} & 17.53\% & 20.20\% & 18.09\% & \textbf{24.84\%} &  \underline{20.77\%} & 17.57\% & 17.23\% & 20.07\% & \textbf{22.44\%}& 24.72\% & 22.17\% & 21.93\% & \underline{25.73\%} & \textbf{29.60\%} \\
& & \Large 2 & 26.00\% & \underline{31.19\%} & 31.01\% & 29.48\% & \textbf{33.33\%} &  32.67\% & 32.06\% & \underline{35.06\%} & 30.38\% & \textbf{36.67\%} &  29.00\% & 29.61\% & \textbf{34.67\%} & \underline{33.19\%} & \textbf{34.67\%} \\
& & \Large 3 & 26.16\% & \underline{28.68\%} & 28.12\% & 27.02\% & \textbf{31.21\%} &  32.33\% & \underline{34.36\%} & 30.94\% & 32.89\% & \textbf{35.92\%} &  27.90\% & \textbf{37.54\%} & 29.23\% & 25.65\% & \underline{36.39\%} \\
& & \Large 4 & \textbf{29.59\%} & 25.94\% & 24.34\% & 24.53\% & \underline{28.16\%} &  \underline{40.01\%} & 33.14\% & 38.38\% & 37.83\% & \textbf{40.32\%} &  \textbf{38.77\%} & 33.04\% & \underline{37.47\%} & 36.28\% & 35.84\% \\
& & \Large 5 & \textbf{30.80\%} & 26.12\% & 25.38\% & 30.12\% & \underline{30.62\%} &  45.20\% & \underline{45.27\%} & 41.17\% & 39.00\% & \textbf{47.36\%} &  \textbf{39.24\%} & 33.56\% & 35.08\% & 36.10\% & \underline{36.36\%} \\
\cmidrule{2-18}
& \Large \multirow{5}{*}{ARC-EASY}
& \Large 1 & \textbf{35.90\%} & \underline{34.05\%} & 29.32\% & 31.45\% & \textbf{35.90\%} &  \underline{27.90\%} & 23.91\% & 22.76\% & 25.03\% & \textbf{28.48\%}& \underline{25.19\%} & 21.92\% & 22.35\% & 23.17\% & \textbf{32.47\%} \\
& & \Large 2 & \textbf{56.67\%} & \underline{52.67\%} & 50.47\% & 48.28\% & \underline{52.67\%} &  36.67\% & \underline{40.46\%} & 37.14\% & 36.88\% & \textbf{42.00\%} &  27.67\% & 31.02\% & \underline{32.04\%} & 30.99\% & \textbf{35.00\%} \\
& & \Large 3 & \underline{55.07\%} & 54.27\% & 53.69\% & 49.65\% & \textbf{57.80\%} &  35.69\% & \underline{39.05\%} & \textbf{44.29\%} & 36.22\% & \textbf{44.29\%} &  37.89\% & 36.75\% & 38.64\% & \textbf{40.72\%} & \underline{40.33\%} \\
& & \Large 4 & \underline{57.98\%} & 57.50\% & 52.16\% & 52.21\% & \textbf{61.29\%} &  \textbf{47.44\%} & 39.98\% & 42.69\% & 39.52\% & \underline{43.84\%} &  40.56\% & 37.90\% & \textbf{51.06\%} & 41.84\% & \underline{48.20\%} \\
& & \Large 5 & \textbf{57.68\%} & 50.84\% & 54.98\% & 47.51\% & \underline{55.19\%} &  \underline{46.04\%} & 40.76\% & 44.85\% & 40.84\% & \textbf{46.71\%} &  \underline{48.81\%} & 41.63\% & 44.30\% & 47.13\% & \textbf{51.26\%} \\
\cmidrule{2-18}
& \Large \multirow{5}{*}{SCIENCE-ELEMENTARY}
& \Large 1 & 32.47\% & 32.23\% & \underline{35.95\%} & 26.75\% & \textbf{45.78\%} &  \textbf{33.05\%} & 32.53\% & 27.93\% & 29.67\% & \underline{32.78\%}& 39.35\% & 45.72\% & \underline{47.42\%} & 42.38\% & \textbf{47.77\%} \\
& & \Large 2 & 52.17\% & \underline{55.69\%} & 54.59\% & \textbf{60.87\%} & \textbf{60.87\%} &  \textbf{55.43\%} & \underline{52.17\%} & 50.09\% & 50.51\% & \underline{52.17\%} &  45.33\% & 47.50\% & \underline{49.07\%} & 43.61\% & \textbf{51.00\%} \\
& & \Large 3 & \underline{55.90\%} & 47.84\% & 53.52\% & 53.26\% & \textbf{65.22\%} &  \underline{54.65\%} & 49.79\% & \textbf{56.07\%} & 51.56\% & \textbf{56.07\%} &  52.74\% & 45.71\% & \textbf{58.58\%} & 50.23\% & \underline{54.95\%} \\
& & \Large 4 & 60.00\% & \textbf{63.76\%} & 61.79\% & \underline{61.96\%} & \textbf{63.76\%} &  \textbf{56.78\%} & 51.41\% & 49.95\% & \underline{55.97\%} & \underline{55.97\%} &  55.51\% & 47.69\% & 59.27\% & \underline{60.57\%} & \textbf{66.52\%} \\
& & \Large 5 & 55.19\% & 49.60\% & 61.95\% & \underline{62.57\%} & \textbf{63.33\%} &  \underline{60.00\%} & 51.17\% & 57.97\% & 55.43\% & \textbf{65.13\%} &  65.67\% & \underline{78.59\%} & 70.97\% & 58.03\% & \textbf{80.67\%} \\
\cmidrule{1-18}
{\fontsize{15}{16}\selectfont\ \textbf{\multirow{15}{*}{WMT14}}}
& \Large \multirow{5}{*}{EN-CS}
& \Large 1 & \underline{51.99\%} & 47.79\% & 42.59\% & 45.30\% & \textbf{52.37\%} &  50.00\% & \underline{55.14\%} & 43.03\% & 48.58\% & \textbf{56.40\%}& 54.97\% & 52.13\% & \underline{55.99\%} & 47.17\% & \textbf{56.98\%} \\
& & \Large 2 & 60.67\% & 61.84\% & \textbf{63.33\%} & \underline{63.04\%} & \textbf{63.33\%} &  59.33\% & 57.96\% & \underline{59.91\%} & 56.72\% & \textbf{62.67\%} &  58.33\% & 60.47\% & \underline{61.46\%} & \textbf{65.67\%} & \textbf{65.67\%} \\
& & \Large 3 & 59.04\% & 51.77\% & \underline{59.58\%} & 49.75\% & \textbf{59.76\%} &  55.61\% & 62.66\% & \underline{66.65\%} & 66.34\% & \textbf{68.33\%} &  65.85\% & \textbf{73.71\%} & \underline{71.61\%} & 69.40\% & 70.63\% \\
& & \Large 4 & \underline{69.65\%} & 57.58\% & 68.55\% & \textbf{71.06\%} & \textbf{71.06\%} &  70.67\% & 65.66\% & \underline{71.56\%} & \textbf{74.00\%} & \textbf{74.00\%} &  70.67\% & 75.08\% & 63.42\% & \textbf{79.99\%} & \underline{75.33\%} \\
& & \Large 5 & 68.54\% & 67.66\% & \underline{70.74\%} & 65.87\% & \textbf{75.00\%} &  \underline{66.20\%} & 55.59\% & 65.41\% & 62.87\% & \textbf{75.00\%} &  70.74\% & 73.88\% & 76.25\% & \underline{82.56\%} & \textbf{85.84\%} \\
\cmidrule{2-18}
& \Large \multirow{5}{*}{EN-RU}
& \Large 1 & \underline{53.24\%} & 51.50\% & 52.46\% & 44.81\% & \textbf{55.64\%} &  47.63\% & \underline{54.51\%} & 49.47\% & 40.89\% & \textbf{55.20\%}& \underline{57.20\%} & 51.76\% & 52.09\% & 51.69\% & \textbf{58.24\%} \\
& & \Large 2 & 57.33\% & \underline{65.31\%} & 57.37\% & 57.95\% & \textbf{65.67\%} &  56.67\% & \textbf{61.33\%} & 55.41\% & \underline{60.55\%} & \textbf{61.33\%} &  60.33\% & \underline{61.14\%} & 59.87\% & 56.73\% & \textbf{64.33\%} \\
& & \Large 3 & 56.34\% & \textbf{59.16\%} & 56.24\% & \underline{57.09\%} & \textbf{59.16\%} &  60.30\% & \textbf{63.46\%} & 55.68\% & \underline{61.61\%} & \textbf{63.46\%} &  59.22\% & \textbf{74.83\%} & 60.67\% & 68.18\% & \underline{70.47\%} \\
& & \Large 4 & \underline{67.39\%} & \textbf{75.00\%} & 57.21\% & 61.12\% & \textbf{75.00\%} &  \underline{70.67\%} & 66.10\% & 61.88\% & \textbf{74.00\%} & \textbf{74.00\%} &  70.67\% & 70.66\% & \textbf{76.84\%} & \underline{76.07\%} & 75.33\% \\
& & 5 & \underline{70.00\%} & 64.49\% & 62.24\% & 56.48\% & \textbf{75.00\%} &  \underline{68.81\%} & 59.22\% & 62.70\% & 61.13\% & \textbf{75.00\%} &  75.82\% & 75.04\% & \underline{81.85\%} & 73.46\% & \textbf{85.84\%} \\
\cmidrule{2-18}
& \Large \multirow{5}{*}{EN-DE}
& \Large 1 & \underline{58.22\%} & 56.86\% & 53.58\% & 48.45\% & \textbf{58.39\%} &  46.60\% & \underline{52.11\%} & 39.72\% & 42.74\% & \textbf{58.50\%}& \underline{51.35\%} & 50.18\% & 47.43\% & 46.93\% & \textbf{54.06\%} \\
& & \Large 2 & 61.00\% & 61.88\% & \textbf{68.33\%} & \underline{68.22\%} & \textbf{68.33\%} &  57.00\% & \underline{64.69\%} & \textbf{65.00\%} & 62.35\% & \textbf{65.00\%} &  59.67\% & \textbf{63.00\%} & \underline{60.79\%} & 55.97\% & \textbf{63.00\%} \\
& & \Large 3 & \underline{61.92\%} & 61.82\% & \textbf{70.00\%} & 56.23\% & \textbf{70.00\%} &  \underline{61.48\%} & 54.38\% & 54.16\% & 51.08\% & \textbf{63.72\%} &  61.19\% & \textbf{65.42\%} & \underline{64.73\%} & 58.35\% & 62.70\% \\
& & \Large 4 & \textbf{70.67\%} & 59.72\% & 64.15\% & 61.70\% & \underline{69.25\%} &  55.78\% & \textbf{73.46\%} & 57.19\% & \underline{71.75\%} & \textbf{73.46\%} &  69.40\% & \textbf{75.63\%} & 74.20\% & 74.59\% & \underline{75.14\%} \\
& & \Large 5 & \textbf{68.67\%} & 57.95\% & 55.48\% & 63.26\% & \underline{64.04\%} &  50.68\% & \underline{67.32\%} & 58.44\% & 67.13\% & \textbf{75.00\%} &  72.06\% & 67.49\% & \underline{73.91\%} & 72.63\% & \textbf{78.75\%} \\
\cmidrule{1-18}

{\fontsize{15}{16}\selectfont\ \textbf{\multirow{5}{*}{GPQA}}}
& \Large \multirow{5}{*}{GPQA}
& \Large 1 & 7.91\% & 8.91\% & 7.79\% & \underline{13.22\%} & \textbf{13.70\%} &  11.38\% & 9.12\% & \underline{12.44\%} & 9.35\% & \textbf{16.12\%} & 6.90\% & 11.29\% & \underline{13.59\%} & 10.95\% & \textbf{15.08\%} \\
& & \Large 2 & 9.67\% & 11.10\% & \underline{13.10\%} & 12.15\% & \textbf{15.33\%} &  11.00\% & 15.62\% & \underline{16.20\%} & 14.77\% & \textbf{18.67\%} & 7.67\% & \underline{13.34\%} & 7.57\% & 8.94\% & \textbf{16.33\%} \\
& & \Large 3 & 13.21\% & 12.09\% & \textbf{14.57\%} & \underline{14.39\%} & \textbf{14.57\%} &  \underline{17.12\%} & 16.64\% & 14.05\% & 13.85\% & \textbf{17.69\%} & 11.08\% & 13.95\% & 15.60\% & \textbf{20.08\%} & \underline{18.35\%} \\
& & \Large 4 & 12.02\% & \underline{14.25\%} & \textbf{15.90\%} & 11.39\% & \textbf{15.90\%} &  \underline{19.68\%} & 18.64\% & 15.92\% & \textbf{20.37\%} & \textbf{20.37\%} & 12.10\% & \textbf{20.88\%} & 16.43\% & 15.65\% & \underline{19.61\%} \\
& & \Large 5 & 17.30\% & \underline{18.46\%} & 18.20\% & 15.10\% & \textbf{18.54\%} &  \textbf{19.35\%} & 17.27\% & 18.20\% & 15.99\% & \underline{18.47\%} & 20.97\% & 21.70\% & 21.26\% & \underline{21.98\%} & \textbf{22.55\%} \\

\cmidrule{1-18}
{\fontsize{15}{16}\selectfont\ \textbf{\multirow{5}{*}{MATH}}}
& \Large \multirow{5}{*}{MATH}
& \Large 1 & 5.95\% & \underline{8.76\%} & 4.98\% & 8.50\% & \textbf{11.35\%} &  6.29\% & 5.48\% & \underline{7.81\%} & 6.07\% & \textbf{9.49\%}& 6.15\% & \underline{9.82\%} & 7.37\% & 5.53\% & \textbf{10.00\%} \\
& & \Large 2 & 7.33\% & \underline{10.61\%} & 7.82\% & 8.80\% & \textbf{12.67\%} &  6.33\% & 6.56\% & \underline{10.02\%} & \textbf{11.67\%} & \textbf{11.67\%} &  7.00\% & 6.36\% & \underline{8.63\%} & \textbf{11.00\%} & \textbf{11.00\%} \\
& & \Large 3 & 11.52\% & 13.15\% & \underline{13.54\%} & \textbf{14.74\%} & \textbf{14.74\%} &  10.31\% & 8.58\% & \underline{11.07\%} & 9.43\% & \textbf{11.37\%} &  9.77\% & \underline{10.97\%} & 9.33\% & 9.22\% & \textbf{11.88\%} \\
& & \Large 4 & \underline{15.23\%} & 14.77\% & 13.07\% & 12.73\% & \textbf{15.32\%} &  \underline{13.43\%} & 11.71\% & 11.52\% & \textbf{14.84\%} & \textbf{14.84\%} &  12.39\% & 11.53\% & \underline{12.97\%} & 12.71\% & \textbf{13.14\%} \\
& & 5 & 15.87\% & 16.69\% & \underline{16.74\%} & 15.27\% & \textbf{17.05\%} &  \textbf{15.49\%} & 13.49\% & 14.47\% & 13.00\% & \underline{14.83\%} &  \textbf{13.17\%} & 12.46\% & 11.55\% & 12.63\% & \underline{12.92\%} \\
\cmidrule{1-18}
{\fontsize{15}{16}\selectfont\ \textbf{\multirow{5}{*}{GSM8K}}}
& \Large \multirow{5}{*}{GSM8K}
& \Large 1 & 21.99\% & 20.17\% & 23.26\% & \underline{24.59\%} & \textbf{25.74\%} &  \underline{18.78\%} & 16.26\% & 15.37\% & 15.27\% & \textbf{25.49\%}& 18.73\% & \underline{24.38\%} & 23.13\% & 16.06\% & \textbf{26.92\%} \\
& &\Large 2 & 23.33\% & 28.02\% & \underline{29.47\%} & 26.08\% & \textbf{31.00\%} &  21.58\% & 22.55\% & 21.40\% & \underline{26.92\%} & \textbf{28.47\%} &  22.00\% & \textbf{32.00\%} & \underline{29.72\%} & 29.67\% & \textbf{32.00\%} \\
& & \Large 3 & 25.49\% & \underline{26.11\%} & \textbf{28.28\%} & 24.02\% & \textbf{28.28\%} &  25.44\% & 27.04\% & 22.38\% & \underline{29.57\%} & \textbf{30.01\%} &  24.45\% & 22.81\% & 19.83\% & \underline{25.23\%} & \textbf{31.31\%} \\
& & \Large 4 & 31.94\% & 27.87\% & 31.14\% & \underline{31.97\%} & \textbf{33.32\%} &  26.22\% & 27.03\% & \textbf{36.67\%} & \underline{32.28\%} & \textbf{36.67\%} &  \underline{24.62\%} & 20.16\% & \textbf{28.32\%} & 23.81\% & \textbf{28.32\%} \\
& & \Large 5 & 31.34\% & 30.36\% & 33.36\% & \underline{37.10\%} & \textbf{40.00\%} &  \textbf{35.00\%} & 33.31\% & 29.11\% & 29.23\% & \underline{34.96\%} &  \textbf{27.87\%} & 23.06\% & 25.08\% & 23.56\% & \underline{25.76\%} \\
\bottomrule
\end{tabular}
} 
\end{sc}
\end{large}
\end{center}
\vskip -0.1in
\end{table}

\begin{table}
\setlength{\abovecaptionskip}{-0.1cm}
\setlength{\belowcaptionskip}{0cm}
    \caption{Generalization Experiments of Rule-based Gate Architecture.}
    \label{tab:app_exp_gen_rb1}
    \vskip 0.15in
    \renewcommand{\arraystretch}{1.2}
    \begin{center}
    \begin{small}
    \begin{sc}
    \fontsize{10pt}{12pt}\selectfont 
    \resizebox{0.5\textwidth}{!}{ 
    \begin{tabular}{l|c|ccc} 
    \hline
    \textbf{Benchmark} & \textbf{Task} & \textbf{Lorahub}&\textbf{Arrow} & \textbf{PEMs} \\
    \hline
    \multirow{7}{*}{\textbf{glue}} 
        & SST-2 & 16.67\% & \underline{18.33\%} & \textbf{20.24\%} \\
         & WNLI & \textbf{22.62\%} & \underline{20.40\%} & 15.34\% \\
         & QNLI & \underline{17.11\%} & \textbf{19.42\%} & 14.95\% \\
         & QQP & \textbf{21.92\%} & \underline{17.21\%} & 16.02\% \\
         & MNLI & \underline{10.59\%} & 9.61\% & \textbf{11.85\%} \\
         & RTE & 18.67\% & \underline{21.33\%} & \textbf{24.53\%} \\
         & CoLA & \underline{19.67\%} & \textbf{22.33\%} & 16.57\% \\
    \hline
    \end{tabular}
    } 
    \end{sc}
    \end{small}
    \end{center}
    \vskip -0.1in
    \end{table}
\begin{table}
\setlength{\abovecaptionskip}{-0.1cm}
\setlength{\belowcaptionskip}{0cm}
    \caption{Generalization Experiments of Rule-based Gate Architecture. (Continue)}
    \label{tab:app_exp_gen_rb2}
    \vskip 0.15in
    \renewcommand{\arraystretch}{1.2}
    \begin{center}
    \begin{small}
    \begin{sc}
    \fontsize{10pt}{12pt}\selectfont 
    \resizebox{0.7\textwidth}{!}{ 
    \begin{tabular}{l|c|ccc} 
    \hline
    \textbf{Benchmark} & \textbf{Task} & \textbf{Lorahub}&\textbf{Arrow} & \textbf{PEMs} \\
    \hline
    \multirow{8}{*}{\textbf{mmlu}} 
        & ARC-HARD & 12.33\% & \underline{19.67\%} & \textbf{21.10\%} \\
         & SCIENCE\_MIDDLE & \underline{21.33\%} & 19.33\% & \textbf{25.00\%} \\
         & RACE & \underline{13.67\%} & 13.00\% & \textbf{15.55\%} \\
         & OBQA& \textbf{14.67}\% & 10.21\% & \underline{12.02\%} \\
         & MC-TEST & \textbf{25.33\%} & \underline{21.67\%} & 20.00\% \\
         & AUX-LAW-90s & 11.67\% & \textbf{13.00\%} & 12.33\% \\
         & ARC-EASY & 11.67\% & \underline{12.33\%} & \textbf{16.17\%} \\
         & SCIENCE-ELEMENTARY & \underline{15.33\%} & \textbf{16.00\%} & 12.67\% \\
    \cline{1-5}
    \multirow{3}{*}{\textbf{wmt14}}
        & EN-CS & 35.33\% & \underline{39.11\%} & \textbf{41.67\%} \\
         & EN-RU & \underline{41.67\%} & \textbf{40.00\%} & 35.00\% \\
         & EN-DE & \underline{35.00\%} & 34.47\% & \textbf{39.55\%} \\
    \cline{1-5}
    \multirow{1}{*}{\textbf{gpqa}}
        & GPQA & 3.58\% & \underline{4.39\%} & \textbf{5.00\%} \\
    \cline{1-5}
    \multirow{1}{*}{\textbf{math}}
        & MATH & 3.35\% & \underline{4.11\%} & \textbf{4.89\%} \\
    \cline{1-5}
    \multirow{1}{*}{\textbf{gsm8k}}
        & GSM8K & \textbf{18.67\%} & \underline{15.11\%} & 14.45\% \\
    \hline
    \end{tabular}
    } 
    \end{sc}
    \end{small}
    \end{center}
    \vskip -0.1in
    \end{table}

\section{Scaling Experiments}
\label{app:exp_sca} 

This section presents the results of two key experiments in the Scaling study: module number scaling and model parameter scaling. Figures \ref{fig:app_exp_module_sca_hyd}, \ref{fig:app_exp_module_sca_mole}, and \ref{fig:app_exp_module_sca_omoe} illustrate the module-related experiments conducted within three different architectures, while Figures \ref{fig:app_exp_module_sca_hyd}, \ref{fig:app_exp_module_sca_mole}, and \ref{fig:app_exp_module_sca_omoe} depict the model-related experiments across the same architectures.

\subsection{Module Scaling}

Figures \ref{fig:app_exp_module_sca_hyd}, \ref{fig:app_exp_module_sca_mole}, and \ref{fig:app_exp_module_sca_omoe} illustrate the performance variations of different gating structures within the three learnable gating architectures—HydraLoRA, MoLE, and OMoE—as the number of modules increases from 5 to 40. Each figure represents performance on a specific benchmark, where accuracy is computed as the mean across all tasks within the benchmark. It can be observed that our approach exhibits a consistently increasing trend in most cases, whereas baseline methods generally show a slight fluctuation followed by a downward trend. Moreover, our method achieves the best performance across all module configurations. Based on these results, it is further demonstrated that \textit{RadarGate} is more effective in large-scale modular settings and exhibits superior module scalability.
\begin{figure}[htpb]
\vskip 0.2in
\begin{center}
\begin{minipage}{\textwidth}
    \centering
    
    \begin{subfigure}[t]{0.32\textwidth}
        \includegraphics[width=\textwidth]{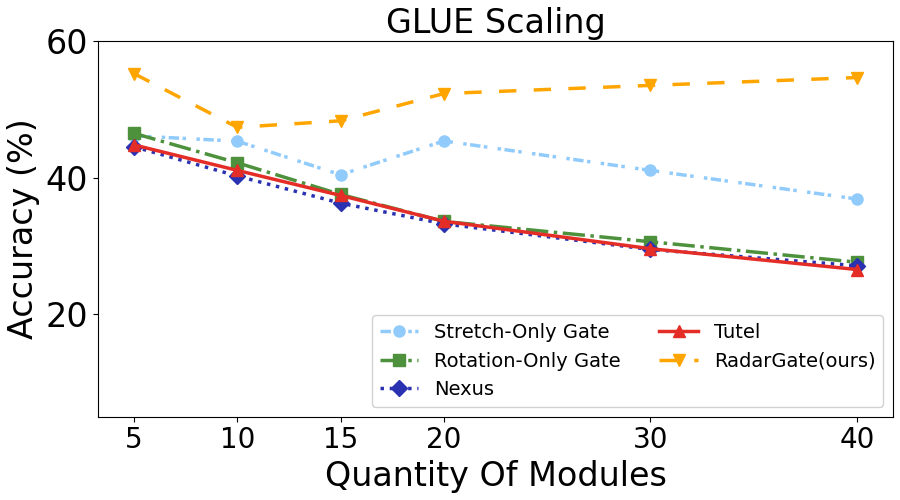}
        \caption{}
        \label{fig:app_exp_module_sca_hyd_subfig11}
    \end{subfigure}
    \hfill
    \begin{subfigure}[t]{0.32\textwidth}
        \includegraphics[width=\textwidth]{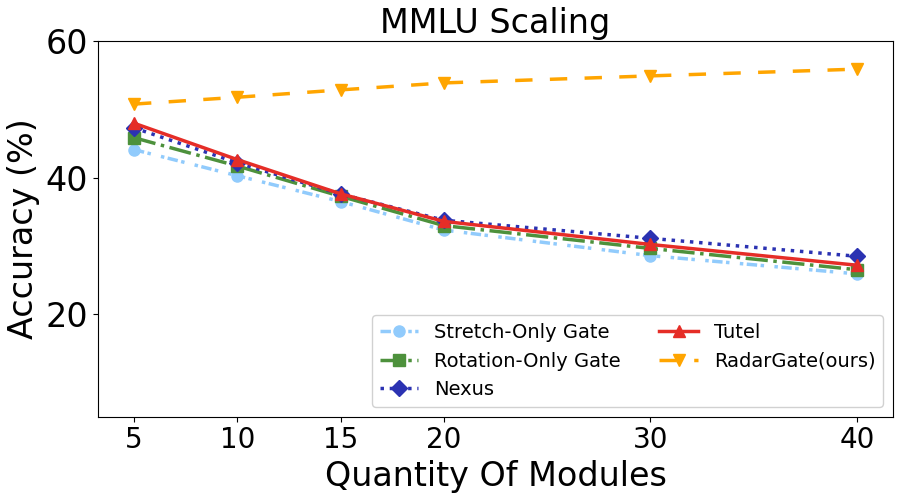}
        \caption{}
        \label{fig:app_exp_module_sca_hyd_subfig12}
    \end{subfigure}
    \hfill
    \begin{subfigure}[t]{0.32\textwidth}
        \includegraphics[width=\textwidth]{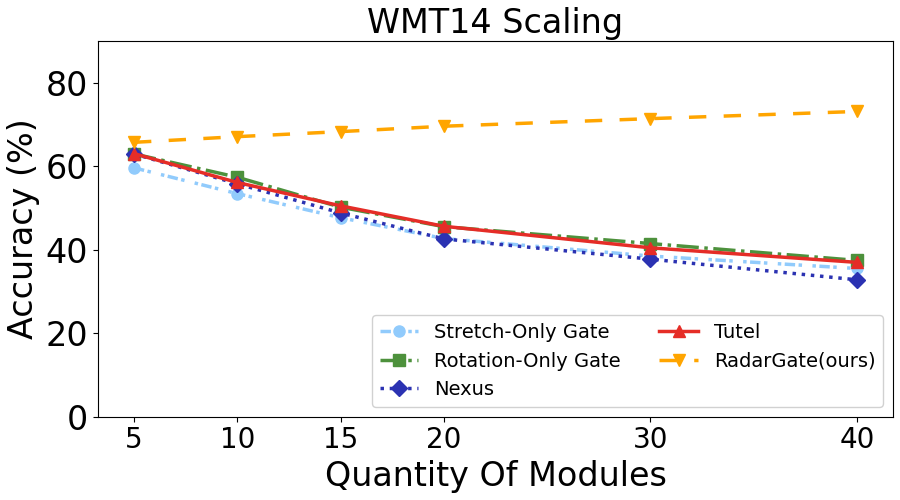}
        \caption{}
        \label{fig:app_exp_module_sca_hyd_subfig13}
    \end{subfigure}
    \hfill
    \begin{subfigure}[t]{0.32\textwidth}
        \includegraphics[width=\textwidth]{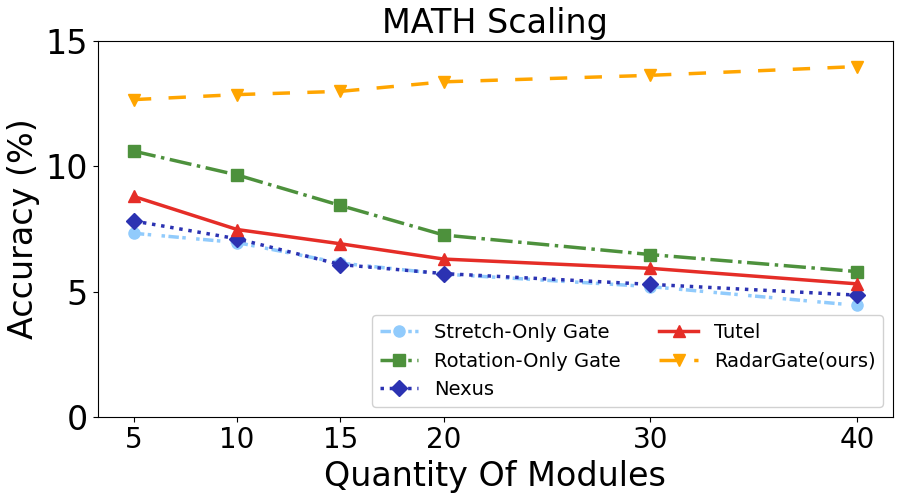}
        \caption{}
        \label{fig:app_exp_module_sca_hyd_subfig14}
    \end{subfigure}
    \hfill
    \begin{subfigure}[t]{0.32\textwidth}
        \includegraphics[width=\textwidth]{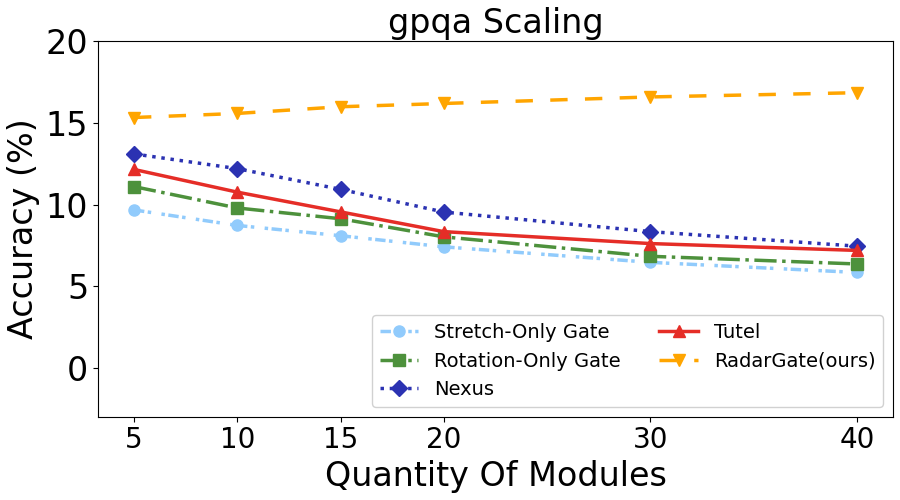}
        \caption{}
        \label{fig:app_exp_module_sca_hyd_subfig15}
    \end{subfigure}
    \hfill
    \begin{subfigure}[t]{0.32\textwidth}
        \includegraphics[width=\textwidth]{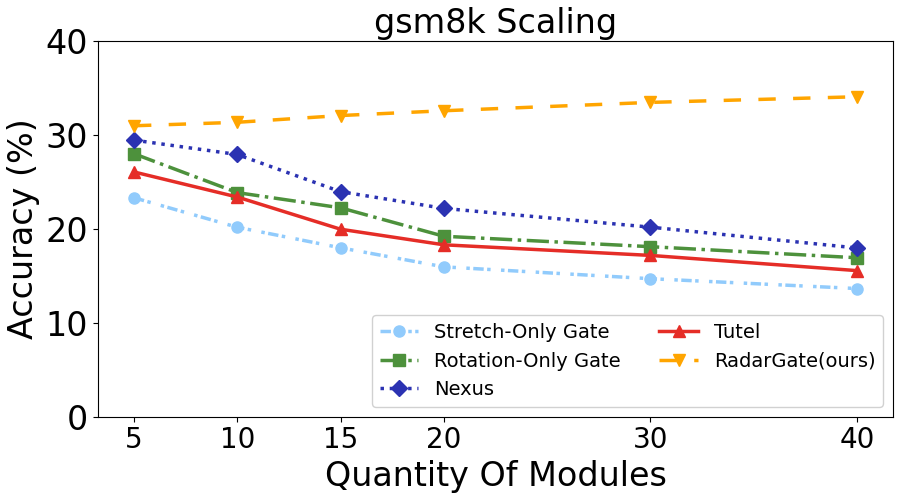}
        \caption{}
        \label{fig:app_exp_module_sca_hyd_subfig16}
    \end{subfigure}
\end{minipage}

\caption{The performance variations of different gates within the HydraLoRA architecture across six different benchmarks as the number of modules increases.} 
\label{fig:app_exp_module_sca_hyd} 
\end{center}
\vskip -0.2in
\end{figure}

\begin{figure}
\vskip 0.2in
\begin{center}
\begin{minipage}{\textwidth}
    \centering
    
    \begin{subfigure}[t]{0.32\textwidth}
        \includegraphics[width=\textwidth]{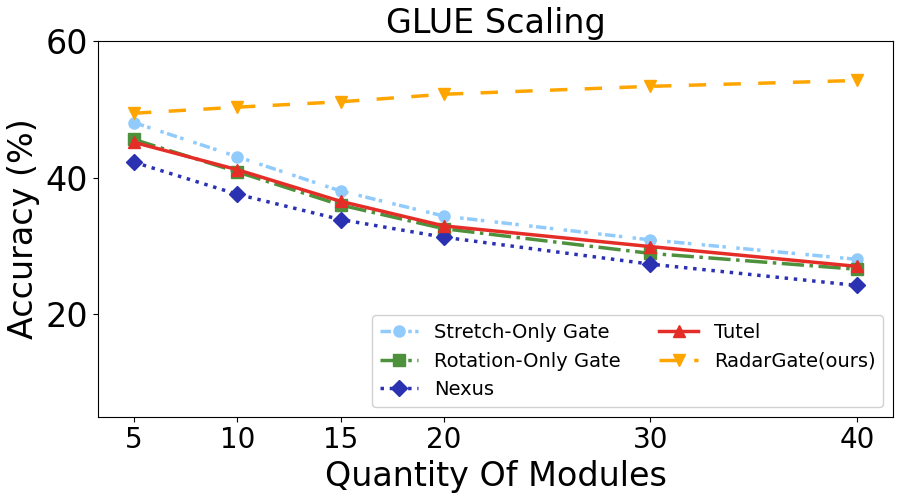}
        \caption{}
        \label{fig:app_exp_module_sca_mole_subfig11}
    \end{subfigure}
    \hfill
    \begin{subfigure}[t]{0.32\textwidth}
        \includegraphics[width=\textwidth]{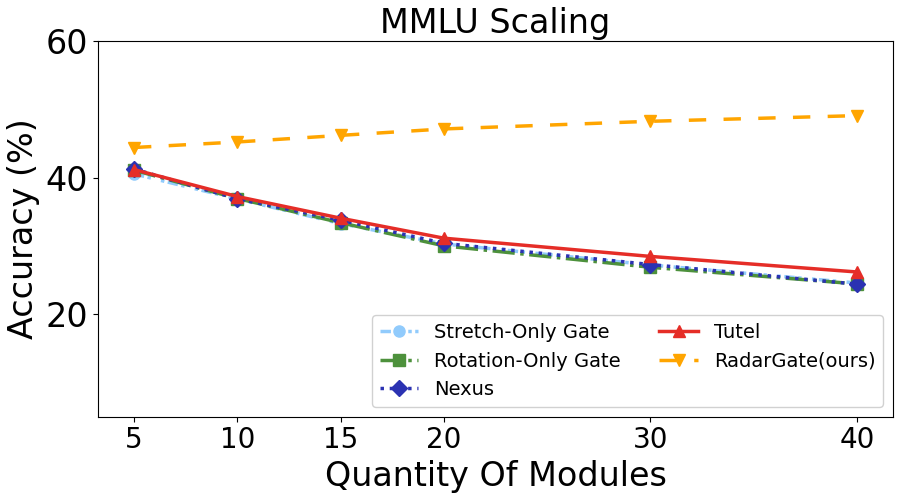}
        \caption{}
        \label{fig:app_exp_module_sca_mole_subfig12}
    \end{subfigure}
    \hfill
    \begin{subfigure}[t]{0.32\textwidth}
        \includegraphics[width=\textwidth]{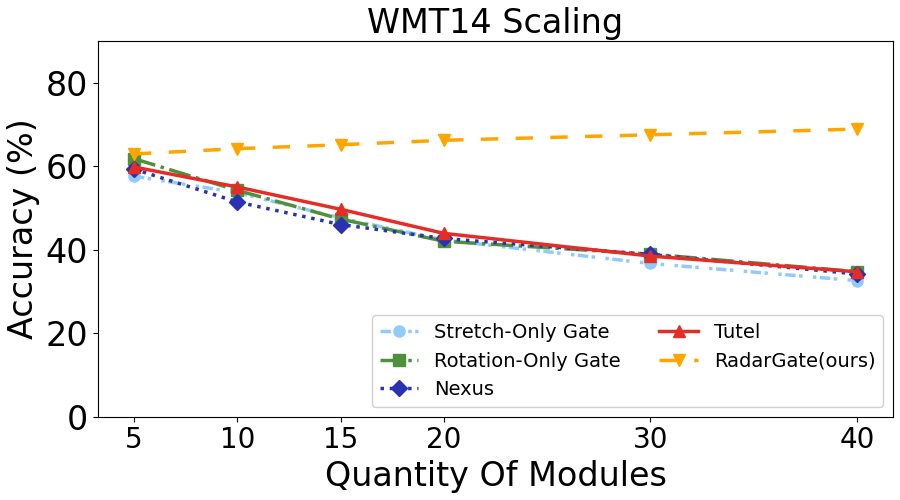}
        \caption{}
        \label{fig:app_exp_module_sca_mole_subfig13}
    \end{subfigure}
    \hfill
    \begin{subfigure}[t]{0.32\textwidth}
        \includegraphics[width=\textwidth]{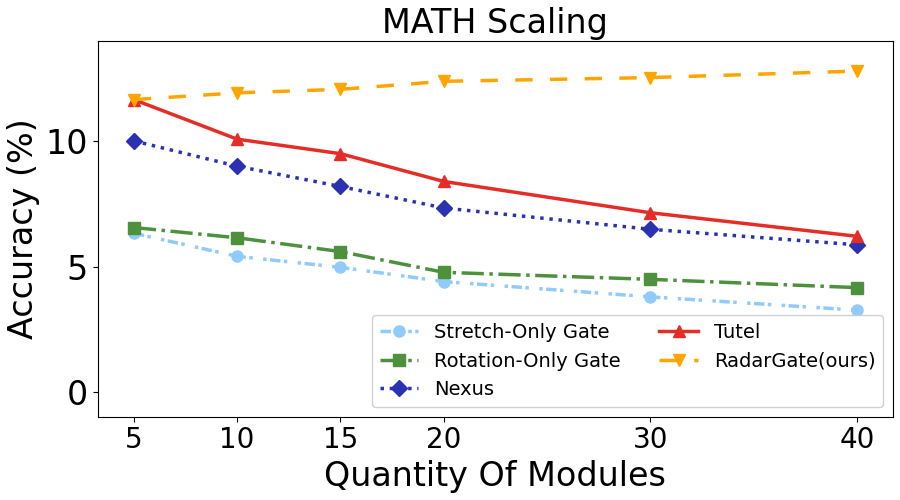}
        \caption{}
        \label{fig:app_exp_module_sca_mole_subfig14}
    \end{subfigure}
    \hfill
    \begin{subfigure}[t]{0.32\textwidth}
        \includegraphics[width=\textwidth]{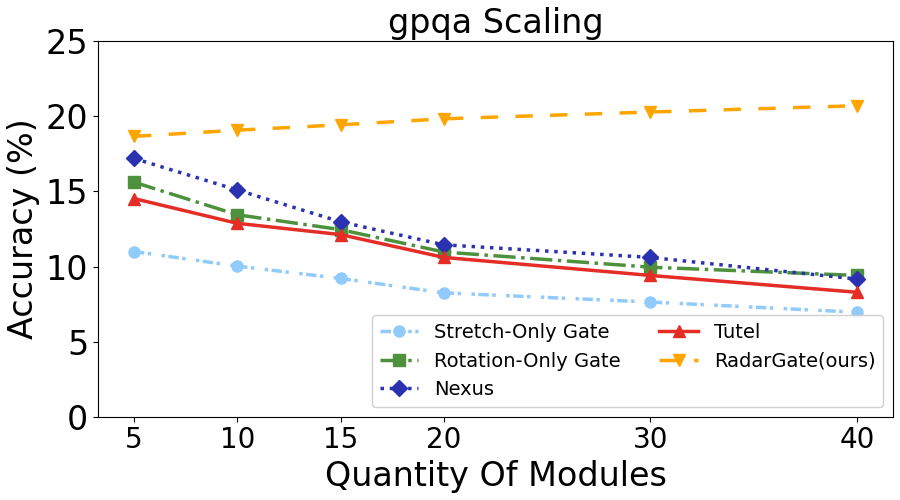}
        \caption{}
        \label{fig:app_exp_module_sca_mole_subfig15}
    \end{subfigure}
    \hfill
    \begin{subfigure}[t]{0.32\textwidth}
        \includegraphics[width=\textwidth]{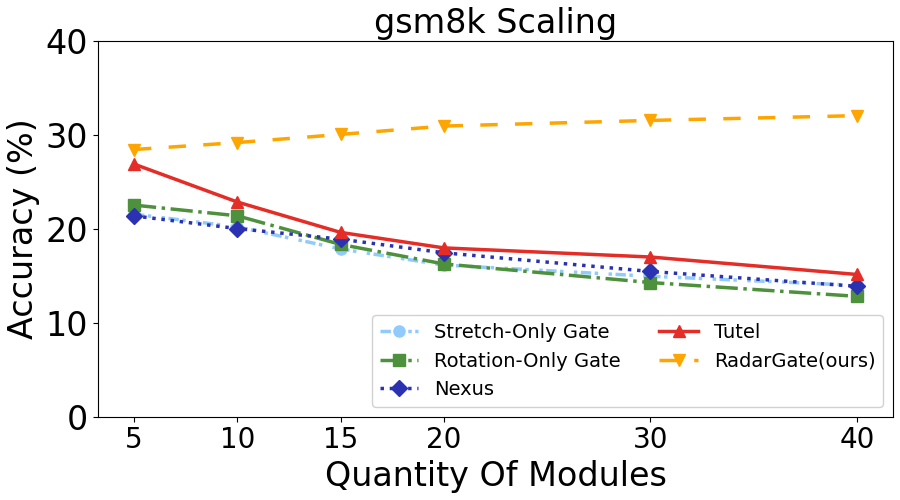}
        \caption{}
        \label{fig:app_exp_module_sca_mole_subfig16}
    \end{subfigure}
\end{minipage}

\caption{The performance variations of different gates within the MoLE architecture across six different benchmarks as the number of modules increases.} 
\label{fig:app_exp_module_sca_mole} 
\end{center}
\vskip -0.2in
\end{figure}

\begin{figure}
\vskip 0.2in
\begin{center}
\begin{minipage}{\textwidth}
    \centering
    
    \begin{subfigure}[t]{0.32\textwidth}
        \includegraphics[width=\textwidth]{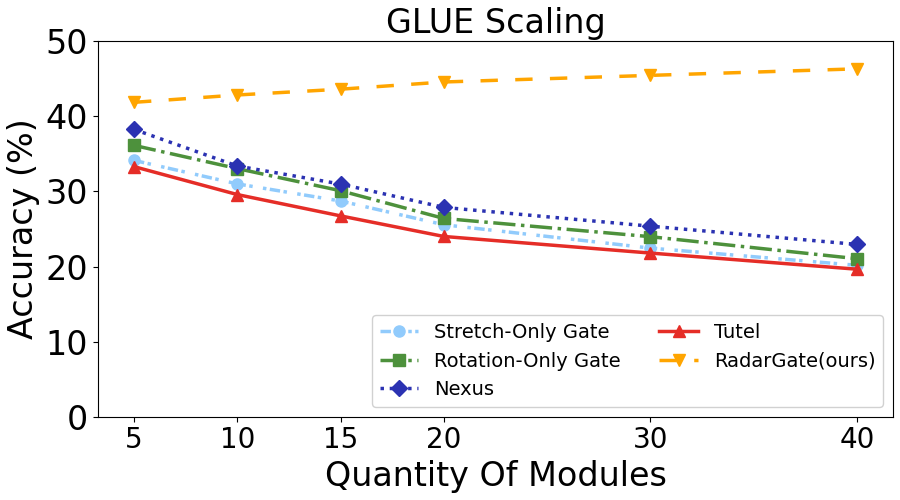}
        \caption{}
        \label{fig:app_exp_module_sca_omoe_subfig11}
    \end{subfigure}
    \hfill
    \begin{subfigure}[t]{0.32\textwidth}
        \includegraphics[width=\textwidth]{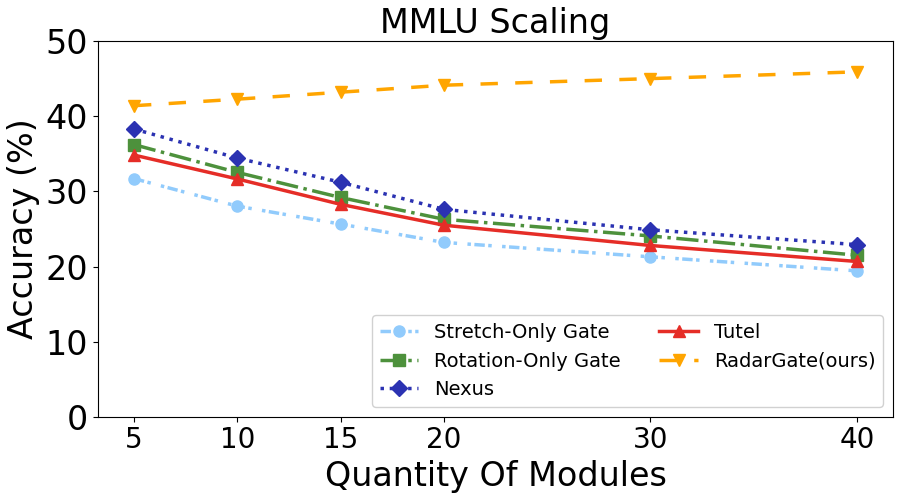}
        \caption{}
        \label{fig:app_exp_module_sca_omoe_subfig12}
    \end{subfigure}
    \hfill
    \begin{subfigure}[t]{0.32\textwidth}
        \includegraphics[width=\textwidth]{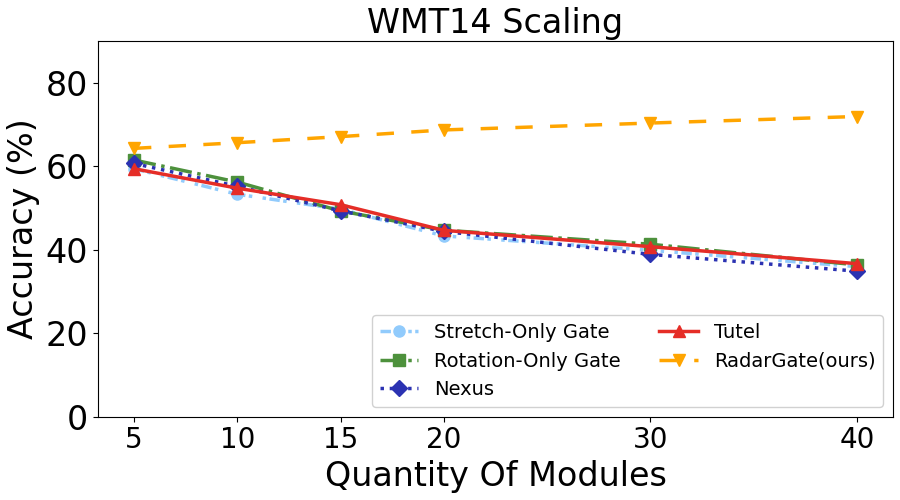}
        \caption{}
        \label{fig:app_exp_module_sca_omoe_subfig13}
    \end{subfigure}
    \hfill
    \begin{subfigure}[t]{0.32\textwidth}
        \includegraphics[width=\textwidth]{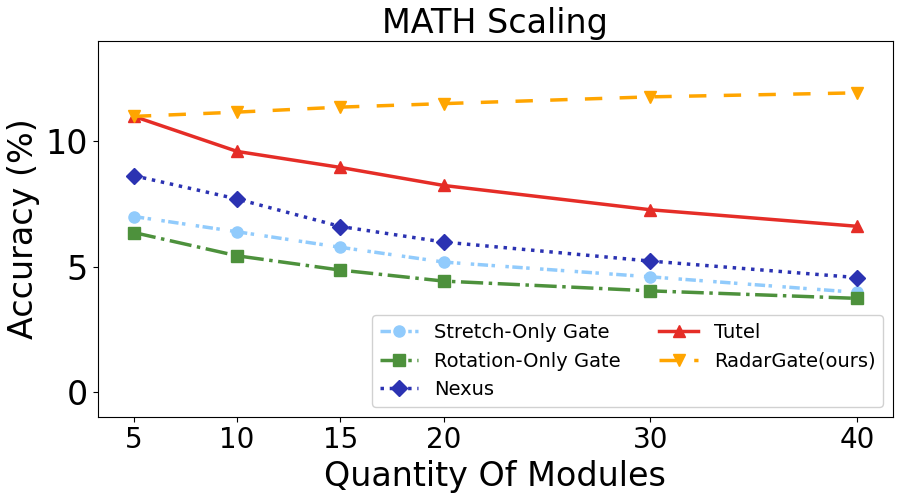}
        \caption{}
        \label{fig:app_exp_module_sca_omoe_subfig14}
    \end{subfigure}
    \hfill
    \begin{subfigure}[t]{0.32\textwidth}
        \includegraphics[width=\textwidth]{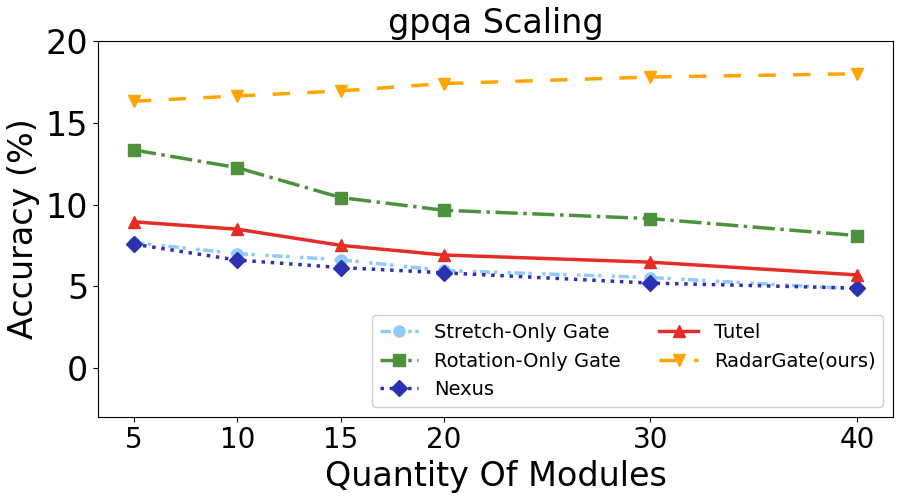}
        \caption{}
        \label{fig:app_exp_module_sca_omoe_subfig15}
    \end{subfigure}
    \hfill
    \begin{subfigure}[t]{0.32\textwidth}
        \includegraphics[width=\textwidth]{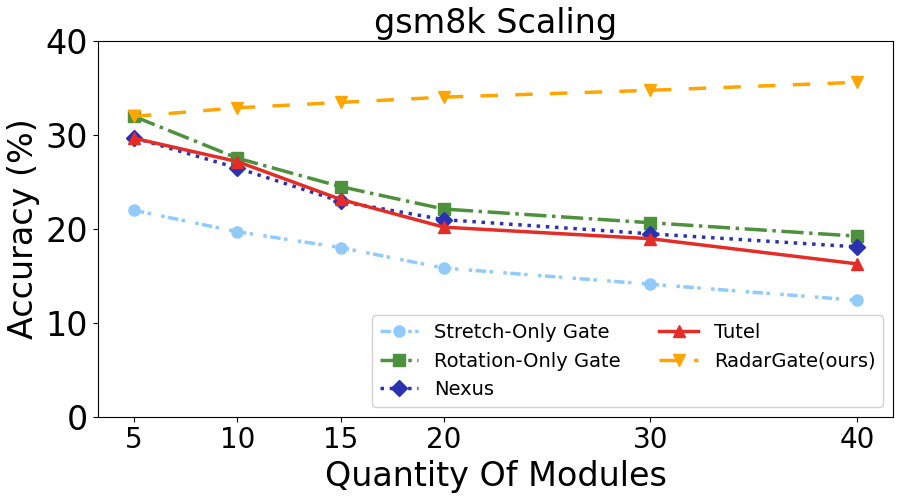}
        \caption{}
        \label{fig:app_exp_module_sca_omoe_subfig16}
    \end{subfigure}
\end{minipage}

\caption{The performance variations of different gates within the OMoE architecture across six different benchmarks as the number of modules increases.} 
\label{fig:app_exp_module_sca_omoe} 
\end{center}
\vskip -0.2in
\end{figure}

\subsection{Model Scaling}

Figure \ref{fig:app_exp_model_sca_hyd}, \ref{fig:app_exp_model_sca_mole}, and \ref{fig:app_exp_model_sca_omoe} illustrate the performance variations of different gating structures within the three learnable gating architectures—HydraLoRA, MoLE, and OMoE—as the model size increases. The five selected models have parameter sizes of 110M for IndicBART, 580M for mt0-base, 770M for Flan-T5-large, 1B for LLaMA3.2-1B, and 3B for LLaMA3.2-3B. Each figure represents performance on a specific benchmark, where accuracy is computed as the mean across all tasks within the benchmark. It can be observed that overall accuracy improves as the model size increases, and \textit{RadarGate} consistently outperforms baseline methods across nearly all cases. Furthermore, our approach achieves the best performance across different model sizes. Based on these results, it is further demonstrated that \textit{RadarGate} adapts more effectively to models of varying parameter scales and exhibits superior scalability in terms of model size.

\section{Convergence Experiments}
\label{app:exp_con}

This section presents the convergence behavior of the loss curves when all gating structures are applied to the MoLE architecture with module numbers set to 5, 15, and 40, as shown in Figure \ref{fig:app_conv_pic}. Overall, existing methods exhibit unstable convergence across all cases, with particularly severe oscillations in the early training stages and suboptimal convergence in the later stages. As the number of modules increases from 5 to 15 to 40, these issues become even more pronounced, making convergence increasingly difficult. In contrast, \textit{RadarGate} demonstrates significantly faster and more stable convergence, achieving the best final convergence performance.

\section{Sample Size Experiments}

\label{app:exp_sam}

This section presents a performance comparison between existing gating methods and our \textit{RadarGate} within the MoLE architecture across six widely used benchmarks under varying sample sizes, with a particular focus on low-data scenarios. As shown in Figures \ref{fig:app_exp_sample_sca_hyd}, \ref{fig:app_exp_sample_sca_mole}, and \ref{fig:app_exp_sample_sca_omoe}, our method consistently achieves the highest average performance across different sample sizes. Notably, when the sample size is 50 or 100, \textit{RadarGate} outperforms all baselines in over 90\% of cases. These results demonstrate that \textit{RadarGate} effectively adapts to various training sample sizes and is particularly well-suited for data-constrained scenarios.

\begin{figure}[htpb]
\vskip 0.2in
\begin{center}
\begin{minipage}{\textwidth}
    \centering
    
    \begin{subfigure}[t]{0.32\textwidth}
        \includegraphics[width=\textwidth]{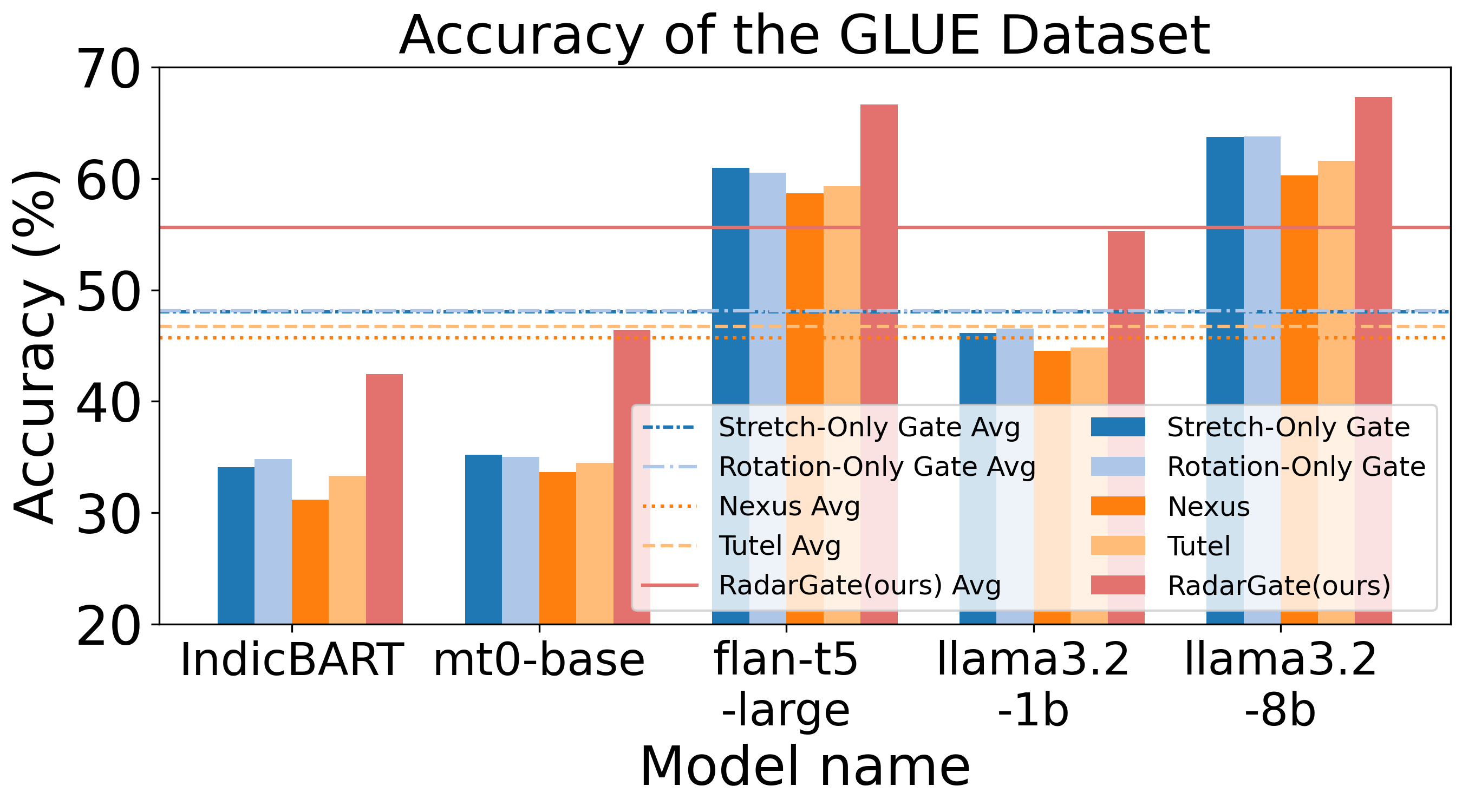}
        \caption{}
        \label{fig:app_exp_module_sca_hyd_subfig1}
    \end{subfigure}
    \hfill
    \begin{subfigure}[t]{0.32\textwidth}
        \includegraphics[width=\textwidth]{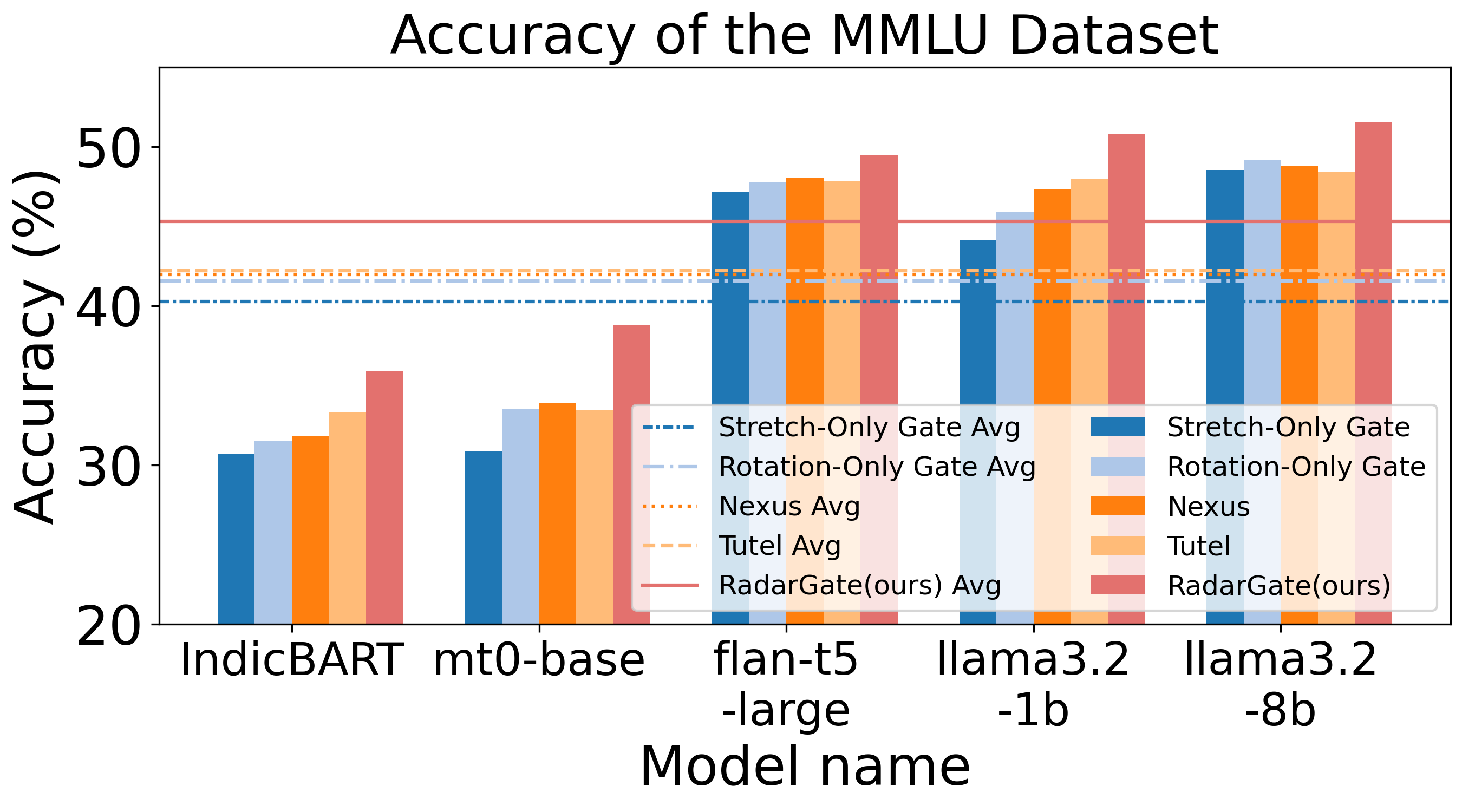}
        \caption{}
        \label{fig:app_exp_module_sca_hyd_subfig2}
    \end{subfigure}
    \hfill
    \begin{subfigure}[t]{0.32\textwidth}
        \includegraphics[width=\textwidth]{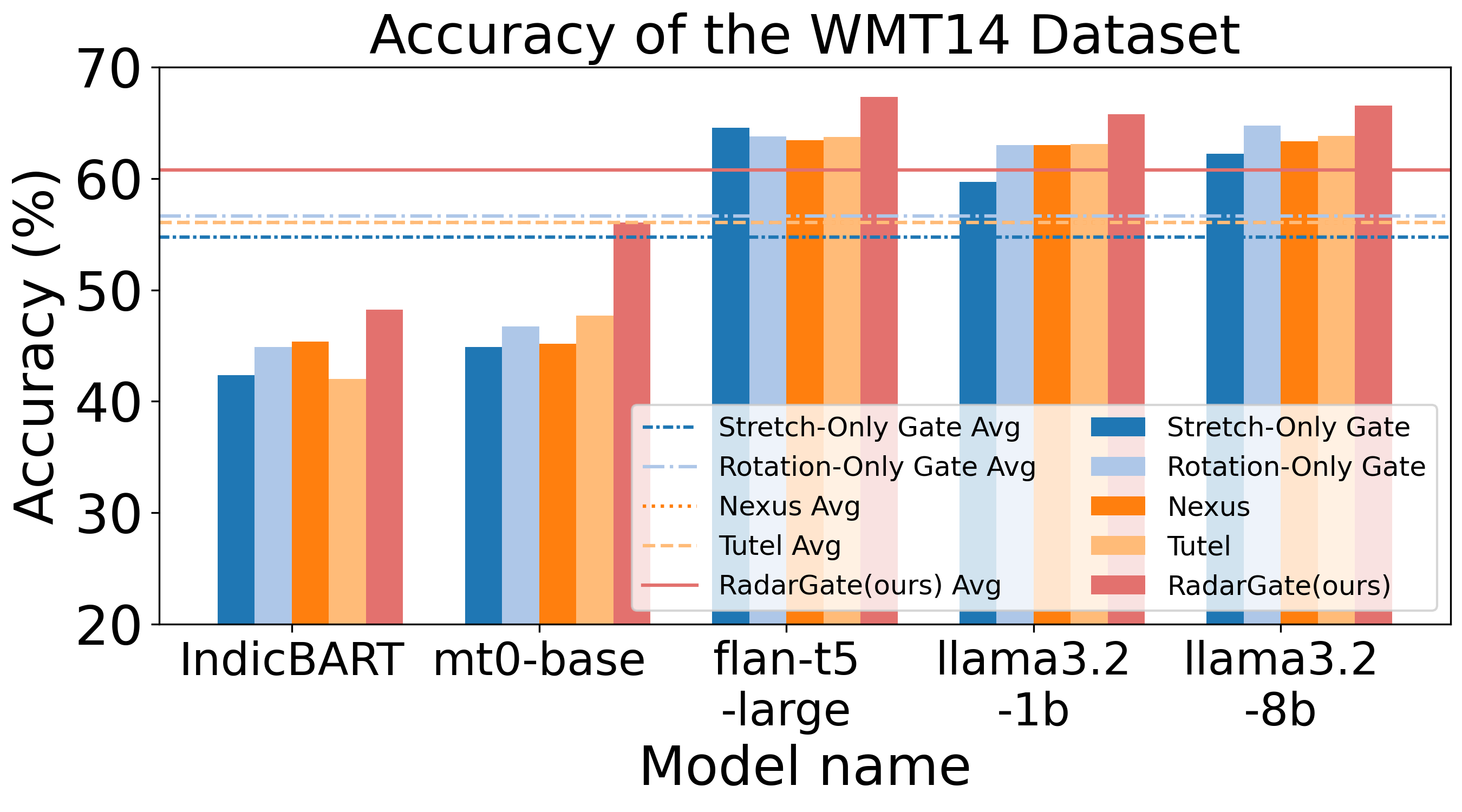}
        \caption{}
        \label{fig:app_exp_module_sca_hyd_subfig3}
    \end{subfigure}
    \hfill
    \begin{subfigure}[t]{0.32\textwidth}
        \includegraphics[width=\textwidth]{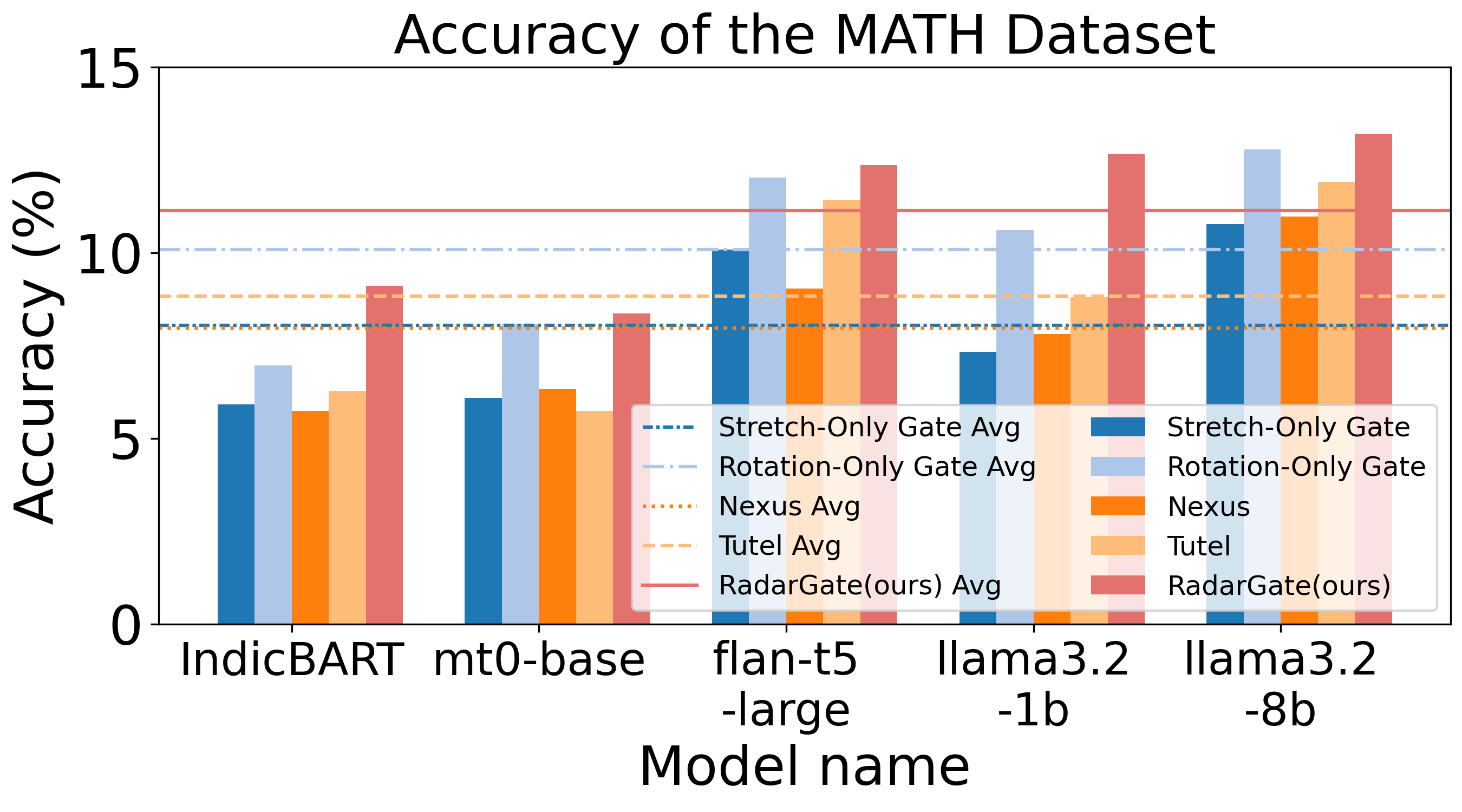}
        \caption{}
        \label{fig:app_exp_module_sca_hyd_subfig4}
    \end{subfigure}
    \hfill
    \begin{subfigure}[t]{0.32\textwidth}
        \includegraphics[width=\textwidth]{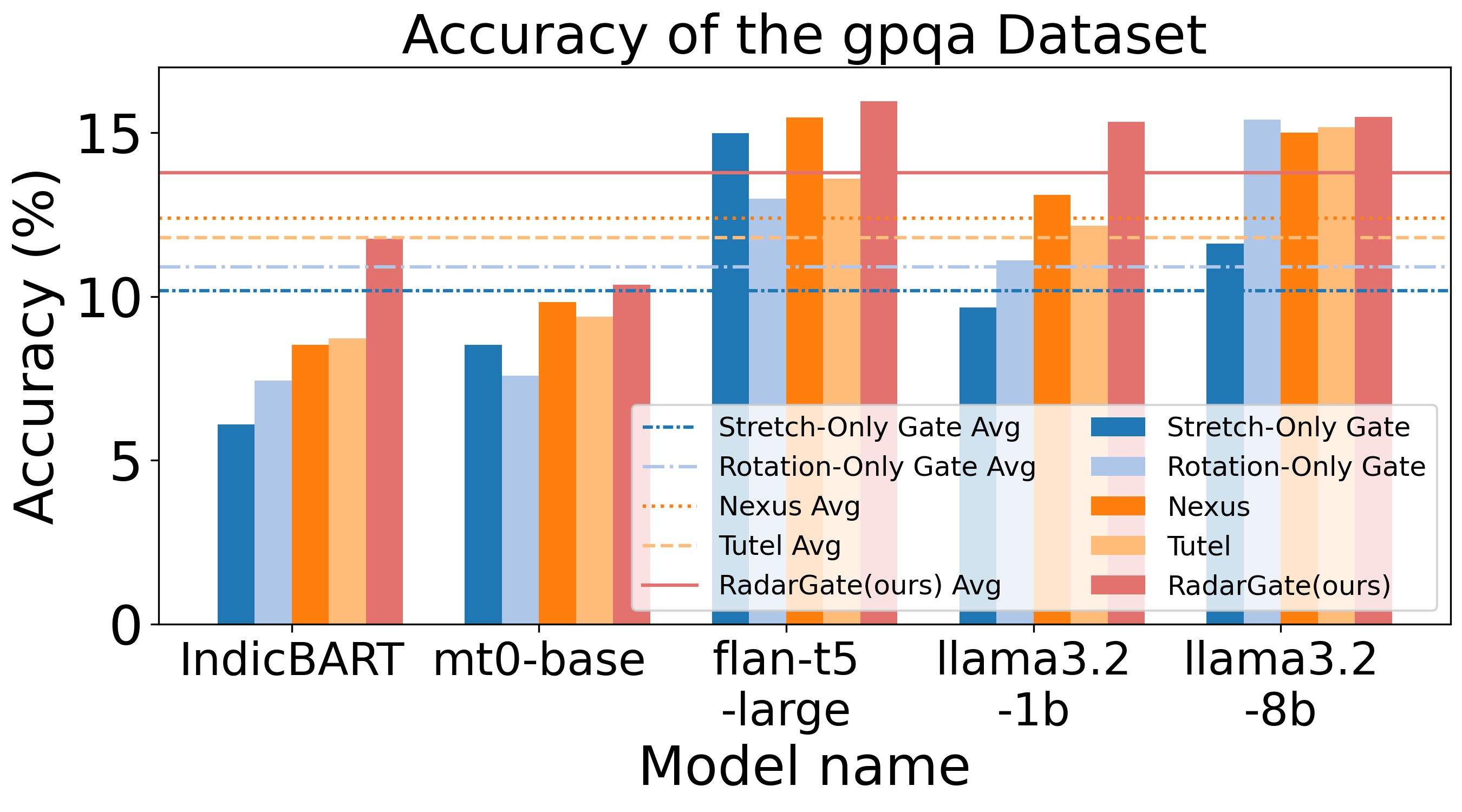}
        \caption{}
        \label{fig:app_exp_module_sca_hyd_subfig5}
    \end{subfigure}
    \hfill
    \begin{subfigure}[t]{0.32\textwidth}
        \includegraphics[width=\textwidth]{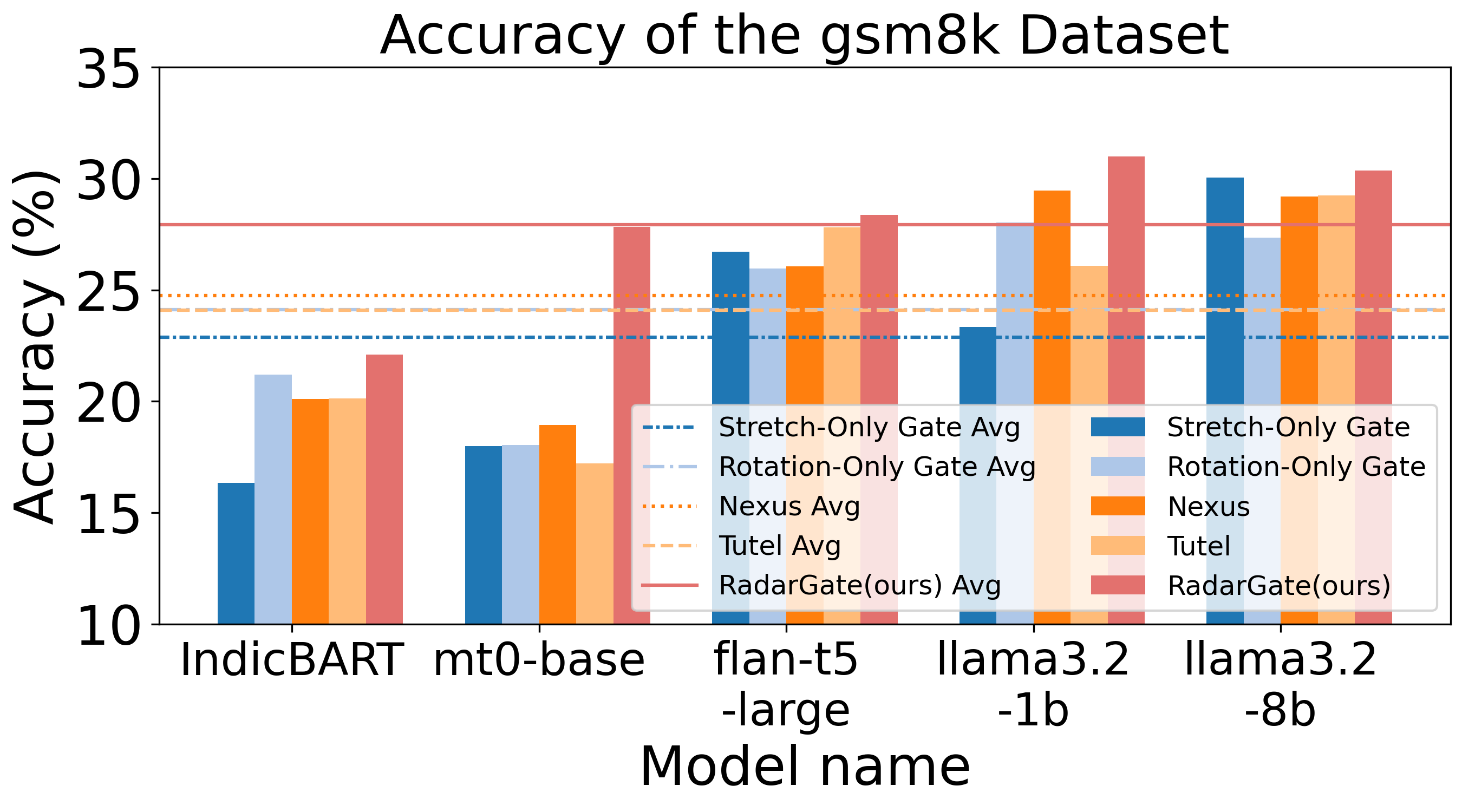}
        \caption{}
        \label{fig:app_exp_module_sca_hyd_subfig6}
    \end{subfigure}
\end{minipage}

\caption{The accuracy of different gates within the HydraLoRA architecture across six different benchmarks as the model parameters vary.} 
\label{fig:app_exp_model_sca_hyd} 
\end{center}
\vskip -0.2in
\end{figure}

\begin{figure}[htpb]
\vskip 0.2in
\begin{center}
\begin{minipage}{\textwidth}
    \centering
    
    \begin{subfigure}[t]{0.32\textwidth}
        \includegraphics[width=\textwidth]{appendix_scaling_models/Figure_mole_GLUE.png}
        \caption{}
        \label{fig:app_exp_module_sca_mole_subfig21}
    \end{subfigure}
    \hfill
    \begin{subfigure}[t]{0.32\textwidth}
        \includegraphics[width=\textwidth]{appendix_scaling_models/Figure_mole_MMLU.png}
        \caption{}
        \label{fig:app_exp_module_sca_mole_subfig2}
    \end{subfigure}
    \hfill
    \begin{subfigure}[t]{0.32\textwidth}
        \includegraphics[width=\textwidth]{appendix_scaling_models/Figure_mole_WMT14.png}
        \caption{}
        \label{fig:app_exp_module_sca_mole_subfig3}
    \end{subfigure}
    \hfill
    \begin{subfigure}[t]{0.32\textwidth}
        \includegraphics[width=\textwidth]{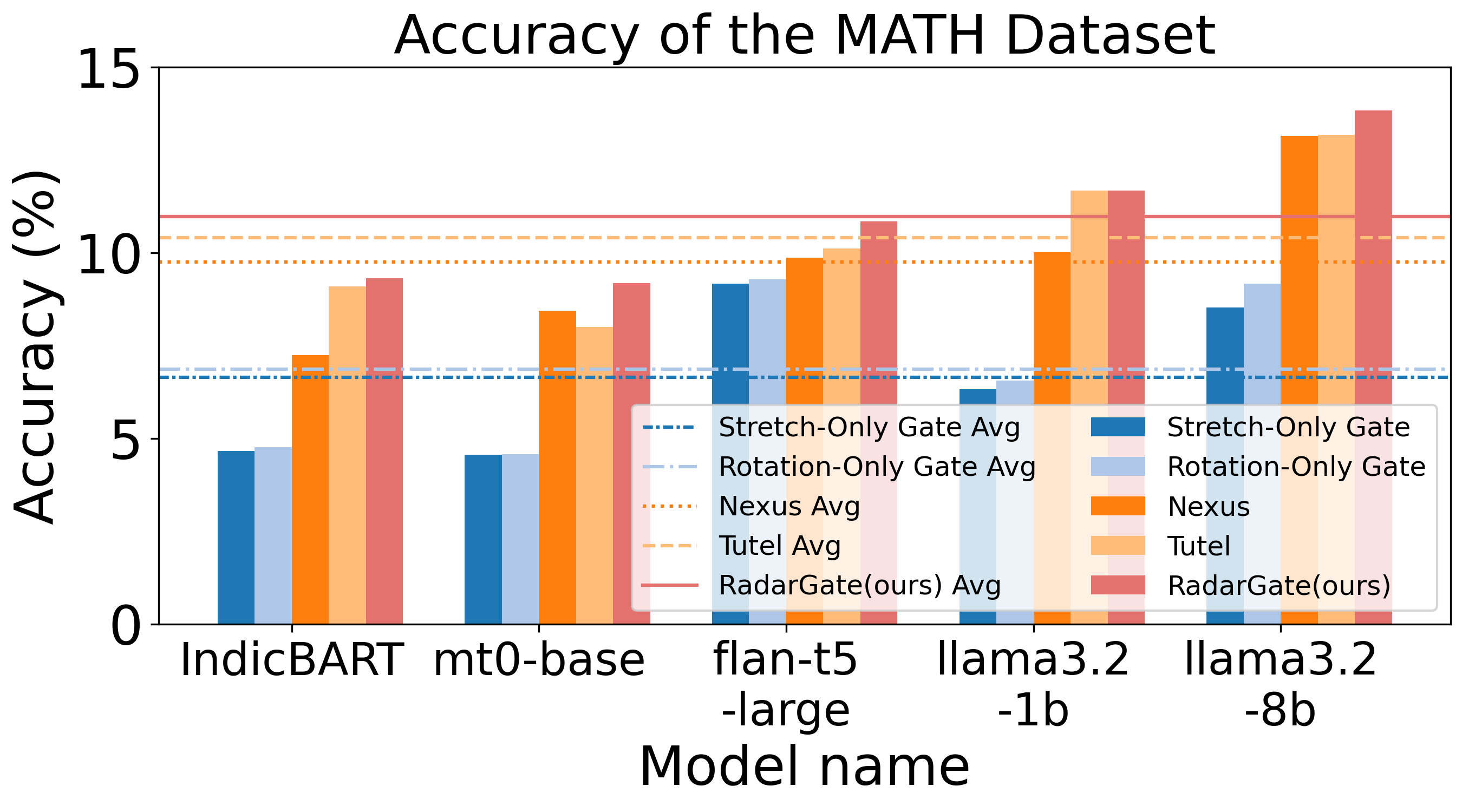}
        \caption{}
        \label{fig:app_exp_module_sca_mole_subfig4}
    \end{subfigure}
    \hfill
    \begin{subfigure}[t]{0.32\textwidth}
        \includegraphics[width=\textwidth]{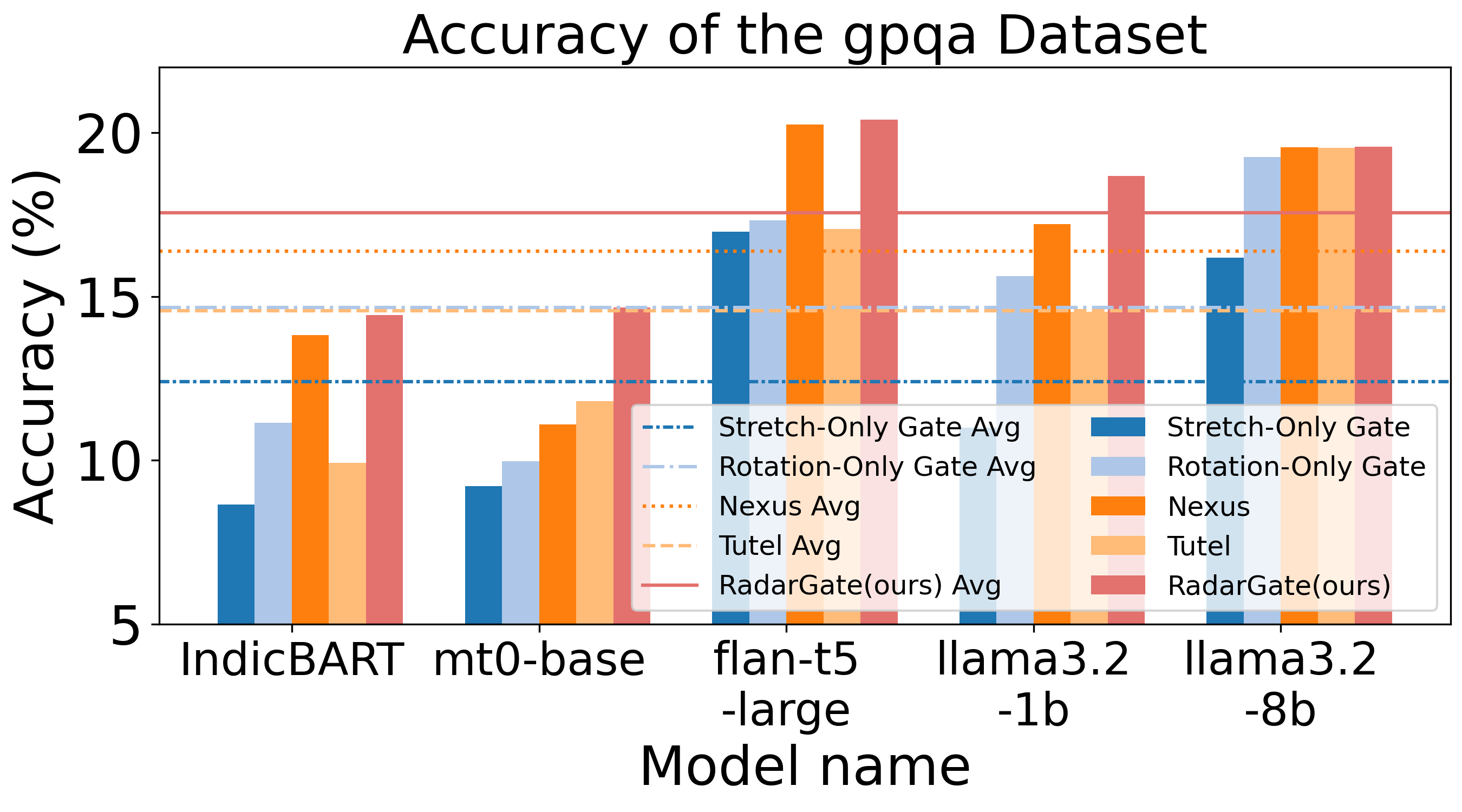}
        \caption{}
        \label{fig:app_exp_module_sca_mole_subfig5}
    \end{subfigure}
    \hfill
    \begin{subfigure}[t]{0.32\textwidth}
        \includegraphics[width=\textwidth]{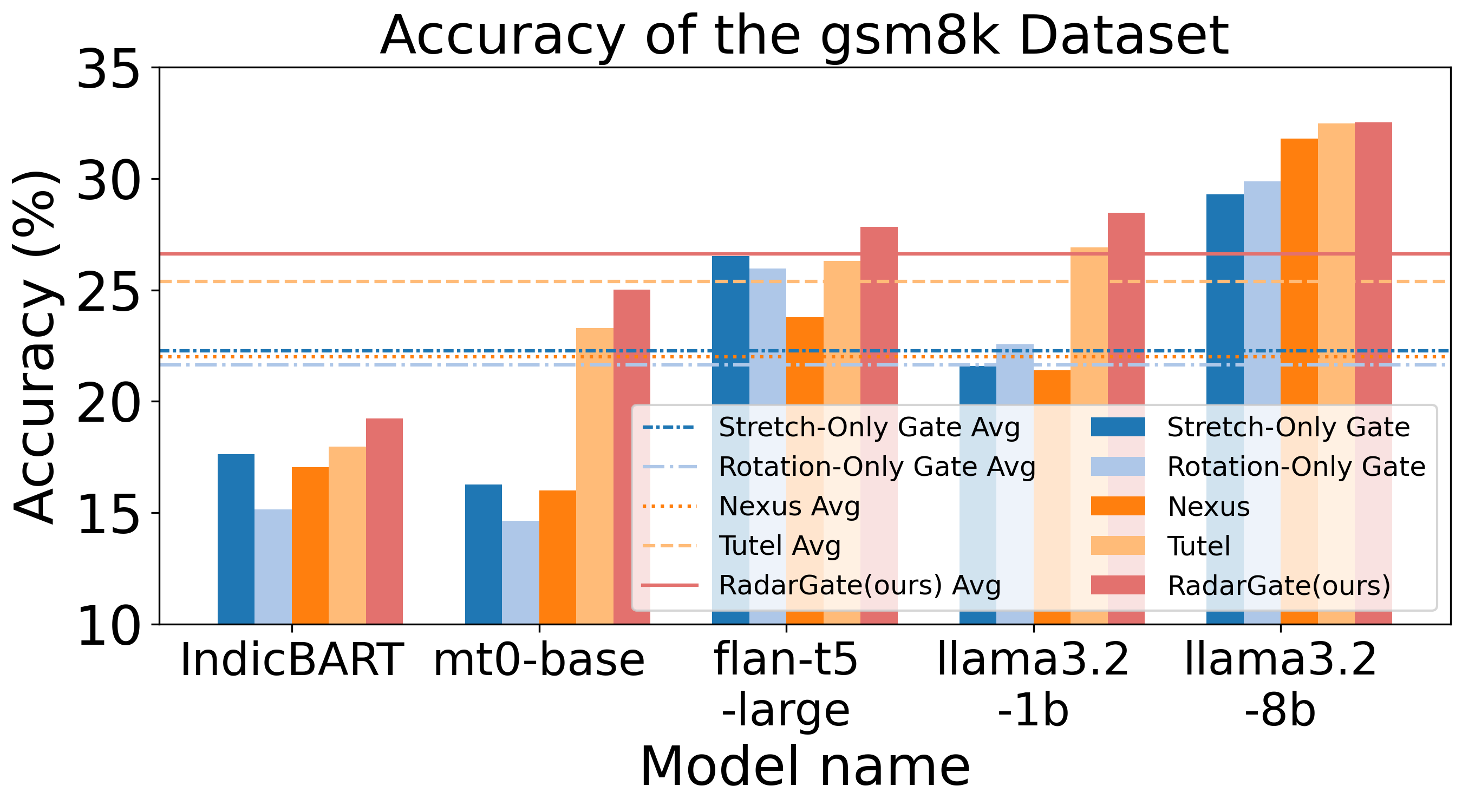}
        \caption{}
        \label{fig:app_exp_module_sca_mole_subfig6}
    \end{subfigure}
\end{minipage}

\caption{The accuracy of different gates within the MoLE architecture across six different benchmarks as the model parameters vary.} 
\label{fig:app_exp_model_sca_mole} 
\end{center}
\vskip -0.2in
\end{figure}

\begin{figure}[htpb]
\vskip 0.2in
\begin{center}
\begin{minipage}{\textwidth}
    \centering
    
    \begin{subfigure}[t]{0.32\textwidth}
        \includegraphics[width=\textwidth]{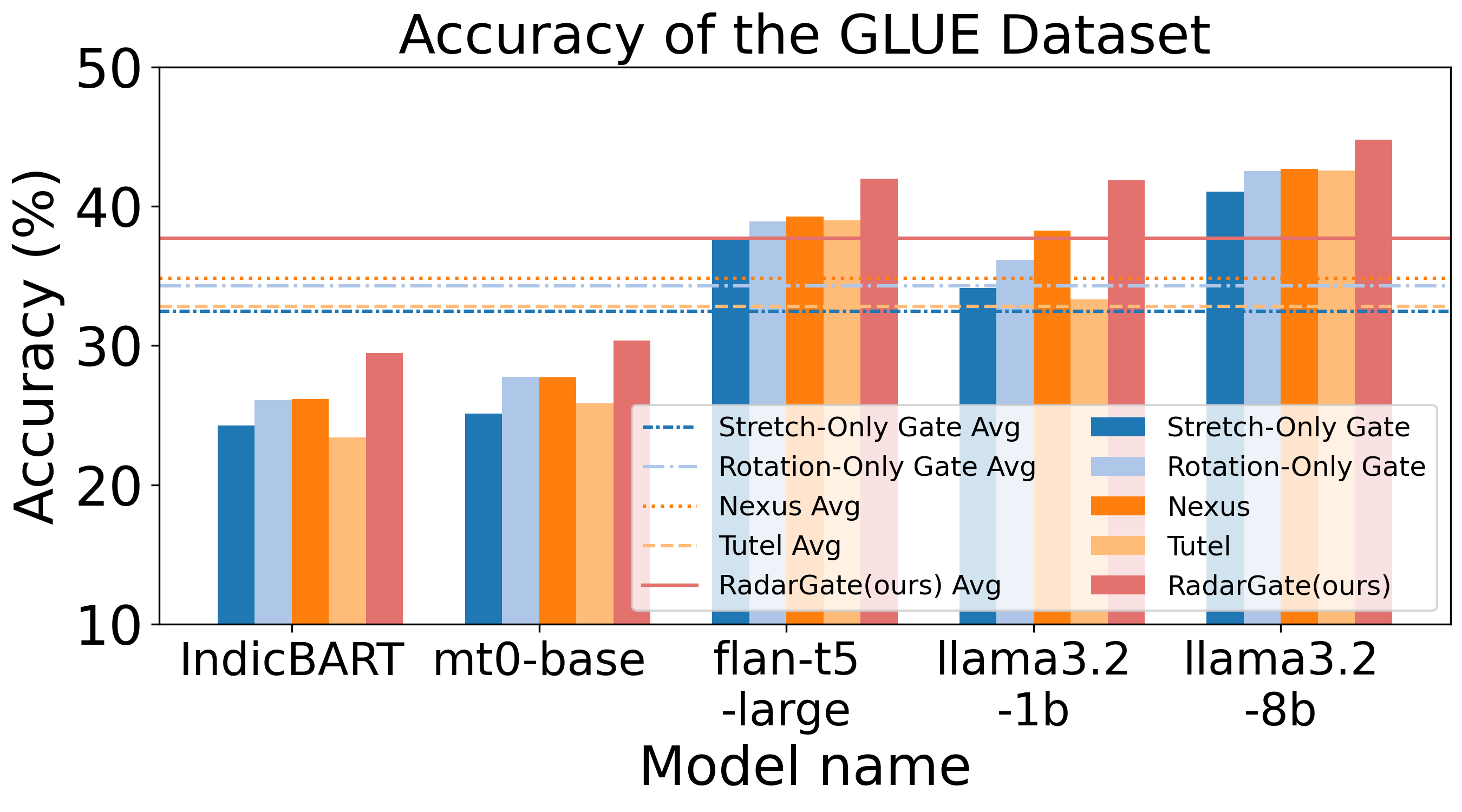}
        \caption{}
        \label{fig:app_exp_module_sca_omoe_subfig21}
    \end{subfigure}
    \hfill
    \begin{subfigure}[t]{0.32\textwidth}
        \includegraphics[width=\textwidth]{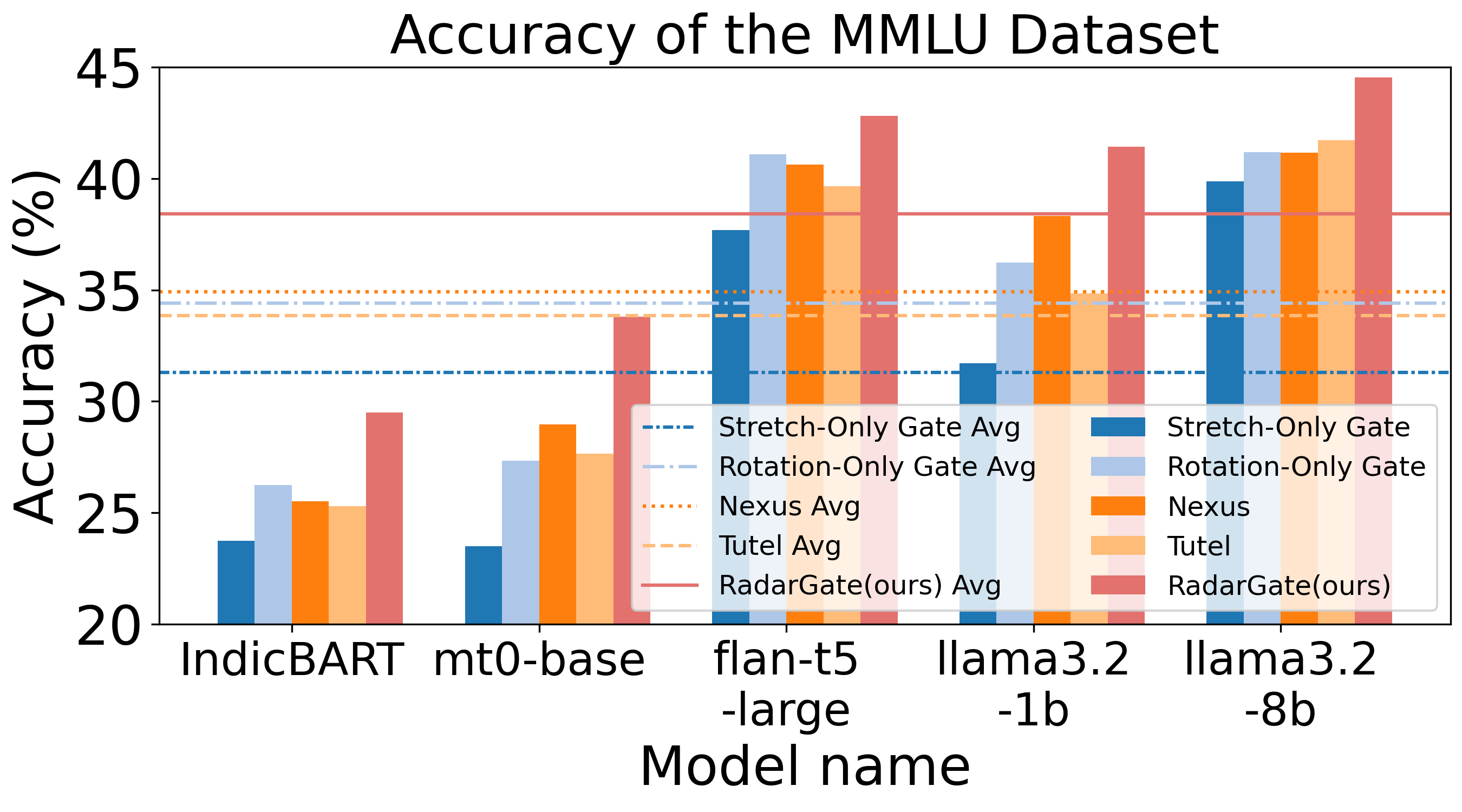}
        \caption{}
        \label{fig:app_exp_module_sca_omoe_subfig22}
    \end{subfigure}
    \hfill
    \begin{subfigure}[t]{0.32\textwidth}
        \includegraphics[width=\textwidth]{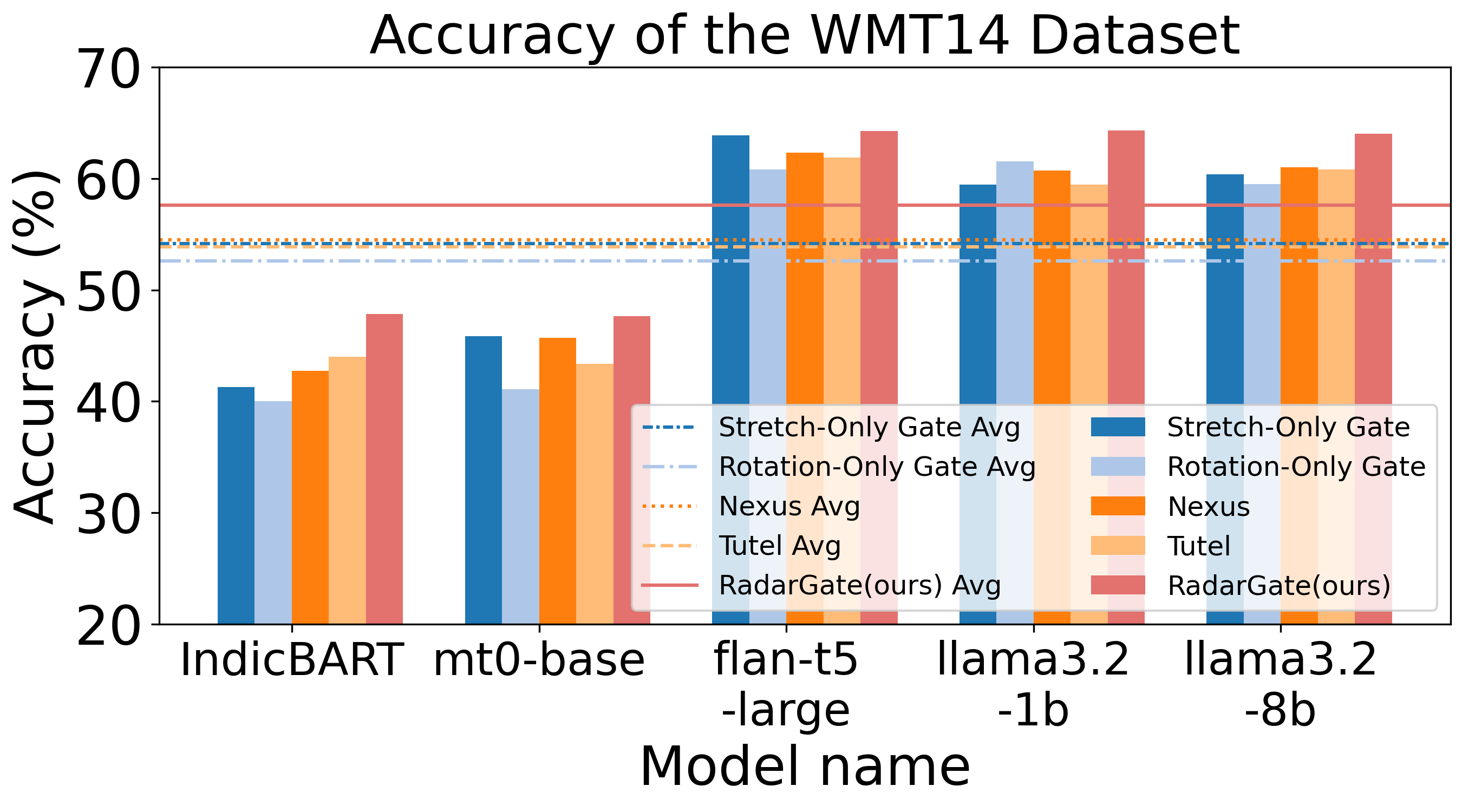}
        \caption{}
        \label{fig:app_exp_module_sca_omoe_subfig23}
    \end{subfigure}
    \hfill
    \begin{subfigure}[t]{0.32\textwidth}
        \includegraphics[width=\textwidth]{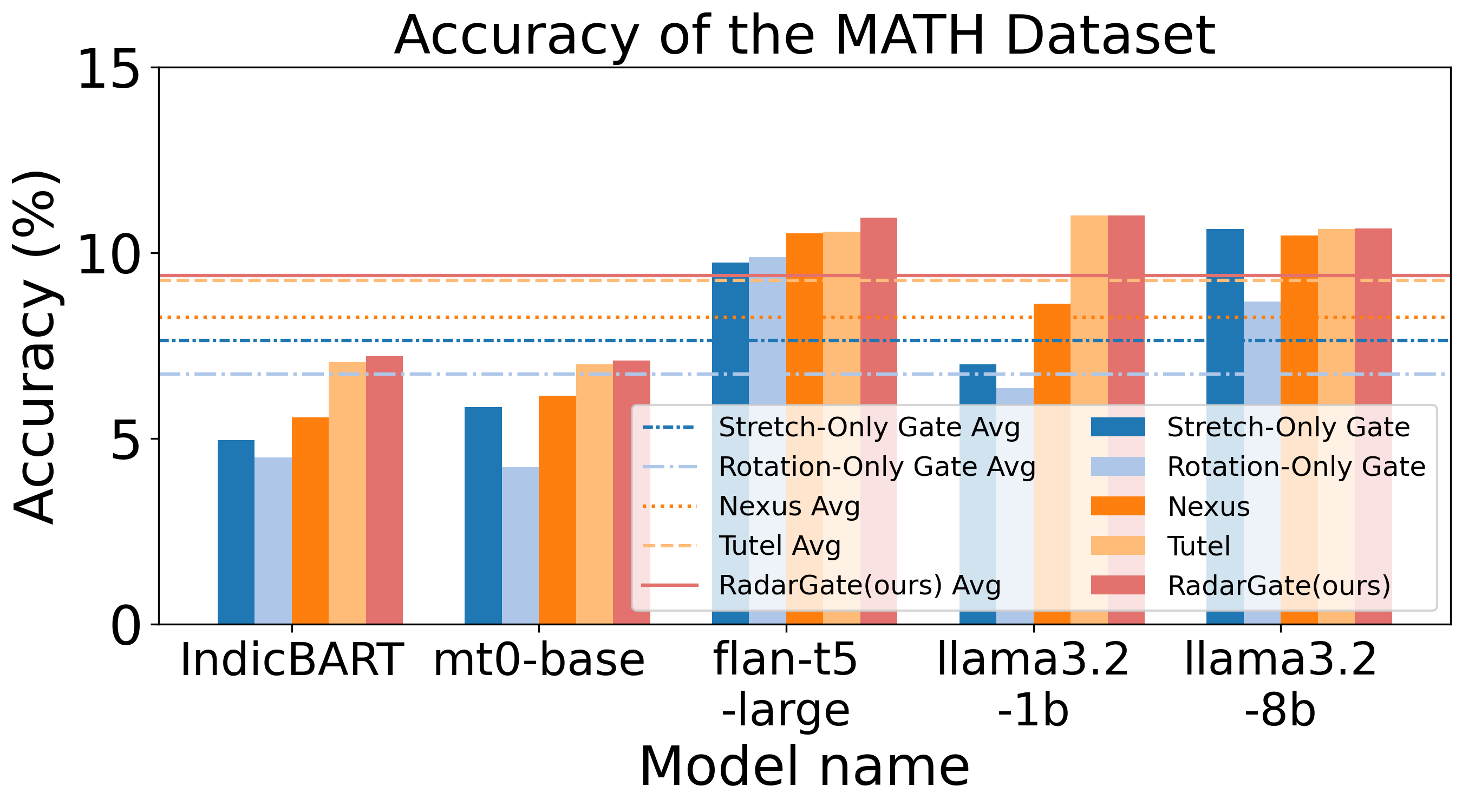}
        \caption{}
        \label{fig:app_exp_model_sca_omoe_subfig24}
    \end{subfigure}
    \hfill
    \begin{subfigure}[t]{0.32\textwidth}
        \includegraphics[width=\textwidth]{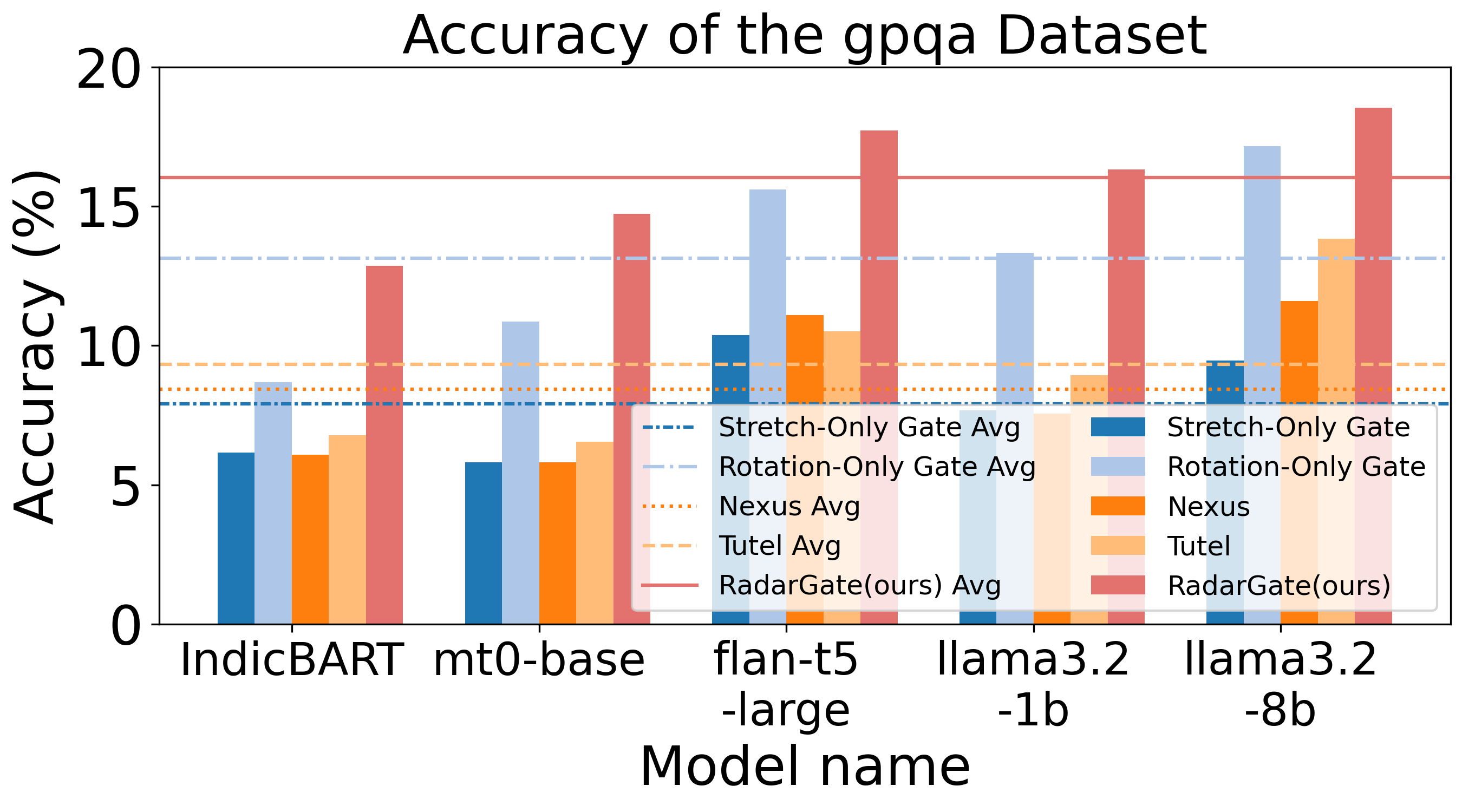}
        \caption{}
        \label{fig:app_exp_model_sca_omoe_subfig25}
    \end{subfigure}
    \hfill
    \begin{subfigure}[t]{0.32\textwidth}
        \includegraphics[width=\textwidth]{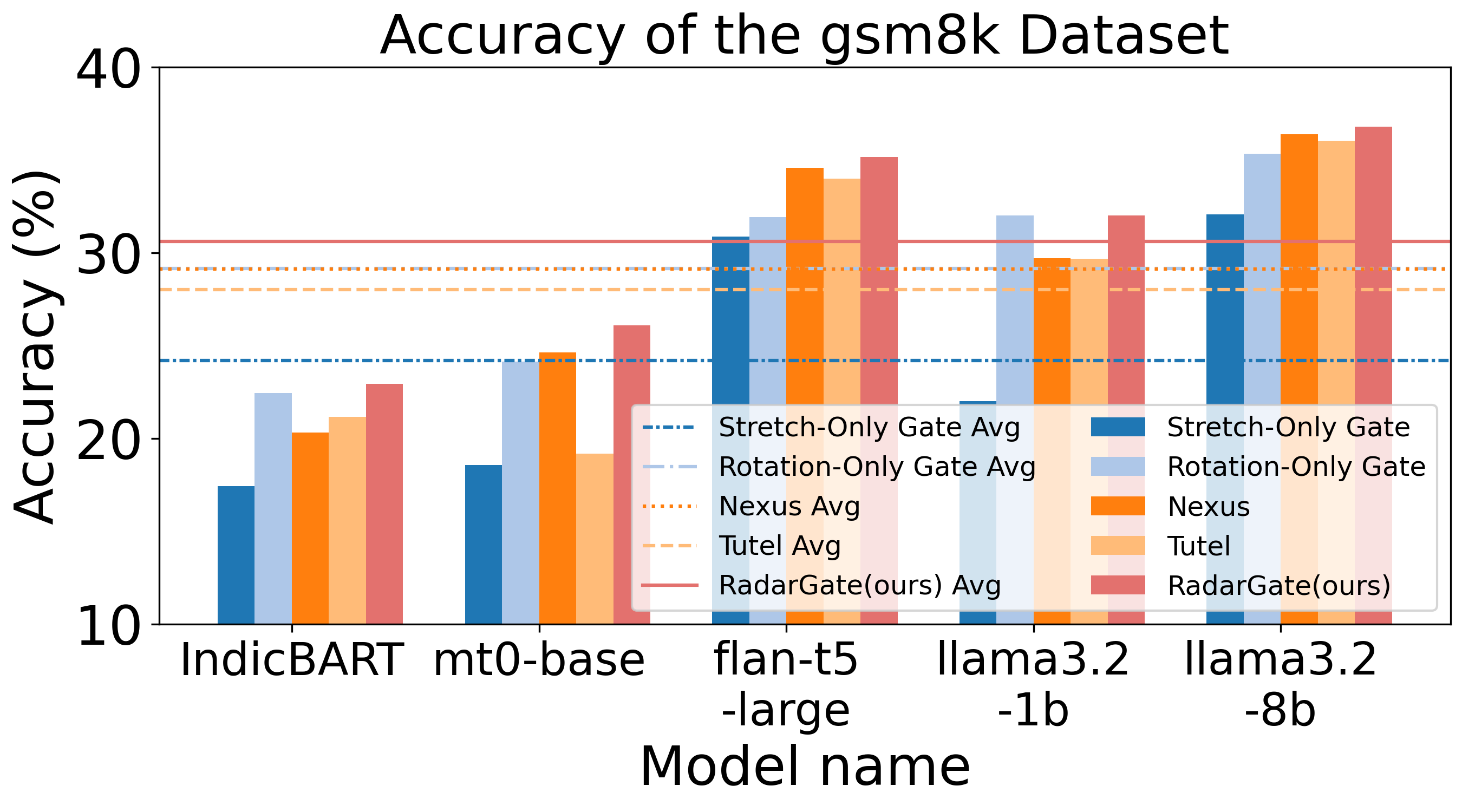}
        \label{fig:app_exp_model_sca_omoe_subfig26}
    \end{subfigure}
\end{minipage}

\caption{The accuracy of different gates within the OMoE architecture across six different benchmarks as the model parameters vary.} 
\label{fig:app_exp_model_sca_omoe} 
\end{center}
\vskip -0.2in
\end{figure}

\begin{figure}[htpb]
\vskip 0.2in
\begin{center}
\begin{minipage}{\textwidth}
    \centering
    
    \begin{subfigure}[t]{0.32\textwidth}
        \includegraphics[width=\textwidth]{Convergence_loss_5.png}
        \caption{}
        \label{fig:app_exp_module_sca_omoe_subfig31}
    \end{subfigure}
    \hfill
    \begin{subfigure}[t]{0.32\textwidth}
        \includegraphics[width=\textwidth]{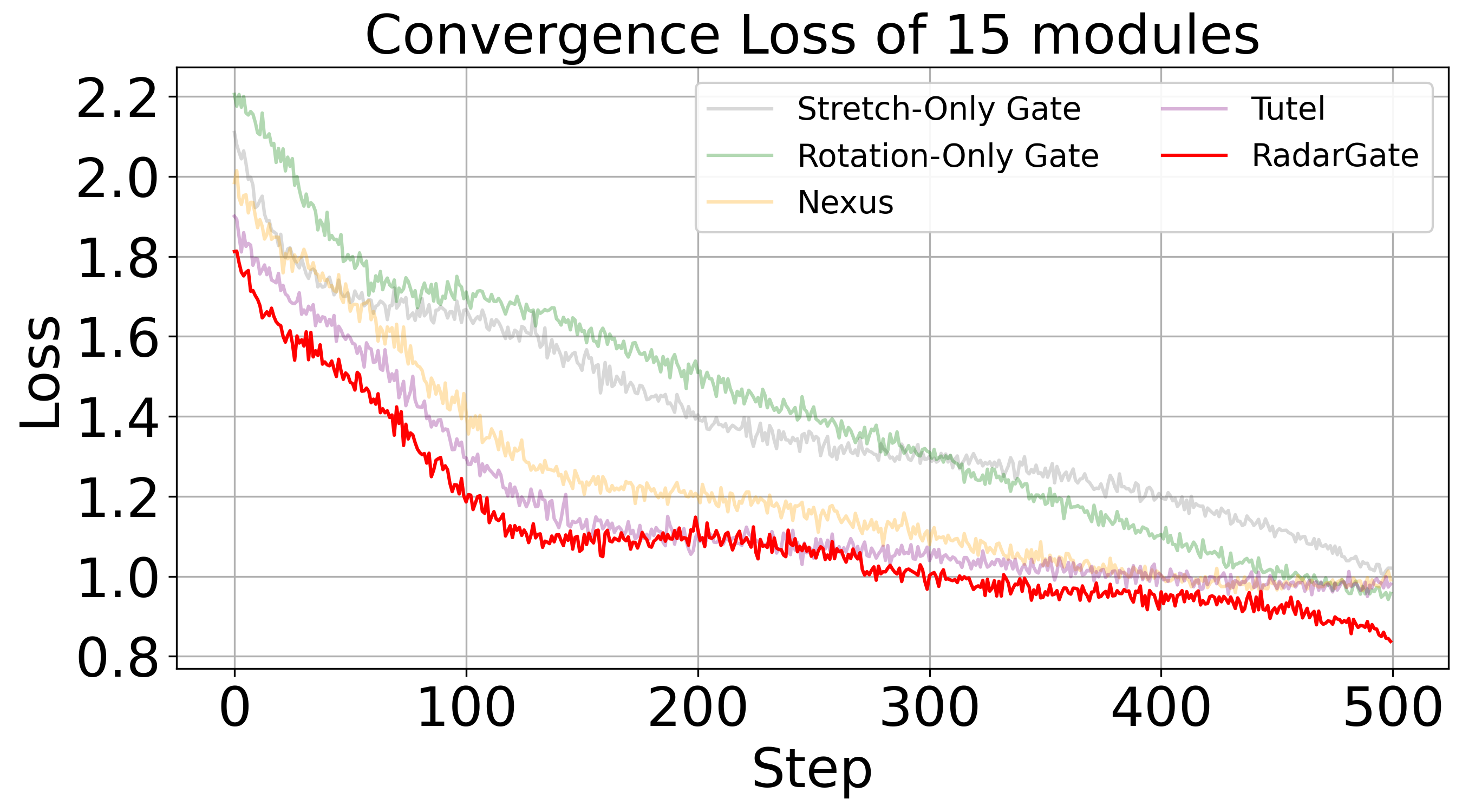}
        \caption{}
        \label{fig:app_exp_module_sca_omoe_subfig32}
    \end{subfigure}
    \hfill
    \begin{subfigure}[t]{0.32\textwidth}
        \includegraphics[width=\textwidth]{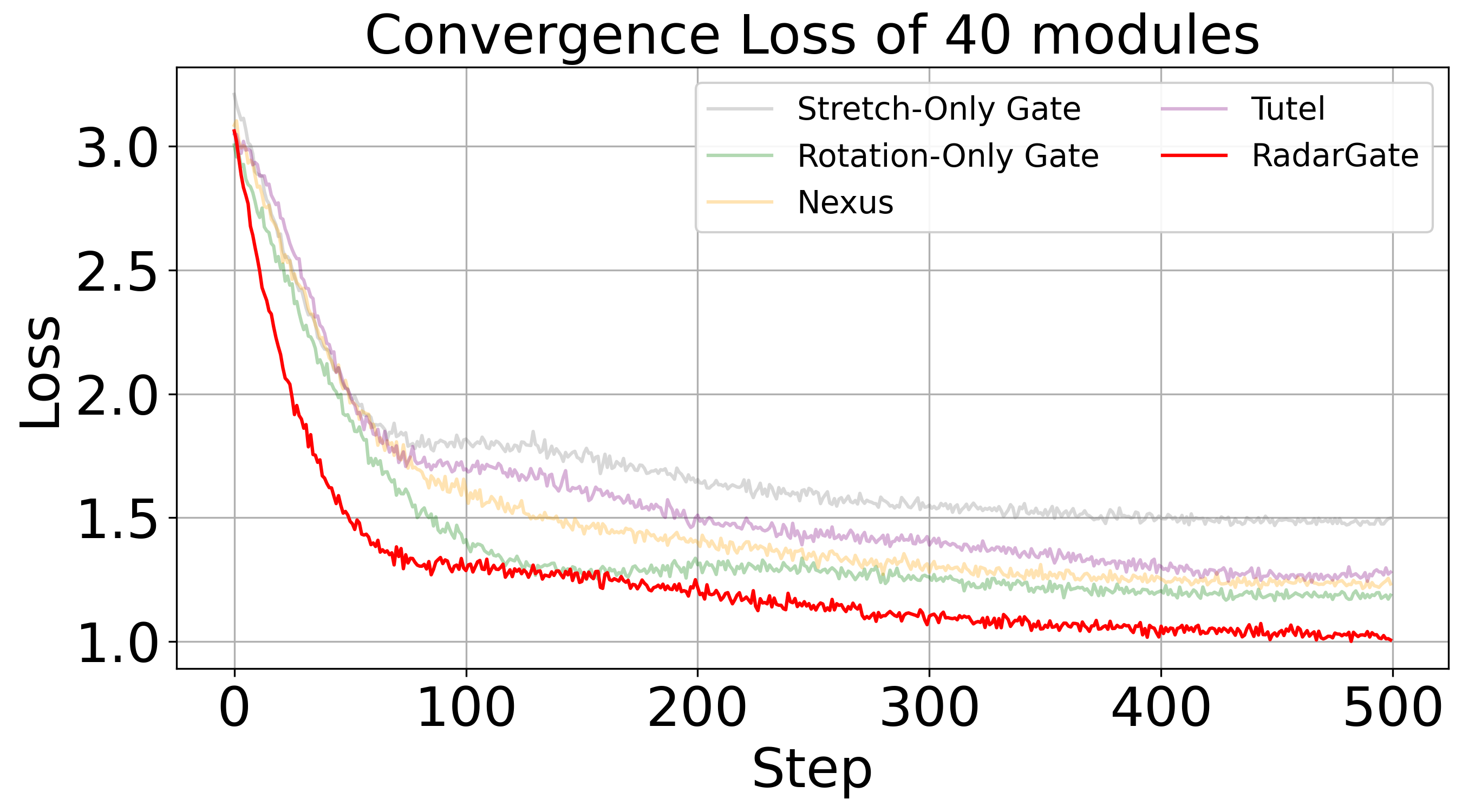}
        \caption{}
        \label{fig:app_conv_pic}
    \end{subfigure}
\end{minipage}

\caption{Figures (a), (b), and (c) show the loss convergence plots for module counts of 5, 15, and 40, respectively.} 
\label{fig:app_exp_sample_conv_omoe} 
\end{center}
\vskip -0.2in
\end{figure}

\begin{figure}[htpb]
\vskip 0.2in
\begin{center}
\begin{minipage}{\textwidth}
    \centering
    
    \begin{subfigure}[t]{0.32\textwidth}
        \includegraphics[width=\textwidth]{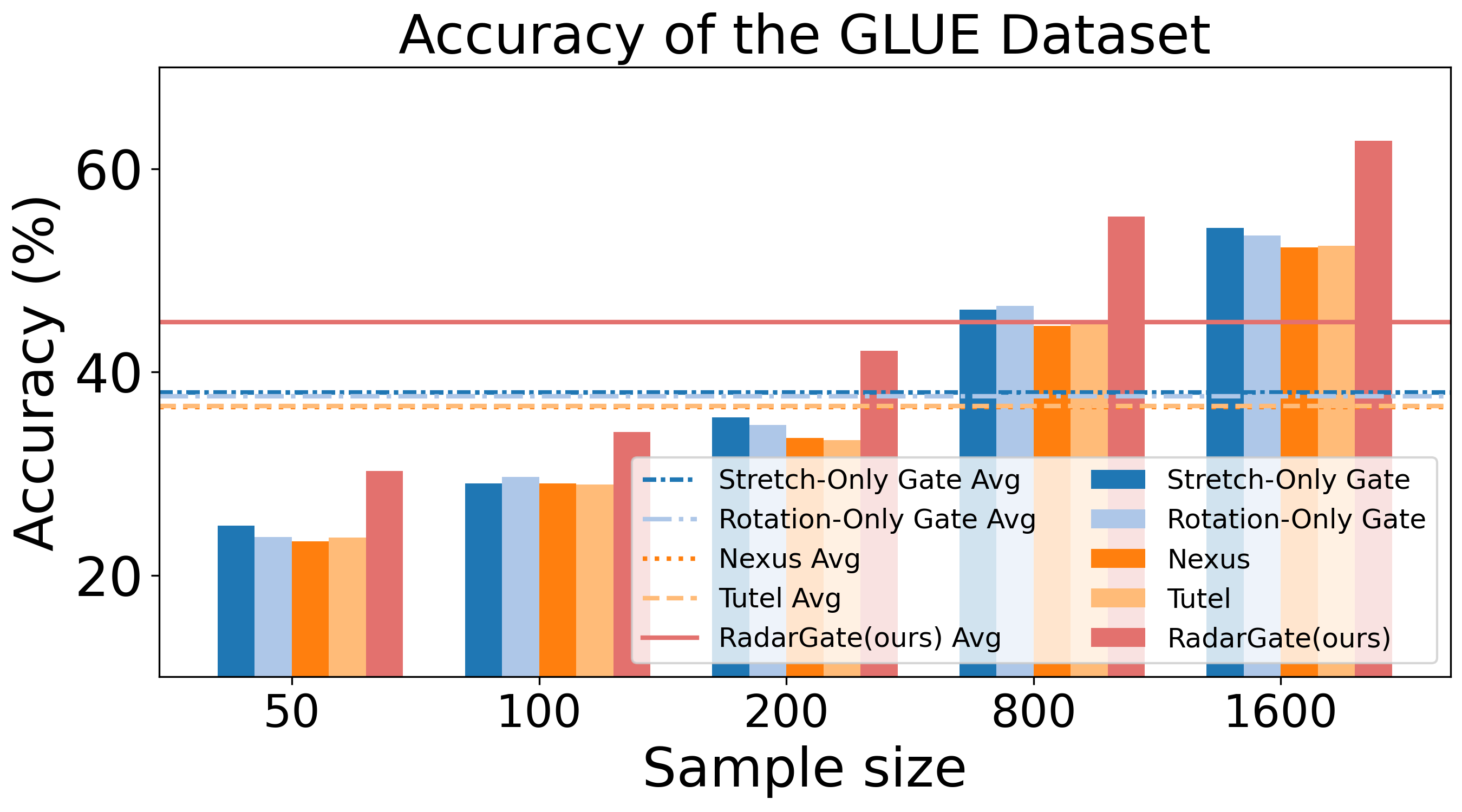}
        \caption{}
        \label{fig:app_exp_module_sca_omoe_subfig41}
    \end{subfigure}
    \hfill
    \begin{subfigure}[t]{0.32\textwidth}
        \includegraphics[width=\textwidth]{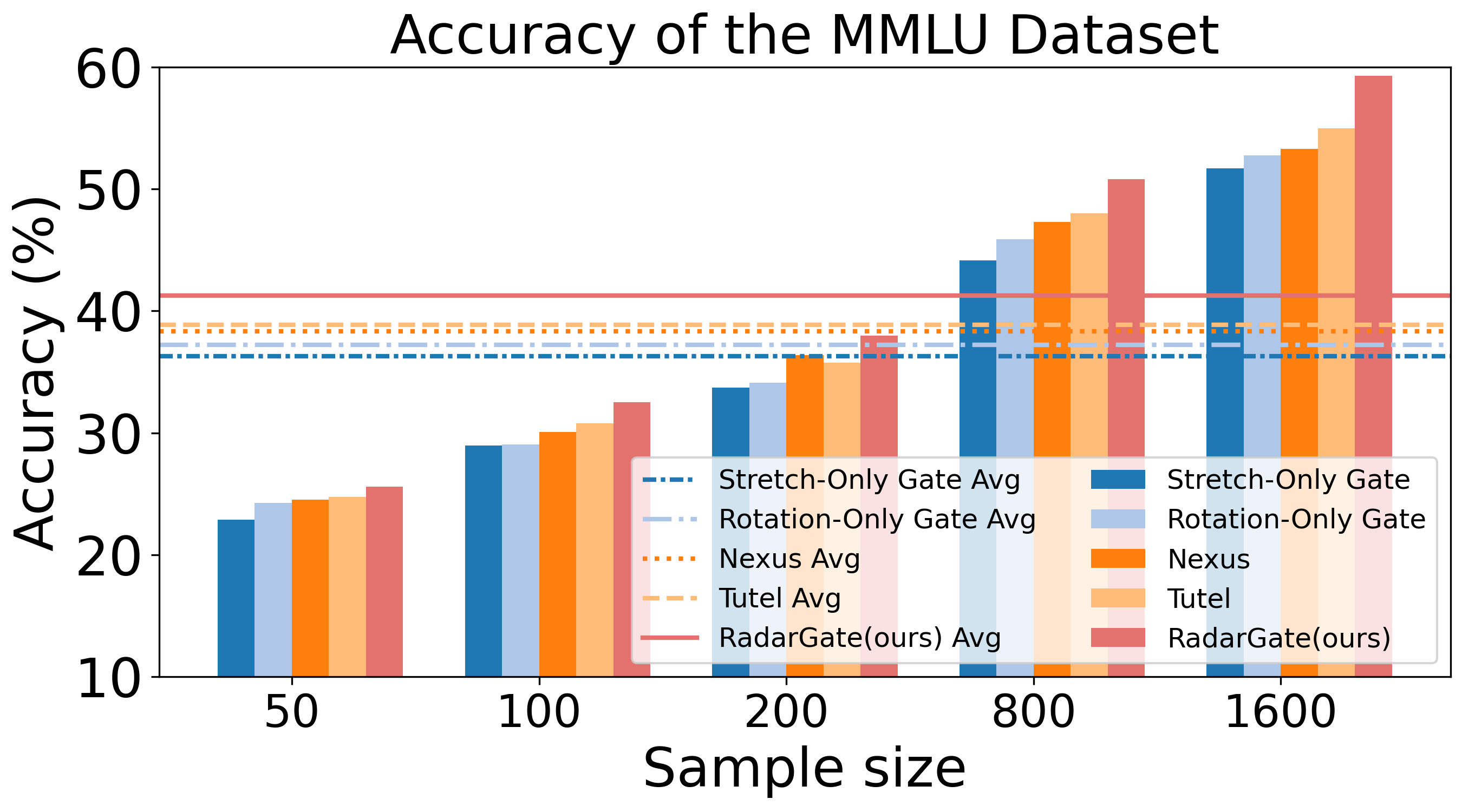}
        \caption{}
        \label{fig:app_exp_module_sca_omoe_subfig42}
    \end{subfigure}
    \hfill
    \begin{subfigure}[t]{0.32\textwidth}
        \includegraphics[width=\textwidth]{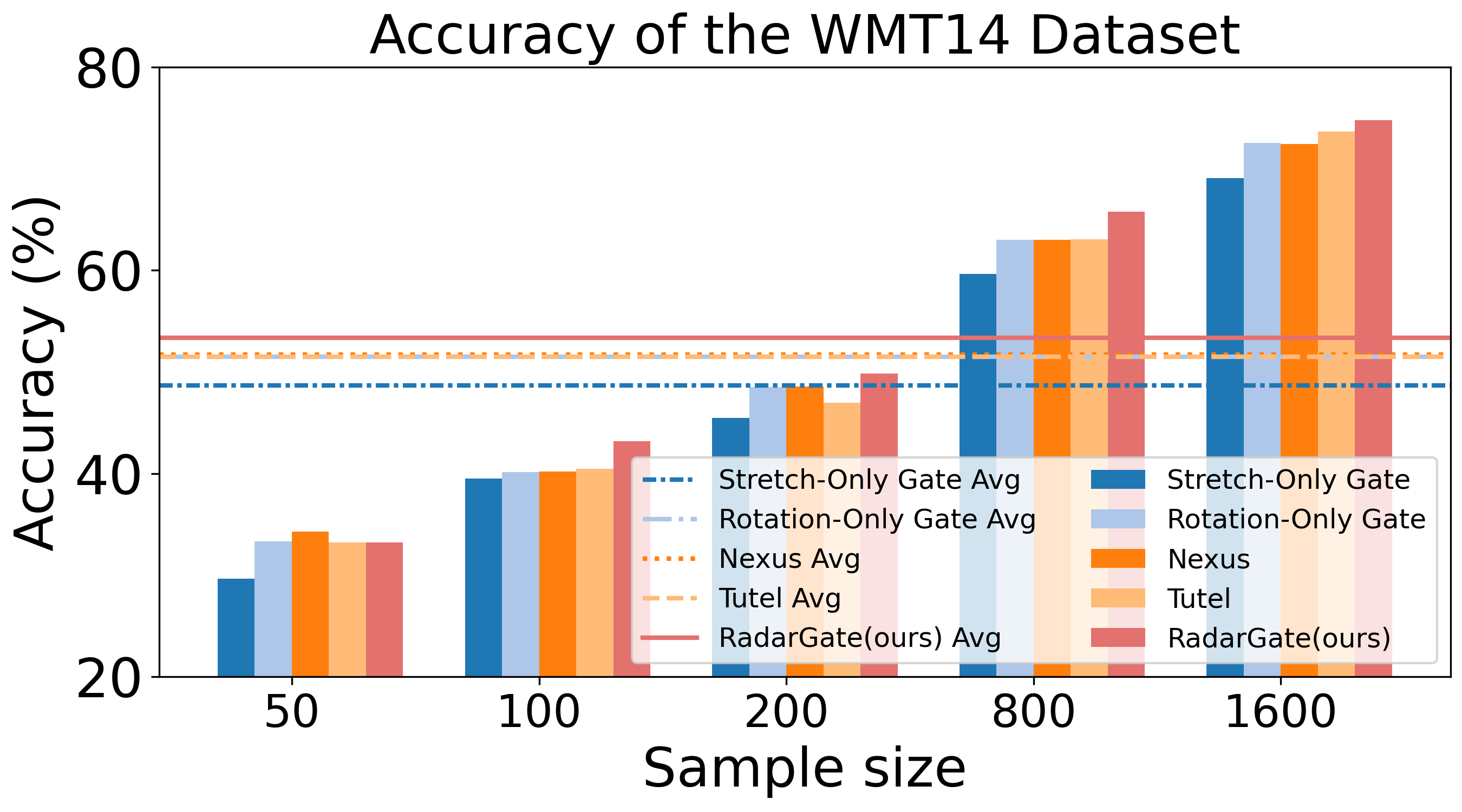}
        \caption{}
        \label{fig:app_exp_module_sca_omoe_subfig43}
    \end{subfigure}
    \hfill
    \begin{subfigure}[t]{0.32\textwidth}
        \includegraphics[width=\textwidth]{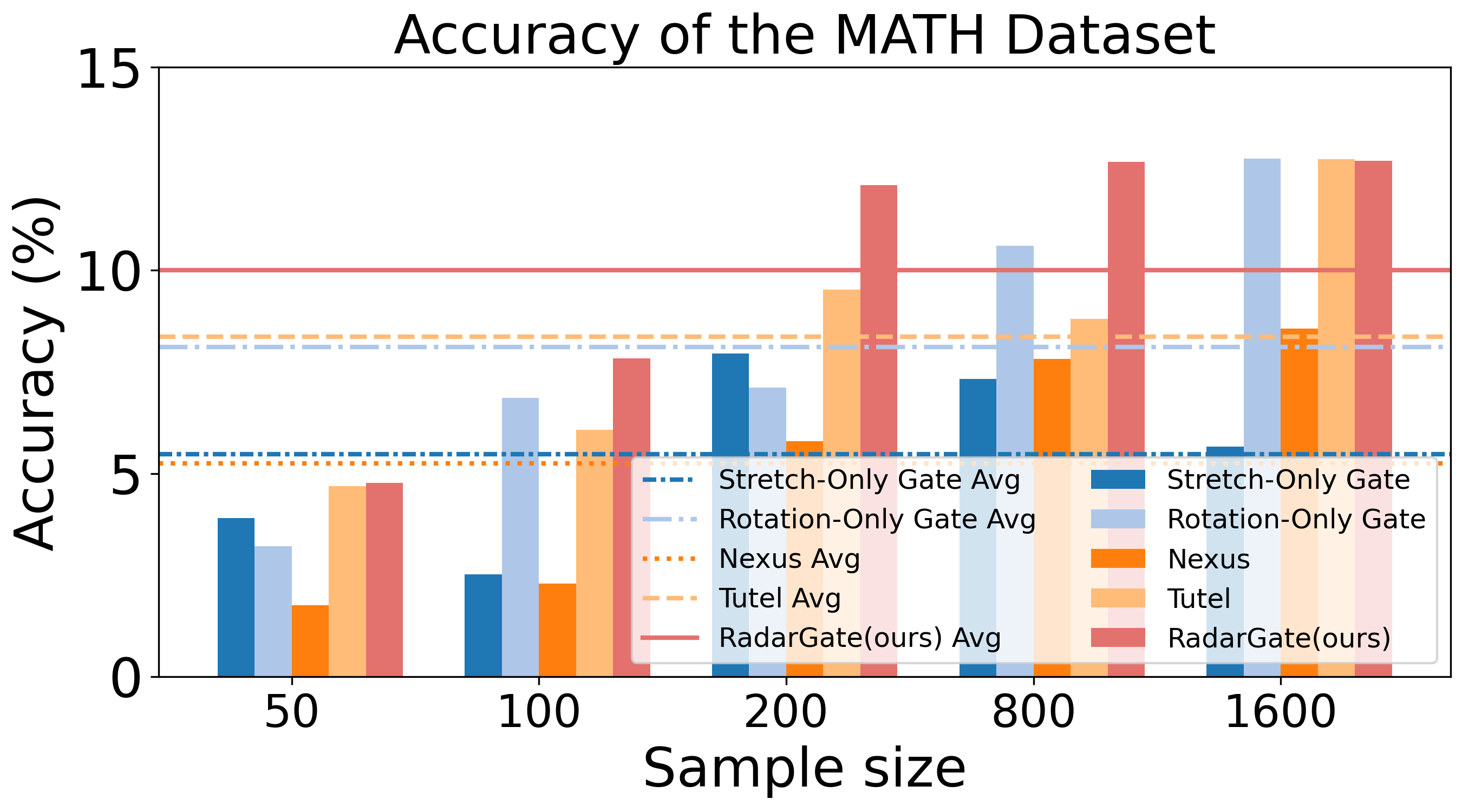}
        \caption{}
        \label{fig:app_exp_module_sca_omoe_subfig44}
    \end{subfigure}
    \hfill
    \begin{subfigure}[t]{0.32\textwidth}
        \includegraphics[width=\textwidth]{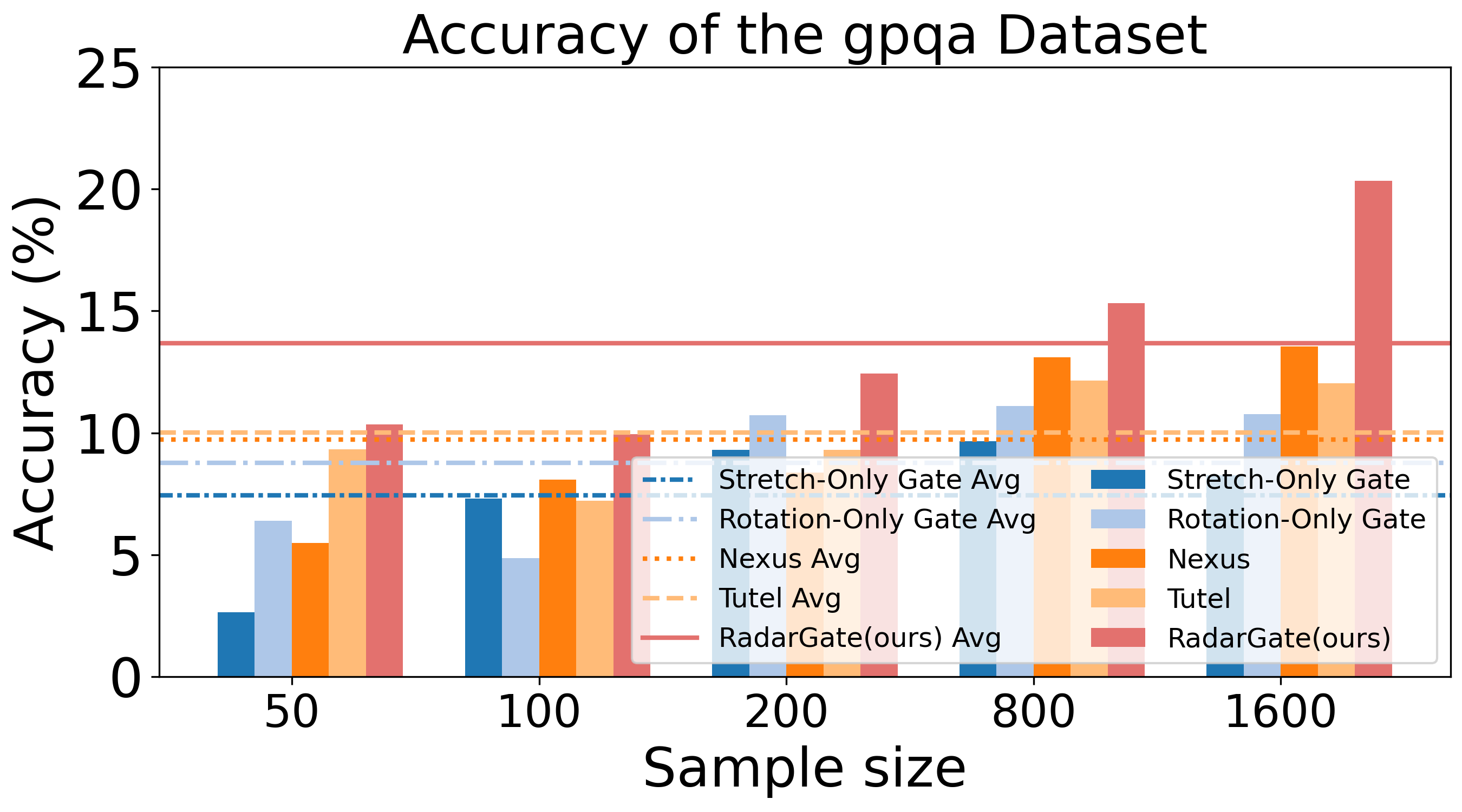}
        \caption{}
        \label{fig:app_exp_module_sca_omoe_subfig45}
    \end{subfigure}
    \hfill
    \begin{subfigure}[t]{0.32\textwidth}
        \includegraphics[width=\textwidth]{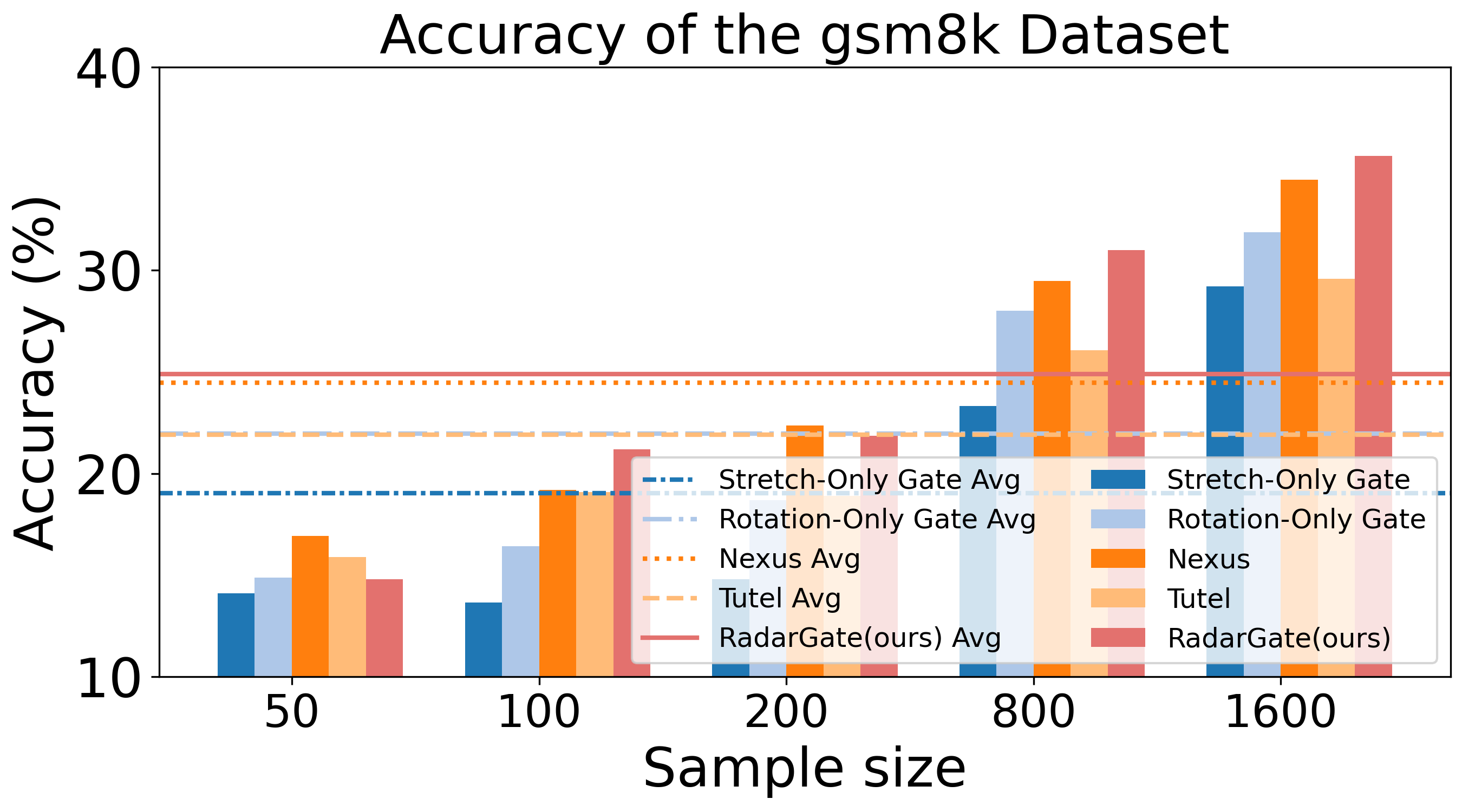}
        \caption{}
        \label{fig:app_exp_module_sca_omoe_subfig46}
    \end{subfigure}
\end{minipage}

\caption{The accuracy variations of different gates within the HydraLoRA architecture across six different benchmarks as the number of training samples increases.} 
\label{fig:app_exp_sample_sca_hyd} 
\end{center}
\vskip -0.2in
\end{figure}

\begin{figure}[htpb]
\vskip 0.2in
\begin{center}
\begin{minipage}{\textwidth}
    \centering
    
    \begin{subfigure}[t]{0.32\textwidth}
        \includegraphics[width=\textwidth]{appendix_sample_size/Sample_size_GLUE_mole.png}
        \caption{}
        \label{fig:app_exp_module_sca_omoe_subfig51}
    \end{subfigure}
    \hfill
    \begin{subfigure}[t]{0.32\textwidth}
        \includegraphics[width=\textwidth]{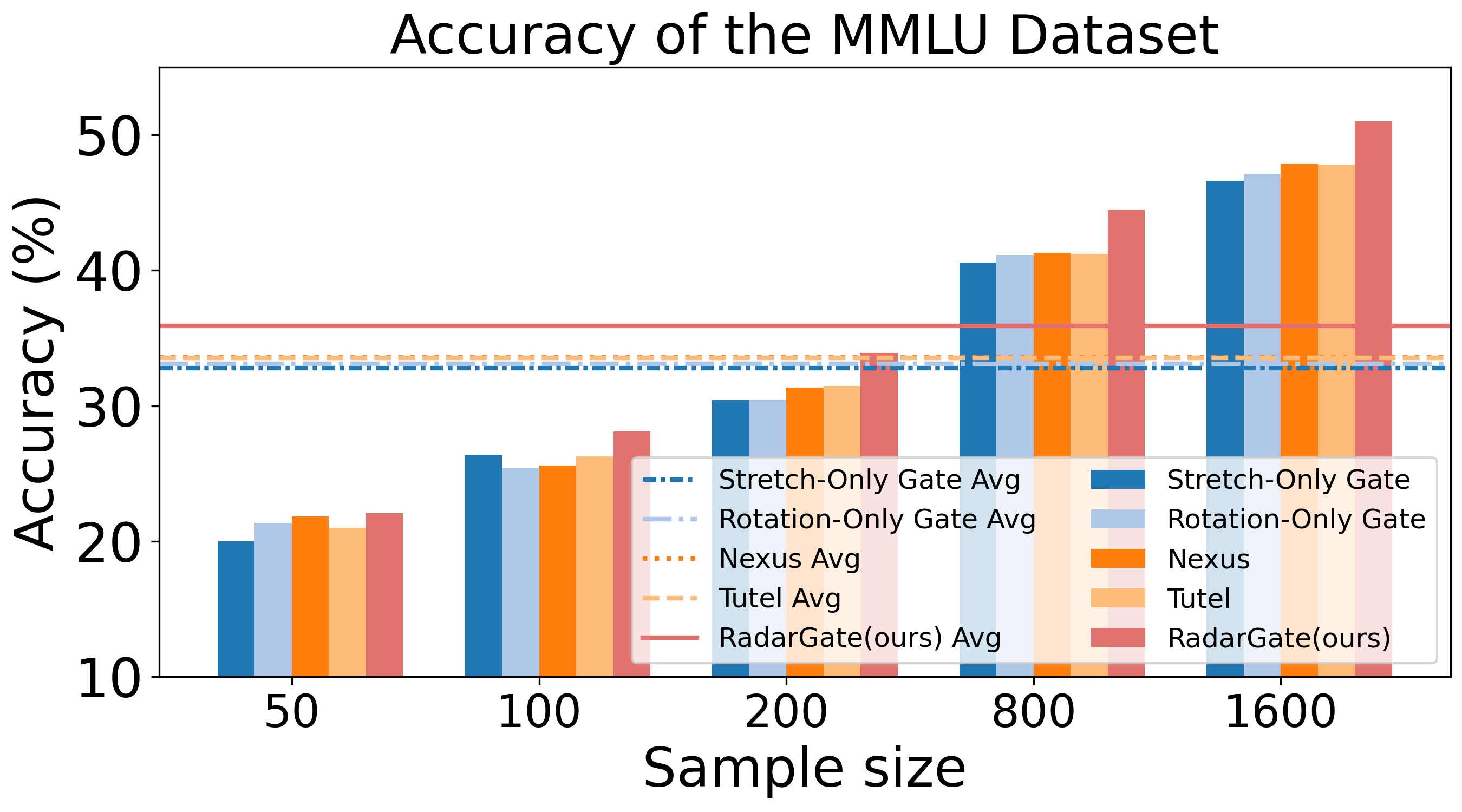}
        \caption{}
        \label{fig:app_exp_module_sca_omoe_subfig52}
    \end{subfigure}
    \hfill
    \begin{subfigure}[t]{0.32\textwidth}
        \includegraphics[width=\textwidth]{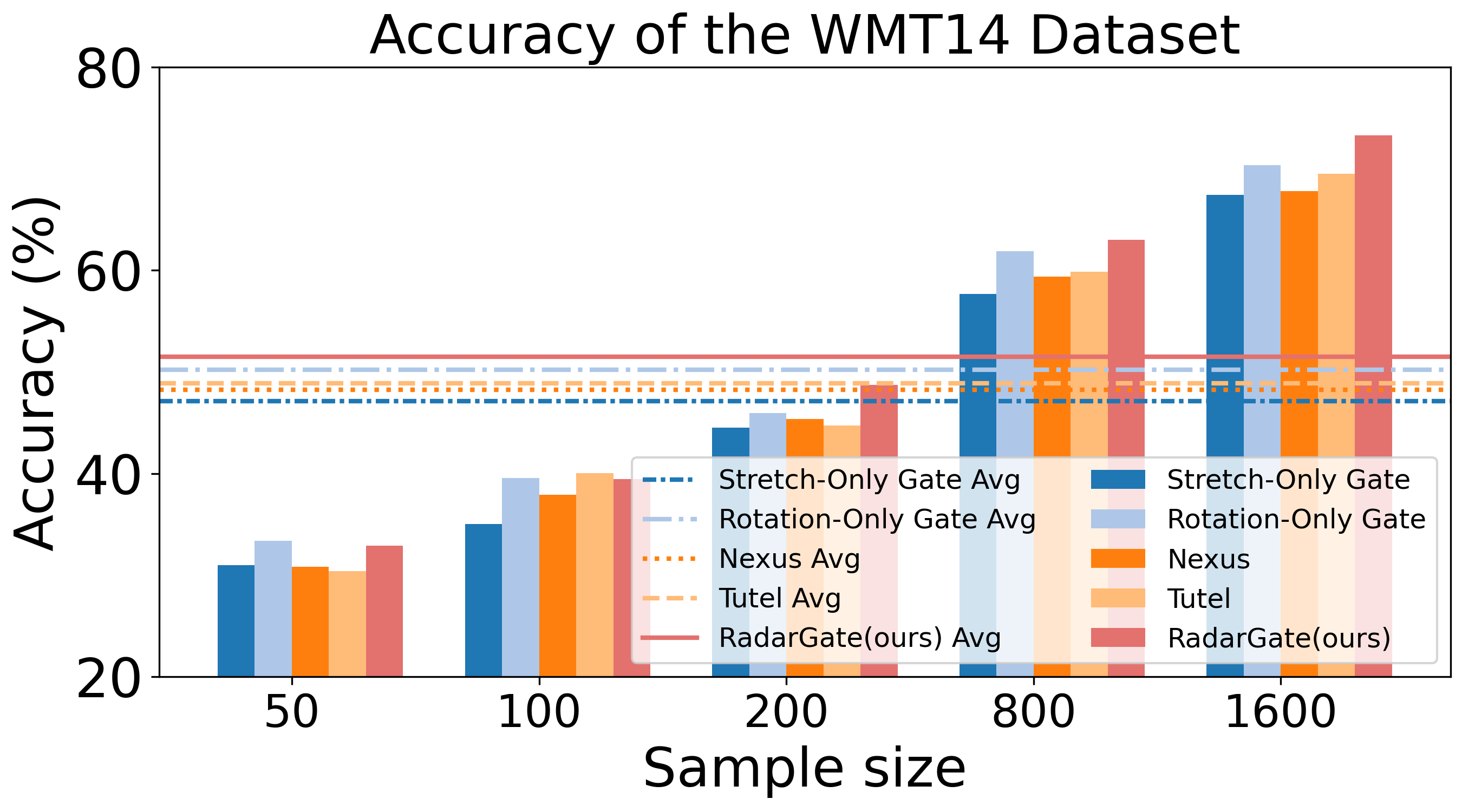}
        \caption{}
        \label{fig:app_exp_module_sca_omoe_subfig53}
    \end{subfigure}
    \hfill
    \begin{subfigure}[t]{0.32\textwidth}
        \includegraphics[width=\textwidth]{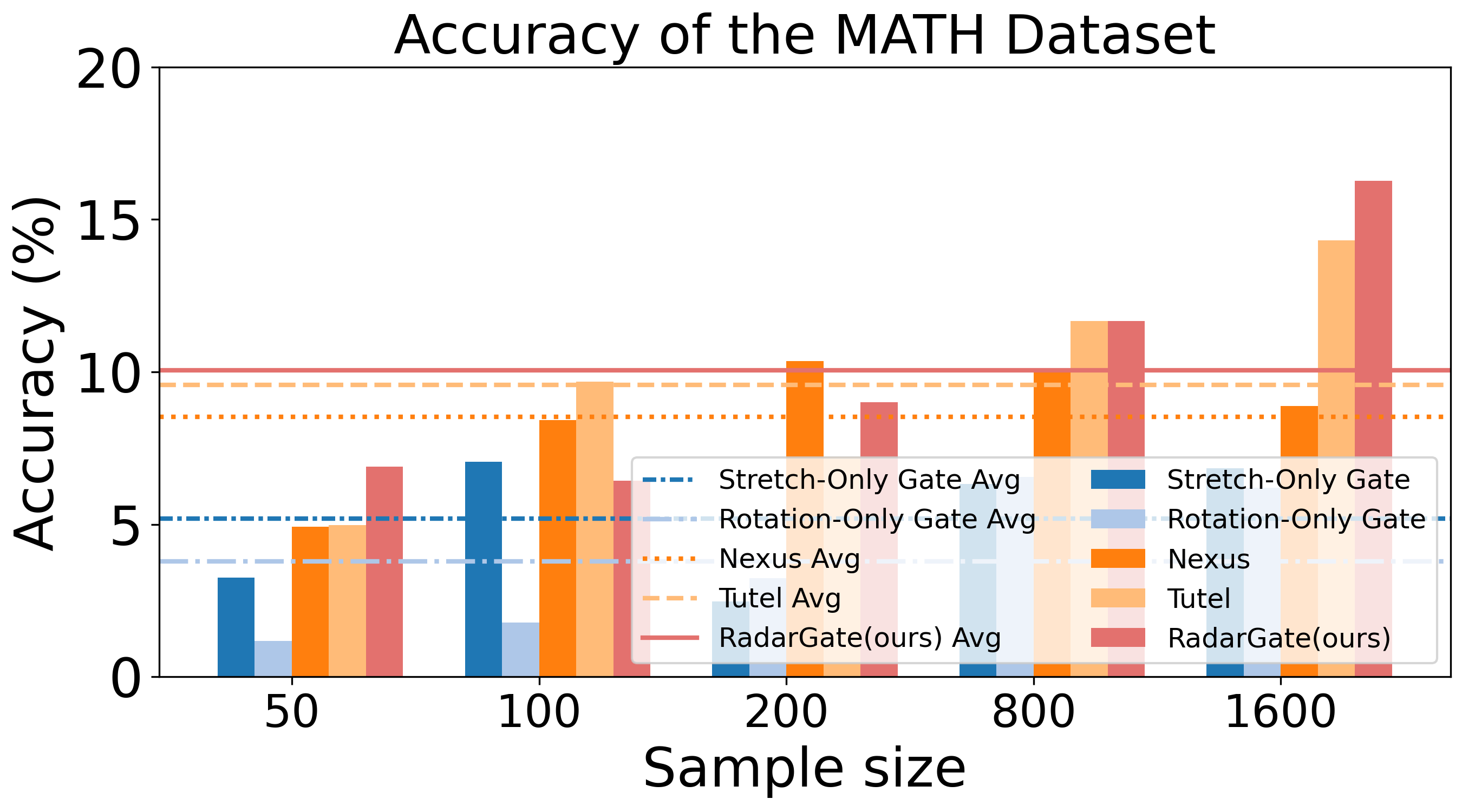}
        \caption{}
        \label{fig:app_exp_module_sca_omoe_subfig54}
    \end{subfigure}
    \hfill
    \begin{subfigure}[t]{0.32\textwidth}
        \includegraphics[width=\textwidth]{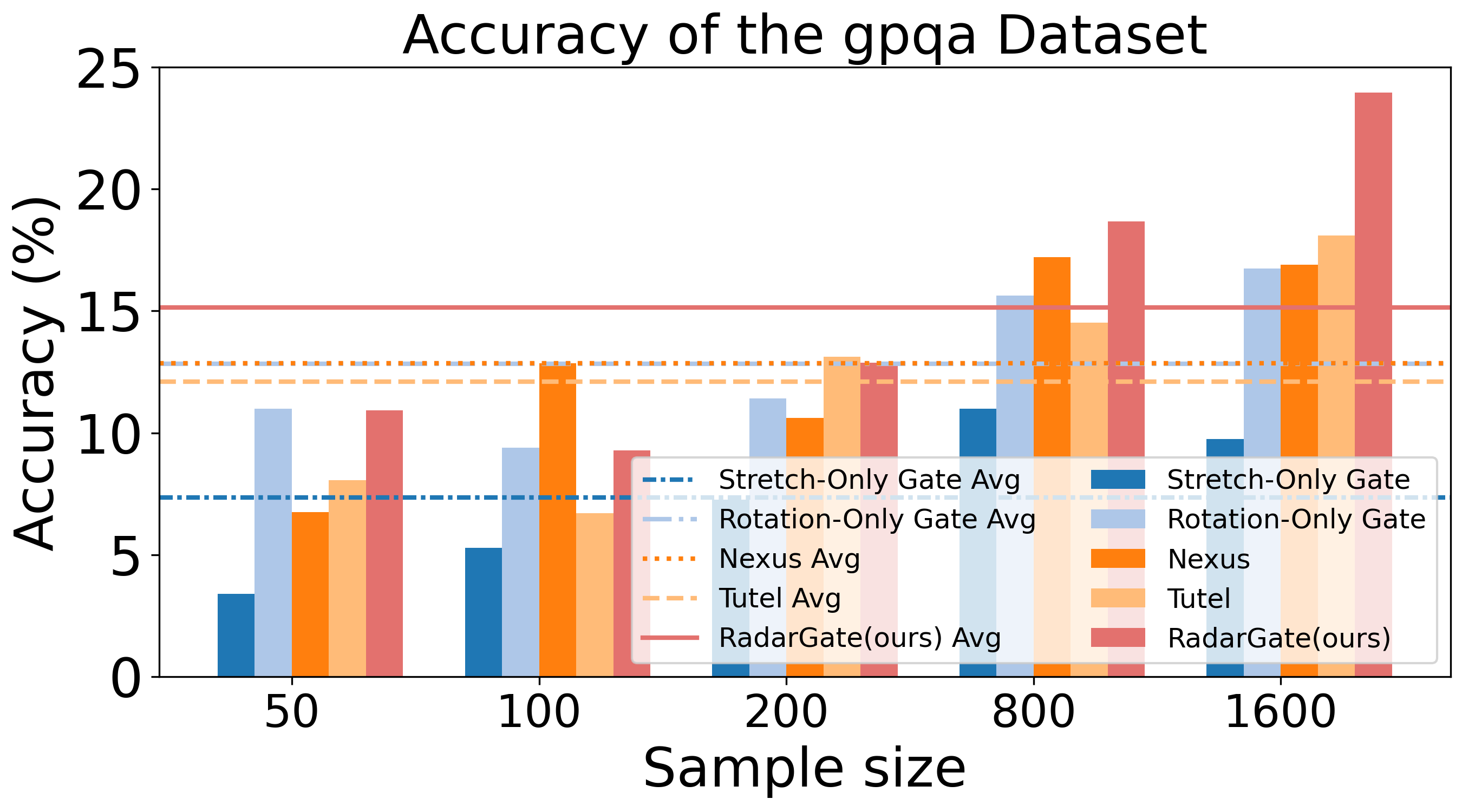}
        \caption{}
        \label{fig:app_exp_module_sca_omoe_subfig55}
    \end{subfigure}
    \hfill
    \begin{subfigure}[t]{0.32\textwidth}
        \includegraphics[width=\textwidth]{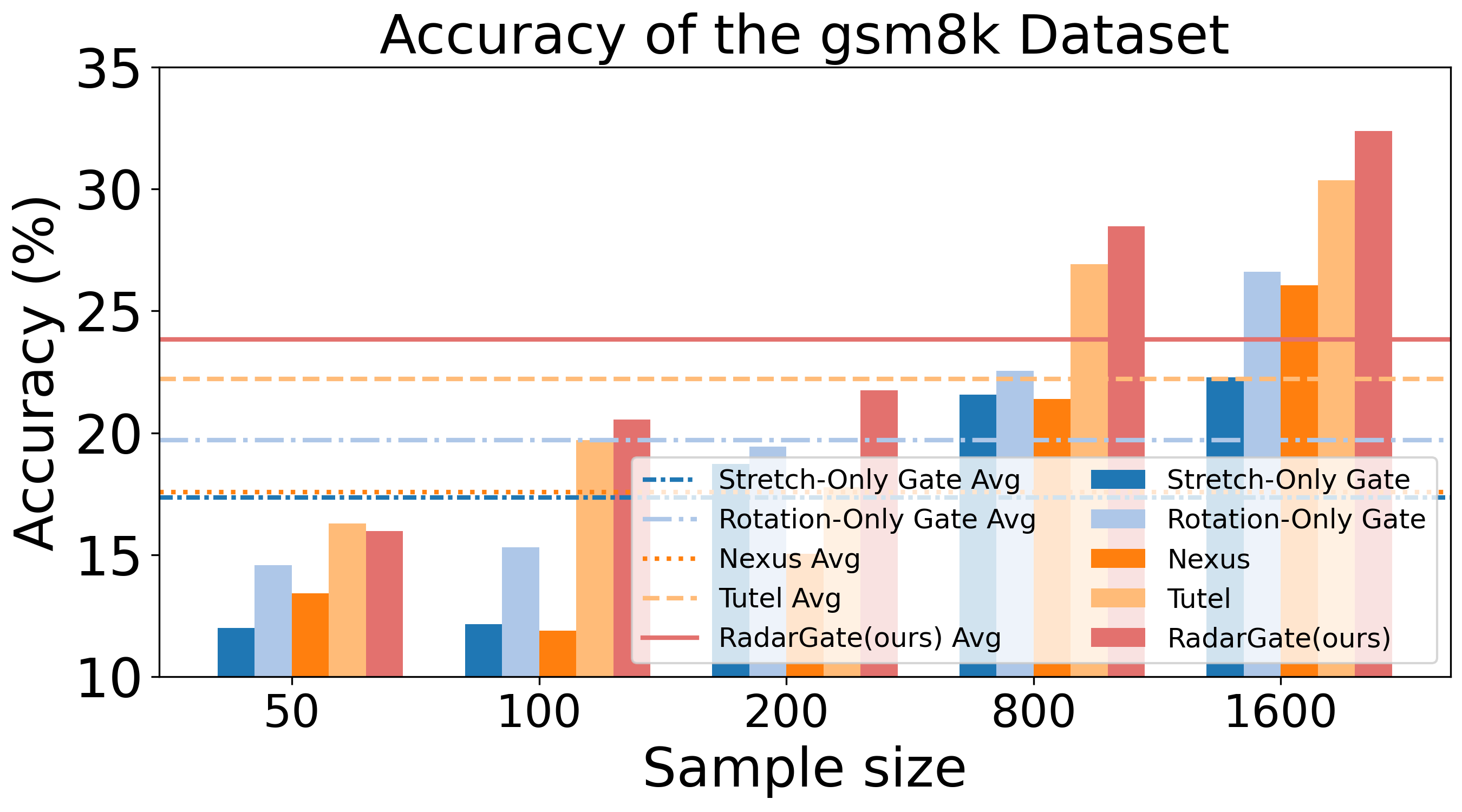}
        \caption{}
        \label{fig:app_exp_module_sca_omoe_subfig56}
    \end{subfigure}
\end{minipage}
\caption{The accuracy variations of different gates within the MoLE architecture across six different benchmarks as the number of training samples increases.} 
\label{fig:app_exp_sample_sca_mole} 
\end{center}
\vskip -0.2in
\end{figure}

\begin{figure}[htpb]
\vskip 0.2in
\begin{center}
\begin{minipage}{\textwidth}
    \centering
    
    \begin{subfigure}[t]{0.32\textwidth}
        \includegraphics[width=\textwidth]{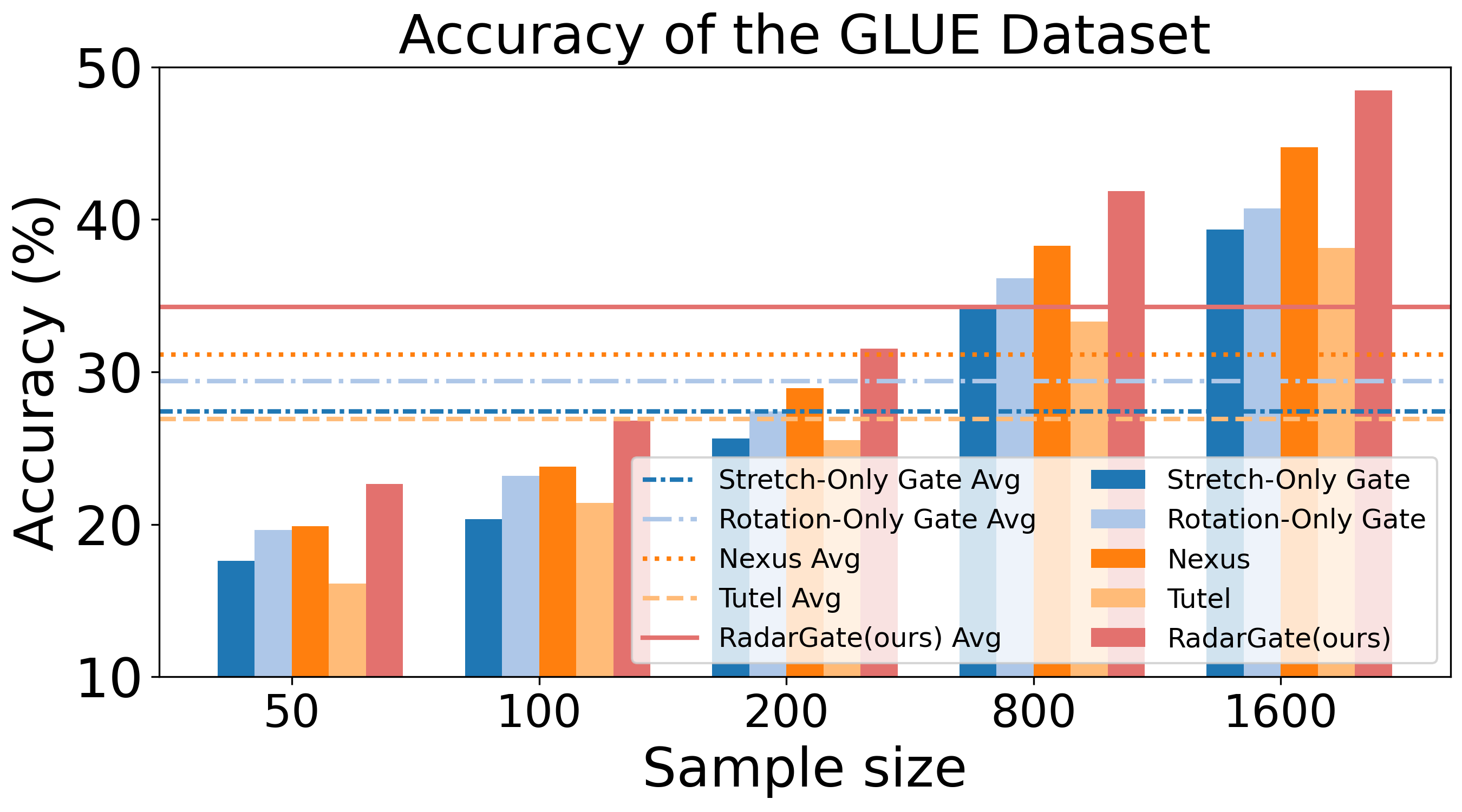}
        \caption{}
        \label{fig:app_exp_module_sca_omoe_subfig61}
    \end{subfigure}
    \hfill
    \begin{subfigure}[t]{0.32\textwidth}
        \includegraphics[width=\textwidth]{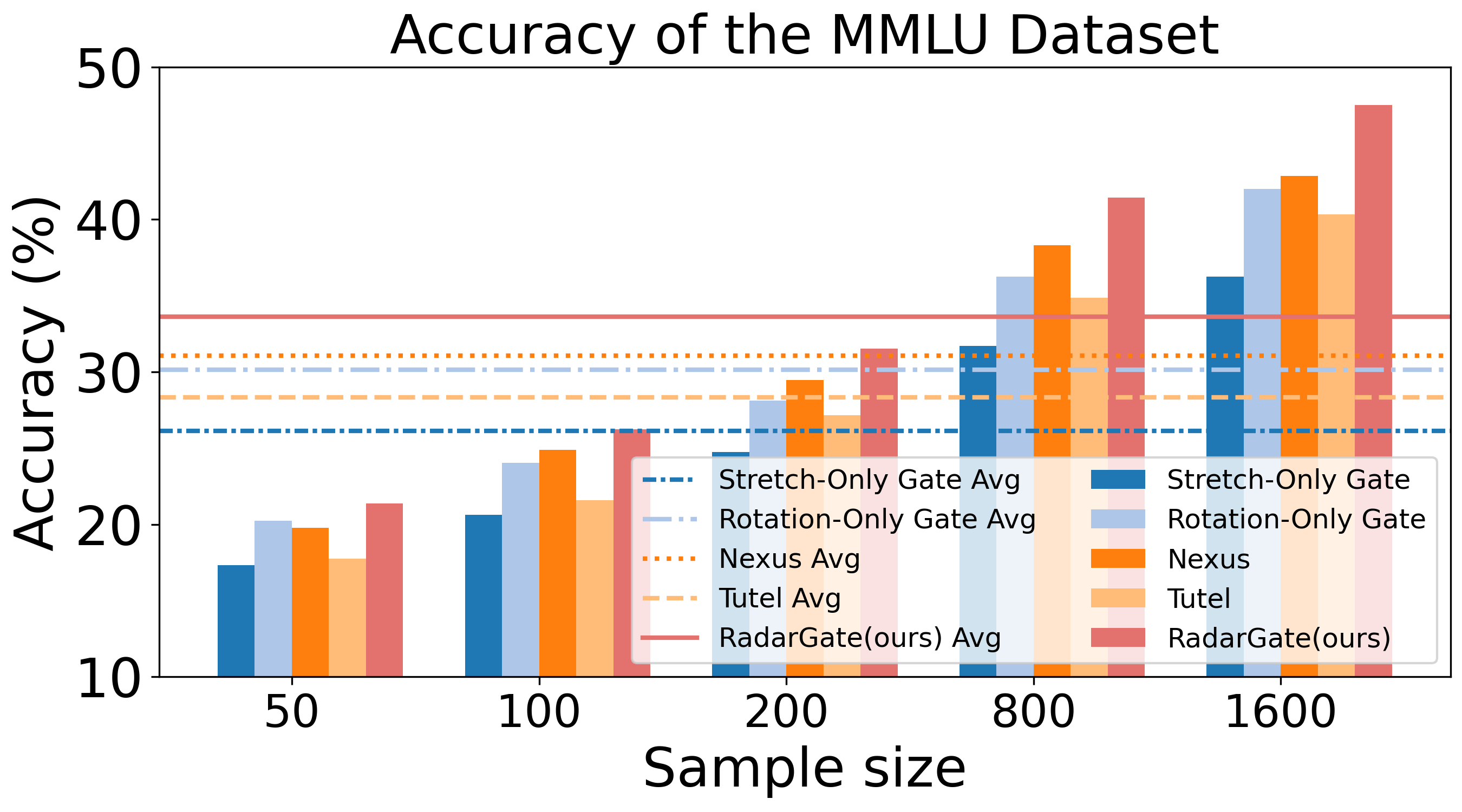}
        \caption{}
        \label{fig:app_exp_module_sca_omoe_subfig62}
    \end{subfigure}
    \hfill
    \begin{subfigure}[t]{0.32\textwidth}
        \includegraphics[width=\textwidth]{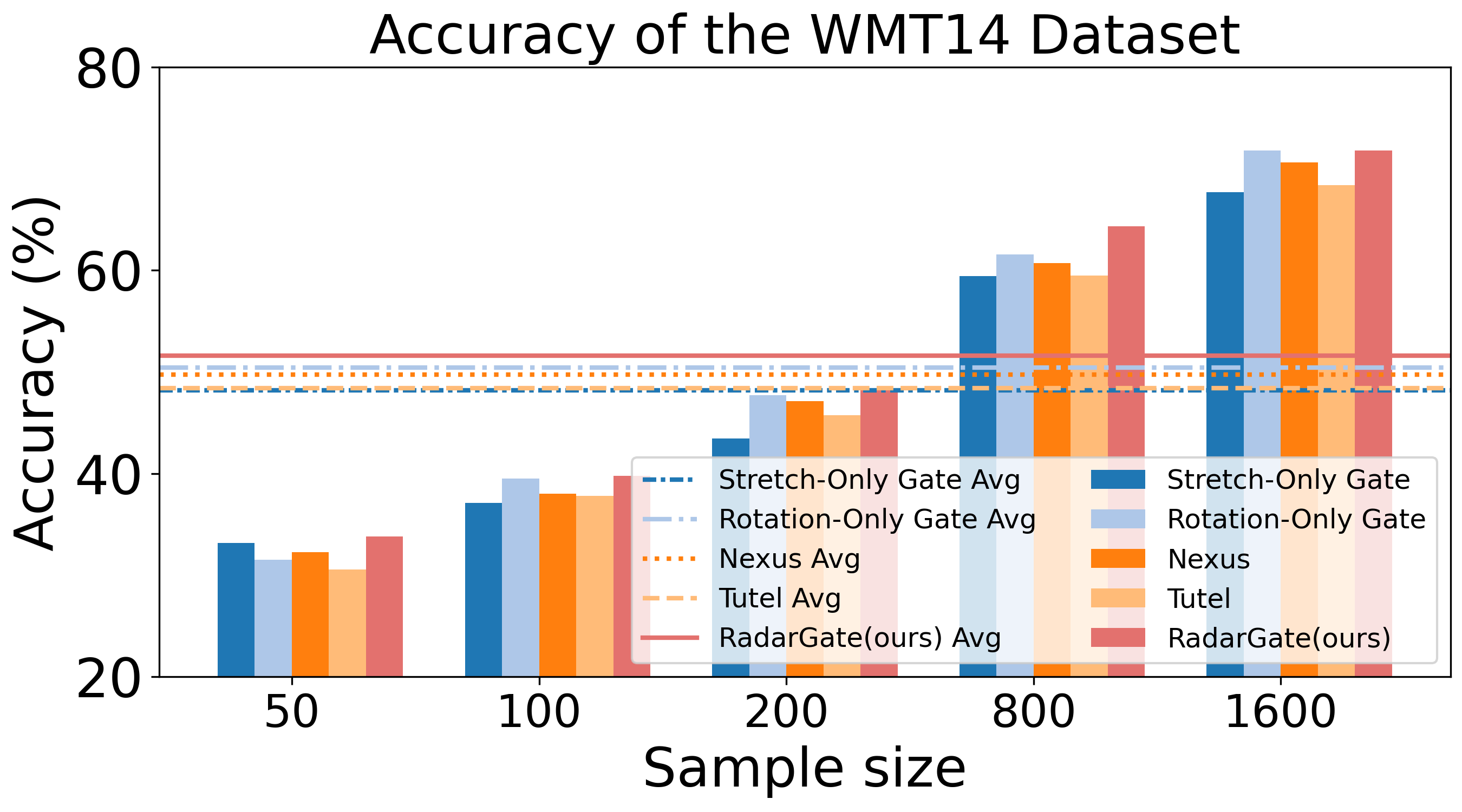}
        \caption{}
        \label{fig:app_exp_module_sca_omoe_subfig63}
    \end{subfigure}
    \hfill
    \begin{subfigure}[t]{0.32\textwidth}
        \includegraphics[width=\textwidth]{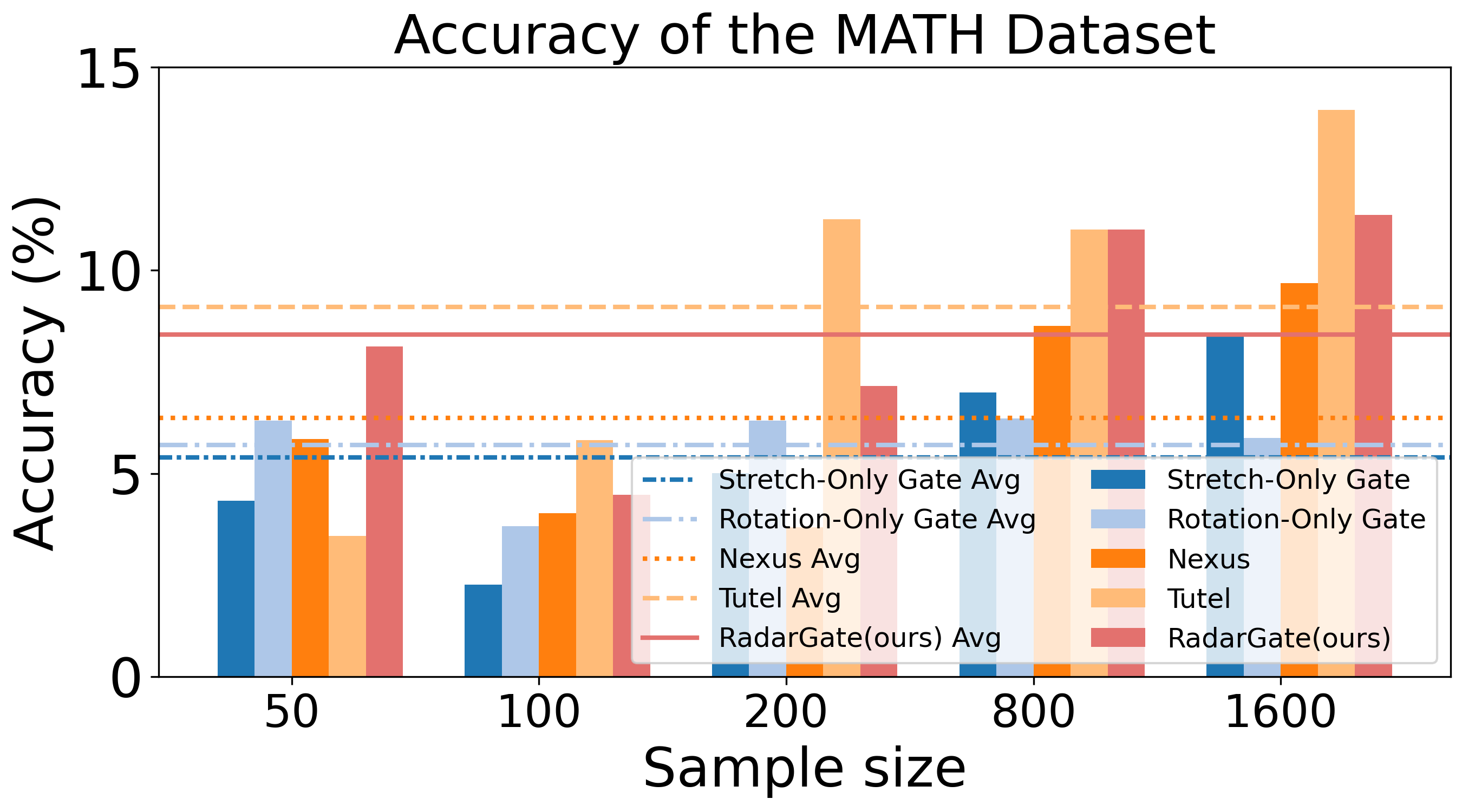}
        \caption{}
        \label{fig:app_exp_module_sca_omoe_subfig64}
    \end{subfigure}
    \hfill
    \begin{subfigure}[t]{0.32\textwidth}
        \includegraphics[width=\textwidth]{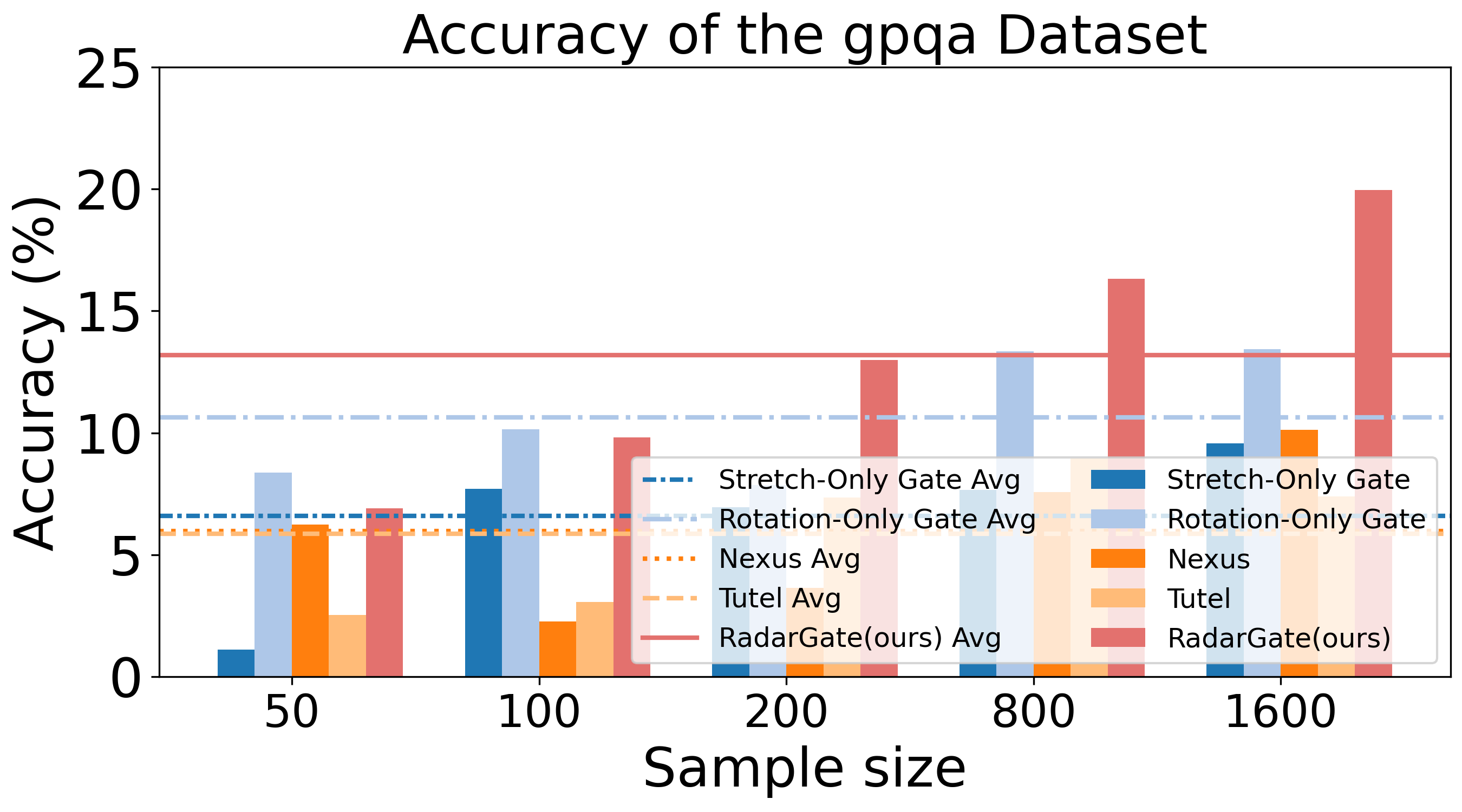}
        \caption{}
        \label{fig:app_exp_module_sca_omoe_subfig65}
    \end{subfigure}
    \hfill
    \begin{subfigure}[t]{0.32\textwidth}
        \includegraphics[width=\textwidth]{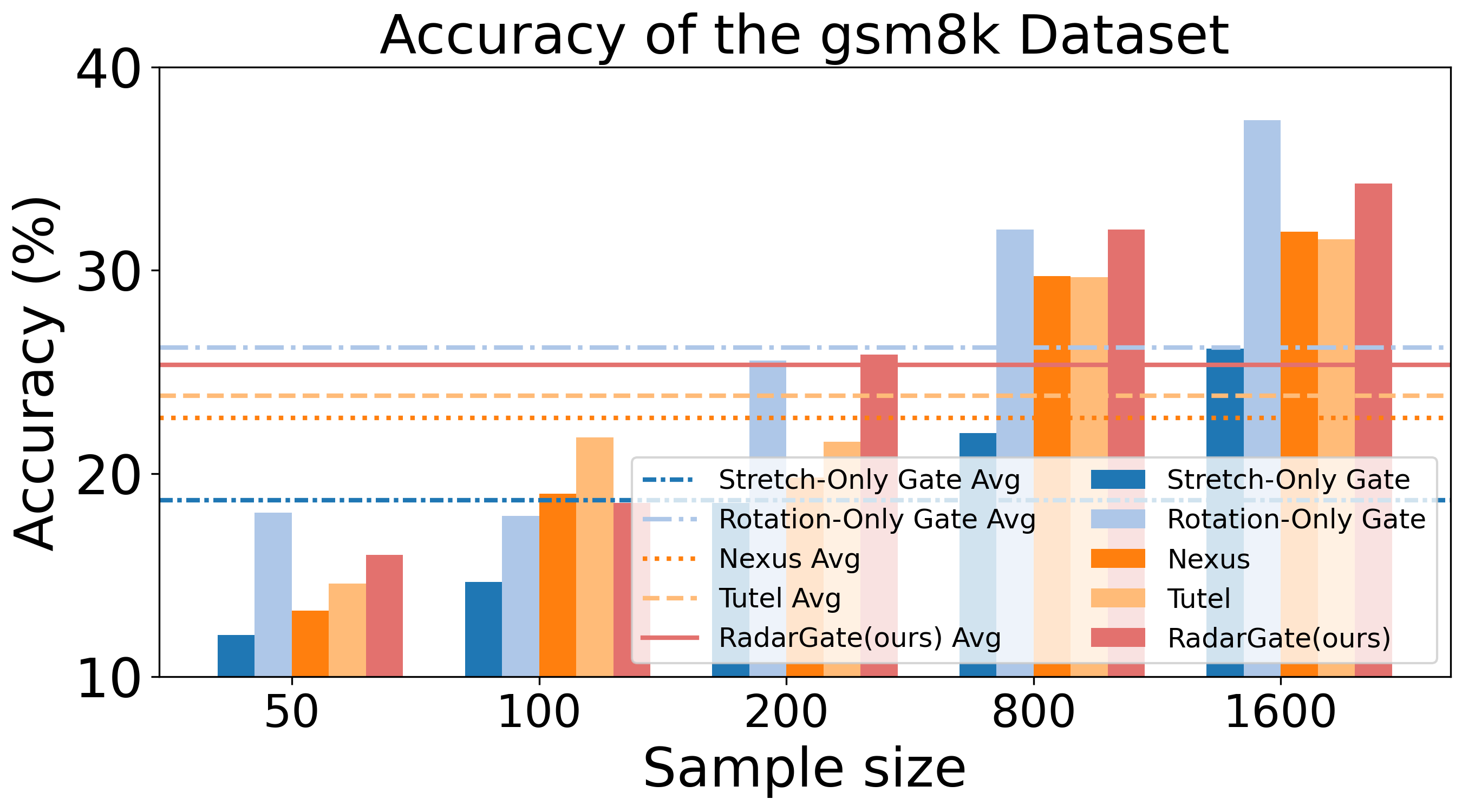}
        \caption{}
        \label{fig:app_exp_module_sca_omoe_subfig66}
    \end{subfigure}
\end{minipage}

\caption{The accuracy variations of different gates within the OMoE architecture across six different benchmarks as the number of training samples increases.} 
\label{fig:app_exp_sample_sca_omoe} 
\end{center}
\vskip -0.2in
\end{figure}

\section{Limitations \& Future Discussion}

Although our RadarGate has demonstrated excellent performance in fitting capability and generalization, and has maintained a high level after LoRAs extension, we also recognize its limitations. Specifically, we believe that RadarGate will also perform excellently in multimodal and even full-modal scenarios to improve fitting and generalization performance, but we have not yet conducted research on this. Future research can extend it to tasks in more modal domains. Similarly, in distributed scenarios, the modular characteristics of RadarGate can bring communication convenience, and the addition of rotational degrees of freedom can also bring better performance. We believe that RadarGate can bring exciting opportunities to resource-constrained scenarios in distributed computing.

\end{document}